\PassOptionsToPackage{amsthm}{newtxmath}
\documentclass[11pt]{alta}

\usepackage{mathtools,amsthm}
\usepackage{needspace,enumitem}
\usepackage{thmtools,thm-restate}
\usepackage{natbib}
\usepackage{wrapfig}
\usepackage{subcaption}
\usepackage{float,afterpage}
\newcommand{\ArxivFloatFigures}{}
\usepackage{array,longtable}
\usepackage{titletoc}
\usepackage{xurl}
\hypersetup{hypertexnames=false}
\titlecontents{appendixsection}
  [1.5em]
  {%
   \if D\thecontentslabel\clearpage\fi
   \addvspace{1em}\normalsize\fontsize{10}{15}\selectfont\bfseries\contentsmargin{2.5em}}
  {\contentslabel{1.5em}}
  {}
  {\hfill\contentspage}
\titlecontents{appendixsubsection}
  [4em]
  {\addvspace{0.3em}\normalsize\fontsize{10}{15}\selectfont\contentsmargin{2.5em}}
  {\contentslabel{2.5em}}
  {}
  {\titlerule*[0.6em]{.}\contentspage}

\newcommand{\tool}{RoPE Profiler}

\newtheorem{lemma}{Lemma}[section]
\newtheorem{proposition}{Proposition}
\newtheorem{corollary}{Corollary}[section]

\newcommand{\Var}{\operatorname{Var}}
\newcommand{\Cov}{\operatorname{Cov}}
\newcommand{\Rea}{\operatorname{Re}}
\newcommand{\dist}{\operatorname{dist}}
\newcommand{\op}{\mathrm{op}}
\newcommand{\cH}{\mathcal{H}}
\newcommand{\cN}{\mathcal{N}}
\newcommand{\cF}{\mathcal{F}}
\newcommand{\norm}[1]{\left\lVert #1 \right\rVert}
\newcommand{\abs}[1]{\left\lvert #1 \right\rvert}
\newcommand{\avg}[1]{\left\langle #1\right\rangle_M}
\newcommand{\eps}{\varepsilon}
\newcommand{\Eqref}[1]{\textup{Eq.}~\eqref{#1}}

\definecolor{takeawaybackground}{RGB}{239,245,249}
\definecolor{failureaccent}{RGB}{204,122,59}

\newif\ifshowcomments
\showcommentstrue

\definecolor{yuyangcolor}{RGB}{46,88,180}
\definecolor{yufengcolor}{RGB}{181,68,106}
\definecolor{haocolor}{RGB}{111,61,121}
\definecolor{gptcolor}{RGB}{13,153,13}

\showcommentsfalse

\title{RoPE at the End of Its Rope? Theory, Diagnosis, and Mitigation of Long-Context Failures}
\author{Yuyang Wu\textsuperscript{1}, Yufeng Du\textsuperscript{2}, Hao Peng\textsuperscript{2}}
\affiliation{\textsuperscript{1}Independent Researcher\\
\textsuperscript{2}University of Illinois at Urbana-Champaign, USA}
\runningtitle{RoPE at the End of Its Rope?}
\newcommand{\ArxivCode}{%
  \par\noindent\textbf{Code:}\enspace
  \url{https://github.com/acetocarmine11/rope-profiler}\par
}
\newcommand{\ArxivCorrespondence}{%
  \par\vspace{1em}\noindent
  {\normalfont\fontsize{8}{10}\selectfont
   \textbf{Correspondence:}\enspace
   \href{mailto:wuyuyang@alumni.pku.edu.cn}{\texttt{wuyuyang@alumni.pku.edu.cn}},
   \href{mailto:haopeng@illinois.edu}{\texttt{haopeng@illinois.edu}}\par}%
}
\makeatletter
\AddToHook{cmd/maketitle/before}{\g@addto@macro\alta@abstract{\ArxivCode\ArxivCorrespondence}}
\makeatother
\hypersetup{
  pdftitle={RoPE at the End of Its Rope? Theory, Diagnosis, and Mitigation of Long-Context Failures},
  pdfauthor={Yuyang Wu, Yufeng Du, Hao Peng}
}

\begin{document}
\begin{abstract}

Long-context failures of RoPE-based language models can arise from RoPE's intrinsic tradeoff between maintaining stable token preferences and distinguishing nearby positions. Determining which weakness to address, and how, requires a more precise characterization of RoPE's behavior in trained models across context lengths. We address a key limitation of prior theory by allowing unequal query–key scales across RoPE frequencies, which aligns well with practical empirical observations. Our theory makes both vulnerabilities measurable for individual heads and inputs, and quantifies how high-frequency components support positional sensitivity while potentially disrupting semantic stability. We also derive \textbf{a theoretical context-length bound} beyond which, under specified conditions, a fixed attention-score comparison cannot jointly avoid \emph{semantic reversal} and \emph{positional insensitivity}.
Guided by our fresh theoretical insights, we introduce \textbf{\tool{}}, a \emph{lightweight, plug-and-play diagnostic toolkit} that augments existing evaluations with zero additional forward passes by reusing cached query and key activations. Reusing activations collected during evaluation, the toolkit incurs little overhead.
It supplements standard benchmark scores with two diagnostic scores that reveal semantic and positional weaknesses and help users prioritize which aspect to address. Crucially, our evaluations across 49 long-context task settings reveal a distinct pattern where reasoning tasks predominantly suffer from semantic reversal, whereas retrieval tasks are primarily vulnerable to positional insensitivity. Guided by our theory and diagnostic profiles, targeted high-frequency rescaling achieves immediate gains without additional training, improving task accuracy by up to 20 percentage points on Qwen3-8B and 25 percentage points on Llama-3.1-8B-Instruct.
\end{abstract}

\maketitle

\afterpage{\clearpage%
\begin{figure}[H]
\centering
\captionsetup[subfigure]{font=normalsize,labelfont=bf,textfont=bf,skip=\abovecaptionskip}
\begin{subfigure}[t]{0.38636187\linewidth}
\centering
\includegraphics[width=\linewidth,trim=578bp 19bp 7.2bp 21.5bp,clip]{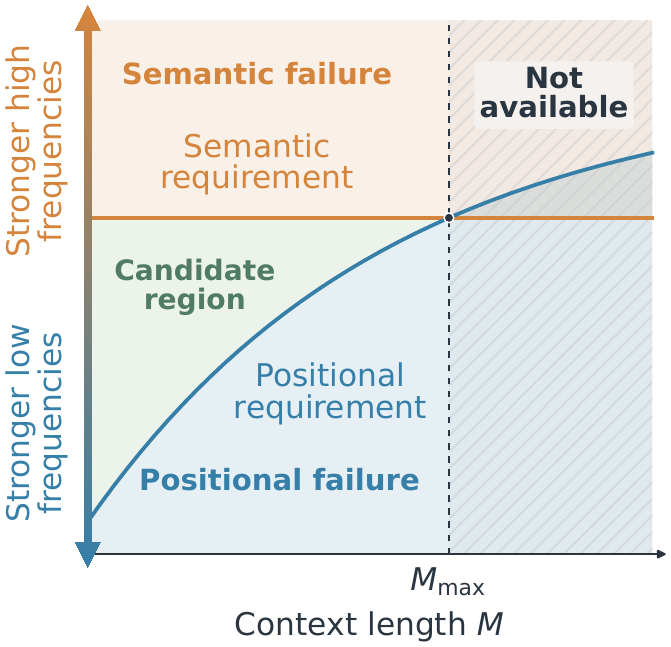}
\caption{Semantic-positional tradeoff}
\label{fig:rope-tradeoff-panel}
\end{subfigure}\hfill
\begin{subfigure}[t]{0.59816194\linewidth}
\centering
\includegraphics[width=\linewidth,trim=4.25bp 4.25bp 351.45bp 3.25bp,clip]{figures/rope_frequency_overview.pdf}
\caption{Rotary frequency dynamics}
\label{fig:rope-frequency-dynamics-panel}
\end{subfigure}
\caption{\textbf{The fundamental semantic-positional tradeoff and rotary frequency dynamics.}
\subref*{fig:rope-tradeoff-panel} The candidate region for semantic stability and adjacent positional sensitivity shrinks with context length. The orange and blue curves represent the semantic upper limit and positional lower limit on the high-frequency norm share. The green region contains candidate values satisfying both requirements; beyond $M_{\max}$, no such value remains within the calibrated family.
\subref*{fig:rope-frequency-dynamics-panel} RoPE rotates coordinate pairs of query and key vectors by frequency-dependent angles across relative token distances.
Rapidly rotating components, when aggregated across channels, form unpredictable Gaussian noise that destabilizes token preferences.
Slowly rotating components remain predictable but produce marginal score changes that fail to resolve adjacent positions.
Consequently, the dominance of different frequency bands triggers distinct failure modes.
Precise formulations are given in \S\ref{subsec:reversal-cutoff} and Appendix~\ref{app:reversal-cutoff-setup}.}
\label{fig:rope-frequency-overview}
\end{figure}
}

\section{Introduction}

Long-context large language models (LLMs) open new possibilities for tackling complex,
long-horizon problems, such as agentic software development and scientific
discovery~\citep{anthropic2025harnesses,zhao2026scienceflow}.
Rotary position embedding (RoPE; \citealp{Su_2024}) and its variants are widely used in these
models~\citep{llama3,qwen3,gemma_2025,qwen2.5,liu2026ministral3} and have consequently become a prominent target
for improving long-context
performance~\citep{chen2023extendingcontextwindowlarge,ICLR2024_874a4d89,ding2024longrope,ICLR2025_e6d58fc6,chen-etal-2025-hope,zhang2024mspoe,wang2026adarope,bloc97ntk2023,an2025string,yang2025rnope,xiong2025dope,lu2025dpe}.
Recent theoretical and empirical studies highlight a tension between
preserving content preferences and distinguishing token positions as context length
grows~\citep{ICLR2025_e6d58fc6,du2026ropedistinguishespositionstokens,ICLR2026_68f1a693,liu2026rotarypositionalembeddingsphase,NEURIPS2024_9f12dd32,jerad2026expressivity}.
Practical progress based on these findings requires a more precise
characterization of RoPE's tradeoff under assumptions that better reflect
learned representations, together with diagnostics that provide actionable
insights for targeted intervention.
We advance both the theoretical and empirical fronts.

Our theory quantifies two fundamental failure modes of RoPE-based single-head attention in a single layer and establishes their quantitative tradeoff across expanding context scales (Figure~\ref{fig:rope-tradeoff-panel}).
As illustrated in Figure~\ref{fig:rope-frequency-dynamics-panel}, RoPE rotates coordinate pairs of query and key activations at geometrically spaced angular frequencies across relative token distances.
These frequency channels, which the attention mechanism aggregates, have different functions: 
 \textcolor[HTML]{CC7A3B}{rapidly rotating components} produce fluctuations whose mean is close to zero and whose variance we can bound analytically with a Gaussian approximation \citep{du2026ropedistinguishespositionstokens,salem1947lacunary,aistleitner2023lacunary},
 whereas \textcolor[HTML]{3B6C8E}{slowly rotating components} are insensitive to adjacent positions and contribute little change to the attention score  \citep{liu2024scaling,hong2024token}.
Consequently, these different functionalities triggers two fundamental failures: \emph{semantic reversal} and \emph{positional insensitivity} (\S\ref{sec:failure-modes}).
In semantic reversal, changing position alone reverses a query's preference between keys, potentially redirecting attention toward distractors \citep{ICLR2026_b619cd6d,ICLR2025_e6d58fc6}.
In positional insensitivity, 
a one-token shift changes the attention score too little to distinguish adjacent positions,
which can hurt tasks that require retrieving elements immediately before or after a target, as shown in frequency-scaling based length-extension techniques \citep{ICLR2024_874a4d89,bloc97ntk2023,wu2026datashapesropefrequency}.
As context length grows, Figure~\ref{fig:rope-tradeoff-panel} illustrates how preserving adjacent positional sensitivity imposes a higher minimum high-frequency norm share, raising the calibrated lower bound on reversal risk and shrinking the candidate region.
Under the stated calibration and parameter conditions, this tradeoff yields a context-length bound $M_{\max}$ beyond which a given attention-score margin can no longer avoid both failures within the calibrated windows (\S\ref{sec:tradeoff}).
Crucially, our finite-window bounds accommodate arbitrary fixed query and key vectors extracted from trained models, including unequal coefficient magnitudes across channels.
By avoiding restrictive assumptions on token distributions or coordinate regularity across channels (such as in \cite{NEURIPS2024_9f12dd32,ICLR2025_e6d58fc6,du2026ropedistinguishespositionstokens}; see also \S\ref{app:related-representation-assumptions}), our analysis establishes a realistic theoretical foundation for this spectral ceiling.

Our theory reveals constraints on mitigating both failures simultaneously. However, we also find that these constraints imply a principled basis for identifying which weakness to prioritize and how to address it for a given task.
Building on our theoretical insights, we introduce \textbf{\tool{}}, a lightweight diagnostic tool that seamlessly integrates into existing benchmarks (\S\ref{subsec:current-tool-definition}).
\tool{} reuses query and key activations cached during evaluation and incurs little additional computational overhead.
\tool{} turns the two theoretical failure criteria into practical diagnostics, supplementing standard benchmark scores with two metrics: the \textit{Semantic Score} measures the stability of key preferences across positions, while the \textit{Positional Score} measures sensitivity across adjacent positions.
Our theory provides a quantitative basis for two training-free interventions: rescaling high-frequency components to adjust their relative strength
and freezing their rotations to suppress position-induced fluctuations \citep{qiao2025rethinking,ICLR2025_e6d58fc6,chen-etal-2025-hope}.
Across 49 task configurations from nine benchmark sources
(described in Appendix~\ref{app:diagnostic-task-set}), the evaluated models show broadly similar rankings of their relative semantic and positional weaknesses (\S\ref{subsec:current-task-preferences}).
Guided by these diagnostic profiles, our training-free interventions yield best observed accuracy gains of up to 20 percentage points on Qwen3-8B and 25 percentage points on Llama-3.1-8B-Instruct, demonstrating that spectral diagnosis provides effective, targeted guidance for long-context optimization (\S\ref{subsec:current-intervention}).

Although our theory aligns with \citet{du2026ropedistinguishespositionstokens}'s pessimistic assessment of RoPE's intrinsic long-context limitations, our precise spectral characterization reveals concrete opportunities for targeted improvements in practice. 
Our findings demonstrate that substantial optimization space remains simply by adjusting frequency allocations across individual attention heads.
We believe further progress lies in incorporating \tool{}'s diagnostics into training and long-context adaptation.
Beyond uniform frequency scaling, designing attention mechanisms that assign complementary roles to different heads can better balance content retrieval and positional reasoning \citep{shi2026implicit,wang2026adarope, xiong2025dope, xiao2025duo, tang2025razor, ICLR2025_9b77f073}.
In the longer term, fully overcoming these intrinsic constraints will require rethinking attention architectures and how positions are encoded in foundation models~\citep{kimiteam2026kimik3, alibi, kazemnejad2023impact, yang2025rnope, hua2025fope}.

\section{Two Failure Modes of RoPE: Semantic Reversal and Positional Insensitivity}
\label{sec:failure-modes}

A good positional embedding (PE) in transformers should achieve two objectives:
\textbf{1) retain the attention's ability to distinguish different tokens:} for a given query, the attention mechanism distinguishes two keys by having different
scores, and PEs should not get in the way of this feature;
\textbf{2) distinguish different positions:} varying the score for the same key across locations informs the model of token positions~\citep{Su_2024,vaswani2017attention,shaw2018self,press2022alibi}.
Unmasked attention with position-independent inputs is permutation-equivariant~\citep{lee2019set}, although causal language models can learn positional information without explicit positional encodings~\citep{haviv2022positional}. RoPE encodes relative distance by rotating query and key coordinate pairs, making their attention score position-dependent.

\subsection{Background}
\label{subsec:failure-background}
\paragraph{RoPE as a Sum of Rotated Complex Numbers.} 
Let $q$ and $k$ denote the query and key vectors from an attention head. Their scaled dot product forms the pre-softmax attention score, which converts into weights for combining value vectors. RoPE~\citep{Su_2024} injects position information by rotating coordinate pairs of queries and keys. Let $p_q$ and $p_k$ represent the token positions, and let $m=p_q-p_k$ denote their relative distance. For a head with $h$ rotary coordinate pairs and base $B$, coordinate pair $n$ rotates by angle $p\omega_n$ at position $p$, where $\omega_n=B^{-n/h}$ for $n=0,\ldots,h-1$. While the base frequency $\omega_n$ decays smoothly with coordinate index $n$, individual components exhibit starkly disparate behaviors across long contexts.\footnote{For example, consider a typical setup with rotary dimension $2h=128$, base $B=500{,}000$ (as in Llama 3 \cite{llama3}), and context length $L=32{,}768$. Over this window, the fastest component with $\omega_0=1$ completes over $5{,}215$ full cycles, while the slowest component with $\omega_{h-1} \approx 2 \times 10^{-6}$ completes approximately $0.01$ cycles, covering merely a tiny fraction of a single rotation.}

Following the complex-valued formulation of RoPE in \citet{zhu-etal-2024-coca}, let $Q_n=q_{n,1}+i q_{n,2}$ and $K_n=k_{n,1}+i k_{n,2}$ denote the complex representations of each rotary pair, with $i^2=-1$. With attention scale $c_{\mathrm{att}}$, the rotary attention score simplifies to
\begin{equation}
S_{q,k}(m)=\Rea\sum_{n=0}^{h-1}z_ne^{im\omega_n},
\qquad z_n=c_{\mathrm{att}}Q_n\overline{K_n},
\label{eq:background-rope-score}
\end{equation}
where $\Rea$ takes the real part and the overline denotes complex conjugation. Here, token content determines the static amplitude coefficients $z_n$, while relative distance $m$ modulates the dynamic rotation phases $m\omega_n$. We write $S(m)$ when the query and key are clear from the context.

\paragraph{The Janus Face of RoPE: Sustaining Content and Localizing Tokens.}
By applying frequency-dependent rotations across relative distances, RoPE enables attention to produce distinct scores according to both token content vectors and spatial positions. Extensive prior studies seek to demonstrate how RoPE successfully realizes these two core functions. These efforts span early decoupled representation designs, approximate additive decompositions and frequency-wise functional division in learned RoPE representations, examinations of locality and symmetry under role shuffling, and symbolic-versus-positional permutations that probe expressive capacity~\citep{kitaev-klein-2018-constituency,ke2021rethinking,chen-etal-2023-locality,han-ji-2025-computation,ICLR2025_e6d58fc6,ICLR2026_68f1a693,ICLR2026_0aee38a6,ICLR2026_993dbaec,chen2024ropeweights,gu2026deconstructing,hong2024token,jin2025massive}. In essence, these investigations explain \textit{how RoPE succeeds}. However, analyzing its functional capability provides only an incomplete picture. The more critical question is how this dual responsibility becomes unbalanced as the context scale continues to expand.

\ifdefined\ArxivFloatFigures\else
\ifdefined\ArxivFloatFigures
\begin{figure}[tbp]
\else
\begin{figure}[H]
\fi
\centering
\captionsetup[subfigure]{font=normalsize,labelfont=bf,textfont=bf,skip=\abovecaptionskip}
\begin{subfigure}[t]{0.59\linewidth}
\centering
\includegraphics[width=\linewidth]{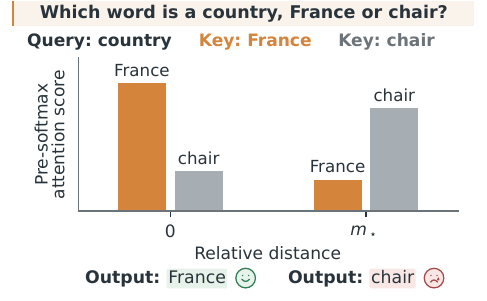}
\caption{Semantic reversal}
\label{fig:semantic-reversal-panel}
\end{subfigure}\hfill
\begin{subfigure}[t]{0.38\linewidth}
\centering
\includegraphics[width=\linewidth]{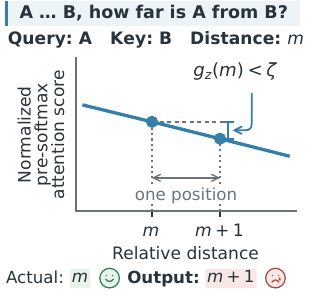}
\caption{Positional insensitivity}
\label{fig:positional-insensitivity-panel}
\end{subfigure}
\caption{\textbf{Illustrative task consequences of two RoPE failure modes.}
\subref*{fig:semantic-reversal-panel} For query \texttt{country}, the scores of keys \texttt{France} and \texttt{chair} reverse order between virtual relative distances $0$ and $m_\star$.
 As attention directs information routing toward higher scores, this reversal favors the distractor \texttt{chair} over the semantically relevant \texttt{France} at relative distance $m_\star$ and can contribute to an incorrect prediction.
\subref*{fig:positional-insensitivity-panel} For query \texttt{A} and key \texttt{B}, the vertical bracket marks a marginal normalized score gap between adjacent virtual relative distances, $g_z(m)<\zeta$. This leaves the model unable to distinguish between the two adjacent positions.}
\label{fig:semantic-reversal}
\end{figure}

\fi

\subsection{Formal Definitions of Semantic Reversal and Positional Insensitivity}

\label{subsec:two-failure-modes}
\label{subsec:semantic-reversal}
\label{subsec:positional-insensitivity}

\textbf{RoPE gives attention a way to distinguish positions, but what is the cost?}
Examining this question through the lens of failure modes, \citet{du2026ropedistinguishespositionstokens} formalize four distinct vulnerabilities in RoPE-based attention.
In this work, we focus on two core failure modes: \emph{semantic reversal} and \emph{positional insensitivity}.
These two dimensions prove theoretically sufficient to uncover the fundamental context limits of RoPE in \S\ref{sec:theory}, while remaining practically informative for our diagnostic tool design in \S\ref{sec:current-diagnostic-tool}.
Furthermore, our theoretical framework naturally extends to the other failure modes analyzed in prior literature, which we detail in Appendix~\ref{app:deferred-reliability-theory}.
Real-world tasks do not necessarily suffer from both vulnerabilities equally, and in \S\ref{subsec:current-task-preferences} we illustrate how specific task families are governed by distinct failure modes.
Below, we formalize the mathematical definitions of these two primary failure modes.

\ifdefined\ArxivFloatFigures

\fi

\paragraph{\textcolor[HTML]{D5843C}{Semantic reversal  (Different Tokens, Same Position).}}
Motivated by \citet{du2026ropedistinguishespositionstokens}, we examine the stability of attention preferences between distinct tokens across varying relative distances.
Figure~\ref{fig:semantic-reversal-panel} illustrates an instance of semantic reversal. For a query and two keys, let $D(m)=S_+(m)-S_-(m)$ denote their ordered score margin. Without loss of generality, we assume $D(0)>0$. A semantic reversal occurs at any relative distance where $D(m)<0$. To capture overall behavior without being biased by individual positions, we evaluate how frequently semantic reversals occur across the entire context window of length $M$ as the semantic reversal probability:
\begin{equation}
p_{\mathrm{rev}}(D;M)
:=\frac{1}{M}\sum_{m=0}^{M-1}\mathbf{1}_{\{D(m)<0\}}
=\mathbb P[D(\mathbf m)<0],
\qquad \mathbf m\sim\operatorname{Unif}\{0,\ldots,M-1\}.
\label{eq:semantic-order-reversal}
\end{equation}
The probability evaluates the stability of the model's preference across the full sequence.

\paragraph{\textcolor[HTML]{357FA8}{Positional insensitivity (Same Token, Different Positions).}}
Motivated by \citet{liu2026rotarypositionalembeddingsphase,sun2026pas}, we evaluate how sensitively an attention score responds to one-token changes in relative distance while holding token representations fixed.
For distant or semantically unrelated token pairs, resolving exact single-token offsets carries little operational value;
however, in retrieval tasks that demand precise relative order, as illustrated in Figure~\ref{fig:positional-insensitivity-panel}, distinguishing adjacent positions becomes essential to resolve local sequence structure \citep{wu-etal-2025-lifbench,wang-etal-2024-ada}.

An intuitive metric is to measure the raw score change $|S(m+1) - S(m)|$ across adjacent locations.
However, this raw difference scales with representation magnitude, leaving the raw gap vulnerable to arbitrary norm variations across heads \citep{qi2025semanticsrediscoveringspatialawareness,jin2025massive}.
To isolate intrinsic positional sensitivity from overall vector magnitude, we normalize the local score difference by the pair's coefficient norm $U(z) = \|z\|_2 = (\sum_n |z_n|^2)^{1/2}$.
The resulting \emph{normalized adjacent response} is
\begin{equation}
g_z(m) := \frac{|S(m+1) - S(m)|}{U(z)},
\qquad U(z) > 0,\quad 0 \leq m \leq M-2.
\label{eq:current-adjacent-response}
\end{equation}
For a chosen response threshold $\zeta > 0$, \emph{positional insensitivity} occurs when $g_z(m) < \zeta$, quantifying the failure to preserve local positional resolution.

\section{A Theory of Long-Context Failures: The Tradeoff Between Semantic Stability and Positional Sensitivity}
\label{sec:theory}
\label{sec:tradeoff}

RoPE relies on a shared set of coordinate rotations to preserve content preferences and distinguish positions.
In this section, we reveal a tradeoff between semantic stability and
positional sensitivity in RoPE, both governed by the norm share of \textbf{high-frequency} components. Over long context windows, \textit{an insufficient high-frequency contribution leaves attention blind to adjacent positions, whereas amplifying these fast oscillations triggers semantic reversal}.
Below, we formalize high-frequency score variation under a finite-window framework, and then derive \textbf{a context-length limit} where avoiding both failure modes becomes mathematically impossible.

\ifdefined\ArxivFloatFigures
\ifdefined\ArxivFloatFigures
\begin{figure}[htbp]
\else
\begin{figure}[H]
\fi
\centering
\captionsetup[subfigure]{font=normalsize,labelfont=bf,textfont=bf,skip=\abovecaptionskip}
\ifdefined\ArxivFloatFigures
\begin{subfigure}[t]{.30\textwidth}
\centering
\includegraphics[width=1.56in]{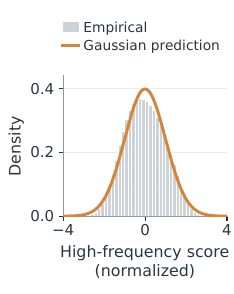}
\else
\begin{subfigure}[t]{1.56in}
\centering
\includegraphics[width=\linewidth]{figures/high_frequency_gaussian_qwen.pdf}
\fi
\caption{Qwen3-8B distribution}
\label{fig:high-frequency-gaussian-qwen}
\end{subfigure}\hfill
\ifdefined\ArxivFloatFigures
\begin{subfigure}[t]{.33\textwidth}
\centering
\includegraphics[width=1.92in]{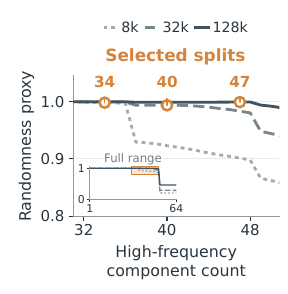}
\else
\begin{subfigure}[t]{1.92in}
\centering
\includegraphics[width=\linewidth]{figures/high_frequency_cutoff_qwen.pdf}
\fi
\caption{Qwen3-8B split selection}
\label{fig:high-frequency-cutoff-qwen}
\end{subfigure}\hfill
\ifdefined\ArxivFloatFigures
\begin{subfigure}[t]{.37\textwidth}
\centering
\includegraphics[width=1.92in]{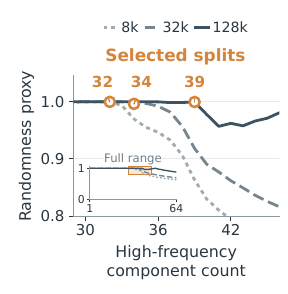}
\else
\begin{subfigure}[t]{1.92in}
\centering
\includegraphics[width=\linewidth]{figures/high_frequency_cutoff_llama.pdf}
\fi
\caption{Llama-3.1-8B split selection}
\label{fig:high-frequency-cutoff-llama}
\end{subfigure}
\caption{
\textbf{(a)} Under our high-frequency split, the predicted score distribution closely matches the empirical distribution.
\textbf{(b,c) Defining the high-frequency band.} The curves show how randomness changes as more components are included. Our theory-guided split selects the largest band satisfying the split criterion while retaining high randomness. The trends vary with context length, and larger $M$ allows more high-frequency components, consistent with our theory. (See Appendix~\ref{app:fig3-experimental-details}.)}
\label{fig:main-high-frequency-cutoff}
\end{figure}

\fi

\subsection{High-Frequency Norm Share and Score Variation}
\label{subsec:partition}

To analyze high-frequency fluctuations rigorously, we first specify an analytical moment tolerance $\eps \in (0, 1/2)$ that governs the precision of finite-window moment control.
Under the standard geometric grid $\omega_n = \rho^n$ with $\rho = B^{-1/h}$, demanding a tighter analytical tolerance requires a higher frequency cutoff, yielding the certified high-frequency band
\begin{equation}
H = \left\{n \in \{0, \dots, h-1\} : M\omega_n \geq \frac{2C_{\mathrm F}}{(1-\rho)\eps}\right\},
\label{eq:main-window-split}
\end{equation}
where $C_{\mathrm F} > 0$ denotes the universal Fourier-frame bound established in Appendix~\ref{app:frame-bound}.
This $\eps$-dependent criterion guarantees that all coordinates in $H$ complete sufficient oscillations over window $M$ to suppress cross-frequency interference.

\begin{restatable}[\textbf{Analytical moment guarantees under finite windows}]{proposition}{HighBandCalibration}
\label{prop:high-calibration}
For any tolerance $\eps \in (0, 1/2)$ and certified window $M$ where $H$ is nonempty, the high-frequency attention score $S_H(m) := \Re \sum_{n \in H} z_n e^{i m \omega_n}$ and its coefficient norm $U_H(z) := (\sum_{n \in H} |z_n|^2)^{1/2}$ over uniformly distributed relative distances $m \sim \operatorname{Unif}\{0, \dots, M-1\}$ satisfy
\begin{equation}
|\mathbb{E}_m[S_H(m)]| \leq \eps U_H(z)
\qquad\text{and}\qquad
\left|\operatorname{Var}_m(S_H(m)) - \frac{1}{2}U_H(z)^2\right| \leq \eps U_H(z)^2.
\label{eq:main-moment-bounds}
\end{equation}
\end{restatable}

Full proofs are deferred to \S~\ref{app:high-exact-calibration}.
Following the statistical behavior observed by \citet{du2026ropedistinguishespositionstokens}, which is an application of the lacunary Central Limit Theorem  \citep{salem1947lacunary,aistleitner2023lacunary}, sums of rapidly oscillating components across relative distances behave as zero-mean Gaussian fluctuations, yielding the operational score approximation $S_H(m) \approx \mathcal{N}(0, U_H(z)^2/2)$.

\ifdefined\ArxivFloatFigures\else

\fi

\paragraph{Error-Bounded Spectral Partitioning.}
Proposition~\ref{prop:high-calibration} links the frequency cutoff $H$ directly to the analytical error tolerance $\eps$, establishing precise control over finite-window moments.
Demanding tighter precision forces slower components out of $H$, raising the frequency cutoff accordingly.
While prior analyses assume specific query and key distributions or amplitude regularity to derive bounds~\citep{NEURIPS2024_9f12dd32, ICLR2025_e6d58fc6, du2026ropedistinguishespositionstokens}, our finite-window moment bounds hold for arbitrary fixed activations extracted from trained models, accommodating general phases and non-uniform coefficient magnitudes.
Figure~\ref{fig:main-high-frequency-cutoff} illustrates the empirical behavior of our theoretically predicted partitioning across 8k, 32k, and 128k relative-distance windows.

\paragraph{\textbf{The Governing Norm Share $r_H$}}
For context length $M$, we quantify the proportion of energy concentrated in fast oscillations through the high-frequency norm share
\begin{equation}
r_H(z; M) := \frac{U_H(z)}{U(z)} \in [0, 1],
\qquad\text{where } U(z) := \|z\|_2 = \left(\sum_{n=0}^{h-1} |z_n|^2\right)^{1/2}.
\label{eq:main-norm-share-def}
\end{equation}
The parameter $r_H(z; M)$ serves as the central mathematical pivot connecting spectral allocations to downstream failure modes.
Specifically, the two failure modes defined in \S\ref{subsec:two-failure-modes} constrain this energy share in opposing directions.
For an ordered score margin, the calibrated lower bound on semantic reversal probability is non-decreasing in $r_H$ (Appendix~\ref{app:reversal-cutoff-setup}), while retaining adjacent positional sensitivity for the same margin imposes a context-dependent lower bound on $r_H$ (Appendix~\ref{subsubsec:local-positional-response}).
This dual constraint establishes the foundation for the fundamental tradeoff derived in \S\ref{subsec:reversal-cutoff}.

\subsection{The Semantic and Positional Tradeoff Limits Context Length}
\label{subsec:reversal-cutoff}

Distinguishing adjacent positions requires attention scores to vary, while preserving token preferences demands that their relative ordering remains stable. High-frequency components govern both requirements. Suppressing these oscillations stabilizes token orderings at the expense of adjacent positional sensitivity. As the context window expands, slower frequencies enter the high-frequency band, leaving fewer static components to sustain local differences. Maintaining a target positional gap imposes a context-dependent lower bound on the high-frequency norm share, yielding a non-decreasing lower bound on reversal probability under the calibration conditions below.

\begin{restatable}[Context-length limit on joint reliability (informal)]{theorem}{MaximumReliableContext}
\label{thm:maximum-reliable-context}
For a fixed ordered score margin and reversal and adjacent-response tolerances
$\delta$ and $\zeta$ satisfying the 
conditions in
Appendix~\ref{app:reversal-cutoff-setup}, there exists a context-length bound
$M_{\max}$ such that, for every calibrated window,
\begin{equation}
M > M_{\max} \implies p_{\mathrm{rev}}(M) > \delta
\quad\text{or}\quad
\max_{0 \leq m \leq M-2} g(m) < \zeta,
\label{eq:main-incompatibility-implication}
\end{equation}
where $p_{\mathrm{rev}}(M)$ and $g(m)$ denote the reversal probability and
normalized adjacent response of the same margin over relative distance $m$.
Beyond $M_{\max}$, this margin cannot satisfy both requirements within the
calibrated family.
\end{restatable}

Appendix~\ref{app:reversal-cutoff-setup} gives the precise calibrated-window
conditions, the construction of $M_{\max}$, and the formal version of
Theorem~\ref{thm:maximum-reliable-context}. The construction balances the
high-frequency share required for local sensitivity against the reversal
probability it induces.

\paragraph{\textbf{A RoPE-of-War between Content and Position} (\textcolor{failureaccent}{\textbf{Diagnose What the Task Needs First}}).}
While existing literature identifies conceptual tensions in rotary attention regarding content selectivity and positional tracking~\citep{ICLR2025_e6d58fc6,du2026ropedistinguishespositionstokens,ICLR2026_68f1a693}, we formalize a quantitative tradeoff showing that semantic stability and positional sensitivity are mutually constrained under rotary representations.
Theorem~\ref{thm:maximum-reliable-context} gives a conditional bound on joint
score-level reliability, determined by the frequency grid, the two reliability
tolerances, and the calibration error.
Within this setting, shifting spectral weight changes the balance between two requirements. 
The analysis suggests weakening high-frequency contributions to stabilize key preferences and strengthening them to distinguish adjacent positions, depending on the task.
\S\ref{sec:current-diagnostic-tool} operationalizes this principle by introducing \textbf{\tool{}} to evaluate these competing vulnerabilities and guide targeted interventions.

\section{Diagnosis and Targeted Mitigation of Long-Context Failures with \tool{}}
\label{sec:current-diagnostic-tool}

Building on our theoretical criteria, we introduce \textbf{\tool{}} to turn spectral failure analysis into actionable diagnostics and targeted model adjustments.
Moving beyond prior diagnostic frameworks that primarily identify internal vulnerabilities (Appendix~\ref{app:related-internal-diagnosis}), our tool establishes a bridge between spectral state profiling and targeted training-free mitigation.
Figure~\ref{fig:diagnostic-workflow} outlines the five-stage pipeline.
In \S\ref{subsec:current-tool-definition}, we formalize how cached activations yield the \textit{Semantic Score} and \textit{Positional Score}, evaluate relative failure susceptibility across benchmark settings, and derive theory-motivated interventions.
In \S\ref{subsec:current-intervention-results}, we evaluate these diagnostic profiles and targeted modulations across diverse tasks and models, demonstrating substantial accuracy gains across a comprehensive long-context suite.

\begin{figure}[tbhp]
  \centering
  \includegraphics[width=\linewidth]{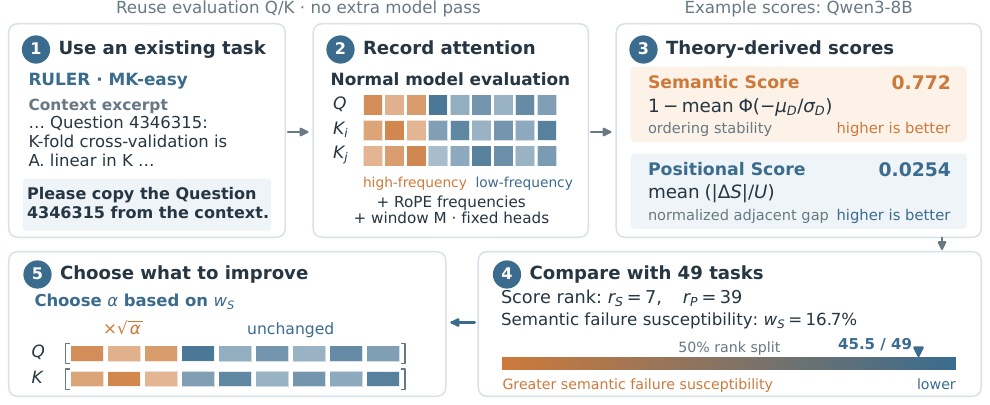}
  \caption{\textbf{Augmenting accuracy-based evaluation with failure diagnostics.}
  For a given model and the 49 reference task configurations
  (described in \S\ref{app:diagnostic-task-set}), we use query and key activations
  from standard evaluation to compute relative semantic failure
  susceptibility $w_S$.
  A higher $w_S$ suggests greater susceptibility to semantic reversal
  relative to positional insensitivity within the reference set.}
  \label{fig:diagnostic-workflow}
\end{figure}

\subsection{From Spectral Theory to Practical Profiling and Interventions}
\label{subsec:current-tool-definition}

During the standard inference pass of benchmark evaluation, \tool{} samples query and key tokens to cache their internal activations $q, k$ across target attention heads (selection details and heatmaps in Appendix~\ref{app:selection-protocol}) (Figure~\ref{fig:diagnostic-workflow}, Steps 1 and 2).
Holding these content vectors fixed, we evaluate virtual attention scores by sweeping RoPE relative distances across the context window, incurring zero additional model inference passes.

\paragraph{Theoretical Scores from Cached States (Figure~\ref{fig:diagnostic-workflow}, Step 3).}
Our theoretical criteria translate directly into two continuous metrics.
For local positional sensitivity, we evaluate query--key pairs $(q, k)$ by averaging the normalized adjacent response $g_{q, k}(m) = |S(m+1) - S(m)| / U(z)$ across relative distances $m$, defining the \emph{Positional Score} as $\bar{g}$.
For semantic stability, we construct triplets $(q, k_1, k_2)$ with distinct keys and compute the finite-window probability $P_{\text{rev}}$ that position shifts invert their reference ranking at distance zero, using our moment bounds and Gaussian approximation.
The \emph{Semantic Score} is defined as $1 - \bar{P}_{\text{rev}}$ to ensure that higher values are better for both metrics.

\paragraph{Relative Failure Susceptibility via Normalized Ranks (Figure~\ref{fig:diagnostic-workflow}, Step 4).}
Directly comparing raw scores across different tasks or failure modes is ineffective because absolute score magnitudes lack a shared scale, and meanwhile semantic and positional metrics operate on distinct numerical ranges.
Therefore, we rank all evaluated tasks separately by each metric, yielding semantic rank $r_S$ and positional rank $r_P$.
We use the relative semantic failure susceptibility $w_S \in [0, 1]$ to show the normalized rank gap: a larger $w_S$ indicates greater susceptibility to semantic reversal relative to positional insensitivity
within the reference set.

\paragraph{Theory-Motivated Interventions (Figure~\ref{fig:diagnostic-workflow}, Step 5).}
The spectral tradeoff driven by the high-frequency norm share motivates a training-free intervention approach with two opposing directions on task-sensitive heads.
Guided by task diagnostic profiles, we assign the upper half of $w_S$ rankings to the semantic direction, attenuating high frequencies or freezing their rotations to suppress position-induced fluctuations~\citep{ICLR2025_e6d58fc6,chen-etal-2025-hope}.
The remaining settings follow the positional direction, using $\alpha > 1$ to amplify adjacent score resolution~\citep{qiao2025rethinking,chiang2025dimension,mikaeili2026untwisting}.
While $w_S$ determines the intervention direction, the specific scaling scalar $\alpha$ is selected from candidate values.
This strategy departs from conventional frequency manipulations~\citep{hua2025fope,xiong2025dope,li2026cope,an2025string} by dynamically assigning the intervention direction per task from its diagnostic profile and restricting modifications strictly to task-sensitive heads.
We report the best observed outcome for each setting, while Appendix~\ref{app:current-intervention} documents complete sweeps across all candidate configurations.

\ifdefined\ArxivFloatFigures
\begin{figure}[htbp]
\else
\begin{figure}[H]
\fi
  \centering
  \includegraphics[width=\linewidth]{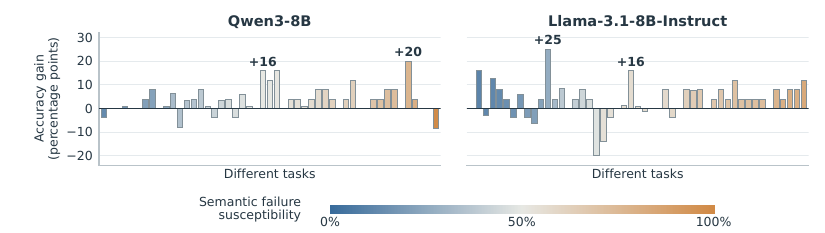}
  \caption{\textbf{Accuracy gains from our exploratory intervention search.}
  Bars show the best observed intervention in each setting's assigned
  semantic or positional direction minus the original
  baseline on identical examples within each setting.
  The baseline is excluded from the candidate maximum.
  Each model's 49 task/configuration settings
  (described in Appendix~\ref{app:diagnostic-task-set}) are ordered from
  left to right by increasing semantic failure susceptibility, from blue to orange.}
  \label{fig:intervention-accuracy-gain}
\end{figure}

\subsection{Empirical Validation Across Diverse Long-Context Tasks}
\label{subsec:current-intervention-results}
\label{subsec:current-intervention}
\label{subsec:current-task-preferences}

\paragraph{Targeted Intervention Gains Reveal Untapped Potential.}
Guided by diagnostic susceptibility profiles, our targeted training-free interventions achieve performance gains across a majority of evaluated tasks (63.3\% of Qwen configurations and 65.3\% of Llama configurations). 
As illustrated in Figure~\ref{fig:intervention-accuracy-gain}, accuracy gains can reach 20\% on Qwen3-8B and 25\% on Llama-3.1-8B-Instruct, showing the practical value of adjusting the spectral contributions identified by our theory.
\newcommand{\TaskPropertiesFigure}[1]{%
\begin{wrapfigure}[#1]{r}
{0.44\textwidth}
\ifdefined\ArxivFloatFigures
\vspace{-20pt}
\else
\vspace{-60pt}
\fi
  \centering
  \includegraphics[width=\linewidth]{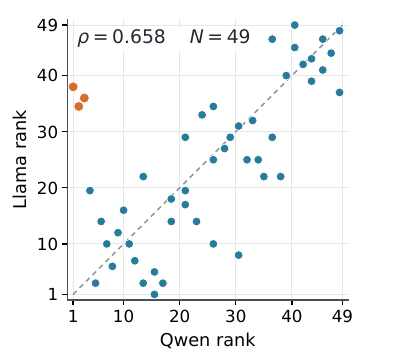}
  \caption{\textbf{Semantic failure susceptibility across tasks.}
  Points compare ranks of semantic failure susceptibility $w_S$ in Qwen3-8B and Llama-3.1-8B-Instruct for the 49
  task/configuration settings described in Appendix~\ref{app:diagnostic-task-set}.
  \textcolor[HTML]{277B9B}{Most tasks} show similar patterns across models, while three
  \textcolor[HTML]{D56E2A}{adjacent-element retrieval tasks} differ.}
  \label{fig:semantic-preference-cross-model}
\end{wrapfigure}%
}
\ifdefined\ArxivFloatFigures
\Needspace{26\baselineskip}
\TaskPropertiesFigure{25}
\fi

\paragraph{RoPE Profiler Reveals Intrinsic Task Properties.}
Beyond intervention gains, Figure~\ref{fig:semantic-preference-cross-model} compares semantic failure susceptibility across the 49 shared settings to inspect how failure patterns relate to underlying task structures.
Across the the \textcolor[HTML]{277B9B}{majority of tasks}, the two models exhibit consistent susceptibility rankings with a Spearman correlation of $\rho = 0.658$, demonstrating that the diagnostic captures intrinsic task properties.
\ifdefined\ArxivFloatFigures\else
\Needspace{0.5\textheight}
\TaskPropertiesFigure{22}
\fi

Specifically, \textit{reasoning tasks cluster in the upper half where semantic reversal dominates, while multi-key retrieval tasks occupy the lower half governed by positional insensitivity}. In contrast, \textcolor[HTML]{D56E2A}{adjacent-element retrieval tasks} deviates sharply between models, ranking in the top three for Qwen but between ranks 34 and 38 for Llama, underscoring that hybrid tasks coupling semantic matching with local order also reflect model-specific representation choices. These patterns align with previous observations on different task settings \citep{goldman2024really,vodrahalli2024michelangelo,NEURIPS2024_babilong,pmlr-v267-modarressi25a,hsieh2024ruler}.

\WFclear
\section{Related Work}
\label{sec:related-work}

\paragraph{RoPE and context extension} RoPE \citep{Su_2024} is widely used in open LLMs \citep{llama3,qwen3,liu2026ministral3}. Existing extension methods involve frequency rescaling, including NTK-aware scaling \citep{bloc97ntk2023}, PI \citep{chen2023extendingcontextwindowlarge}, ABF \citep{xiong2024effective}, YaRN \citep{ICLR2024_874a4d89}, LongRoPE \citep{ding2024longrope}, and Resonance RoPE \citep{wang-etal-2024-resonance}. Others manipulate frequency components \citep{ICLR2025_e6d58fc6, hua2025fope, chen-etal-2025-hope, li2026cope, xiong2025dope}, adjust scaling per head \citep{wang2026adarope}, or reindex positions \citep{an2025string, zhang2024mspoe}. These methods are task-agnostic, while we diagnose task-specific failures before deciding the intervention direction.

\paragraph{Theory of RoPE frequencies} Existing RoPE theories derive base-dependent context bounds \citep{NEURIPS2024_9f12dd32,liu2024scaling,liu2026rotarypositionalembeddingsphase} and roles of low-frequency dimensions \citep{hong2024token,jin2025massive}. Prior works observe a tension between semantic and positional resolution \citep{ICLR2025_e6d58fc6,du2026ropedistinguishespositionstokens,ICLR2026_68f1a693}. These analyses assume specific distribution properties of vector space or amplitude regularity, whereas our bounds hold for real-world activations without additional assumptions.

\paragraph{Internal diagnostics and specialized heads}

Retrieval heads \citep{ICLR2025_9b77f073}, attention sinks \citep{ICLR2024_5e5fd18f}, position-wise features \citep{NEURIPS2024_1403ab1a}, positional attention bias \citep{hsieh-etal-2024-found}, and other internal signals \citep{tan2026contrastiveattributionwildinterpretability} explain long-context behavior inside the model. Our RoPE Profiler also utilizes the cached activations but both identify theoretic vulnerabilities and diagnose with task-specific failure susceptibility. For further discussion, see Appendix~\ref{app:additional-related-work}.

\section{Conclusion}
\label{sec:conclusion}

In this paper, we have studied long-context failures of RoPE through the tradeoff between semantic stability and positional sensitivity.
High-frequency components support fine positional distinctions while introducing fluctuations that can destabilize token preferences.
Our finite-window analysis accommodates unequal coefficient magnitudes across frequencies and, under the stated calibration conditions, yields a context-length bound on jointly avoiding semantic reversal and positional insensitivity.
Building on these insights, we introduced \tool{} to measure both vulnerabilities from cached activations and guide targeted frequency interventions.
Across the evaluated Qwen and Llama settings, we observed broadly consistent patterns of failure susceptibility, and our exploratory intervention search identified accuracy gains on a majority of settings without retraining.
Together, these findings highlight the value of adapting frequency contributions to the needs of each task.
We hope this work informs the design of long-context models that preserve relevant token preferences while resolving the positional distinctions a task requires.

\section*{Acknowledgments}
We thank members of the Alta group at the University of Illinois Urbana-Champaign for their helpful feedback. This work was supported in part by an Amazon AICE award, NSF Grant No. CHE2505932, and a Capital One ASKS award.

\section*{AI Use Statement}
We used AI tools to aid or polish writing, adjust manuscript length and layout, create figures, organize the appendices, and search for related literature. We reviewed all generated sentences and figures, and checked against original sources for the validity of related literature. We take responsibility for the final content of this work. 

\bibliographystyle{iclr2027_conference}
\bibliography{references}

\clearpage
\appendix
\startcontents[appendix]
\begingroup
\hypersetup{linkcolor=black}
\color{black}
\renewcommand{\textcolor}[3][]{#3}
\setlength{\parskip}{0pt}
\begin{center}
  {\LARGE\bfseries Appendix\par}
\end{center}
\vspace{1.2em}
{\Large\scshape Contents\par}
\vspace{1em}
\printcontents[appendix]{appendix}{1}[2]{}
\endgroup

\begingroup
\let\AppendixSection\section
\renewcommand{\section}{\clearpage\AppendixSection}
\section{Formal Theory for the Main Results}
\label{app:formal-theory}

This appendix states the score representation, the finite-window moment
guarantee, the local-response requirement, and the precise calibrated
context-length theorem.  Their proofs are collected in
Appendix~\ref{app:all-theory-proofs}; complementary failure modes appear in
Appendix~\ref{app:deferred-reliability-theory}.

\subsection{Notation for the Deferred Analysis}
\label{app:theory-notation}

We record the general notation used by the proofs and relate it to the
high-frequency score $S_H$ in \S\ref{subsec:partition}.
Consider a fully rotary head with $h\geq1$ rotary coordinate pairs and base
$B>1$. Define
\begin{equation}
\rho:=B^{-1/h},
\qquad
\omega_n:=\rho^n,
\qquad n=0,\ldots,h-1,
\qquad
\cF:=\{0,\ldots,h-1\}.
\label{eq:grid}
\end{equation}
For fixed coefficients $z\in\mathbb C^h$ and a frequency band
$J\subseteq\cF$, write
\begin{equation}
X_J^z(m):=\Rea\sum_{n\in J}z_ne^{im\omega_n},
\qquad
J\subseteq\cF.
\label{eq:score}
\end{equation}
Thus the fully rotary score $S(m)$ in the main text is $X_\cF^z(m)$,
and its high-frequency contribution is $S_H(m)=X_H^z(m)$.
The general band norm is
\begin{equation}
U_J(z):=\left(\sum_{n\in J}|z_n|^2\right)^{1/2},
\qquad
U(z):=U_\cF(z)=\|z\|_2.
\label{eq:norms}
\end{equation}
For an integer window length $M\geq2$, let
\begin{equation}
\avg{f}:=\frac1M\sum_{m=0}^{M-1}f(m).
\label{eq:average}
\end{equation}
The notation $\Var_M(f)$ denotes $\avg{(f-\avg f)^2}$ for real-valued
$f$. Equivalently, these are the mean and variance under
$\mathbf m\sim\operatorname{Unif}\{0,\ldots,M-1\}$.

Let $C_{\mathrm F}>0$ be the absolute Fourier-frame constant in
Lemma~\ref{lem:frame-bound}. For a fixed tolerance $0<\eps<1/2$, define
\begin{equation}
\Gamma_\eps(\rho):=\frac{2C_{\mathrm F}}{(1-\rho)\eps},
\qquad
\cH_\eps(M):=\{n\in\cF:M\omega_n\geq\Gamma_\eps(\rho)\},
\qquad
\cN_\eps(M):=\cF\setminus\cH_\eps(M).
\label{eq:certified-partition}
\end{equation}
A certified window satisfies $M\geq\Gamma_\eps(\rho)$, so the
high-frequency prefix $H=\cH_\eps(M)$ is nonempty. Its complement is
$N=\cN_\eps(M)$. The random variable used in the calibration proof is
\begin{equation}
Y_H^z:=X_H^z(\mathbf m)=S_H(\mathbf m),
\qquad
\mathbf m\sim\operatorname{Unif}\{0,\ldots,M-1\}.
\label{eq:high-window-random-variable}
\end{equation}
For $z\neq0$, the two norm shares are
\begin{equation}
r_H(z;M):=\frac{U_H(z)}{U(z)},
\qquad
r_N(z;M):=\frac{U_N(z)}{U(z)},
\qquad
r_H(z;M)^2+r_N(z;M)^2=1.
\label{eq:high-share}
\end{equation}
These coincide with the main-text high-frequency norm share and its
complement. Their squares are the corresponding energy fractions.

Apply the normalized adjacent response from
\Eqref{eq:current-adjacent-response} to $X_\cF^z$. Failure to attain
response level $\zeta>0$ anywhere in the window means
\begin{equation}
\max_{0\leq m\leq M-2}g_z(m)
=\max_{0\leq m\leq M-2}
\frac{|X_\cF^z(m+1)-X_\cF^z(m)|}{U(z)}
<\zeta.
\label{eq:positional-window-failure}
\end{equation}
For the ordered margin $D(m)=X_\cF^d(m)$ with reference ordering
$D(0)>0$, the coefficient-based and function-based reversal notation
refer to the same probability:
\begin{equation}
p_{\mathrm{rev}}(d;M)
=p_{\mathrm{rev}}(D;M)
:=\mathbb P[D(\mathbf m)<0]
=\mathbb P[X_\cF^d(\mathbf m)<0].
\label{eq:main-pairwise-reversal}
\end{equation}
In particular, the main theorem's $g_d$ is the normalized adjacent
response of this same margin, with denominator $U(d)$.

\subsection{RoPE Score Reduction}
The reduction uses RoPE's relative-rotation identity~\citep{Su_2024}.
\label{app:deferred-rope-score-reduction}

\begin{restatable}[RoPE score reduction]{proposition}{RoPEScoreReduction}
\label{prop:rope-score-representation}
For a fully rotary RoPE head, as the real query and key blocks range over
$\mathbb R^{2h}$, the class of fixed query--key scores and their finite real
linear combinations as functions of relative distance is exactly
$\{X_\cF^z:z\in\mathbb C^h\}$; hence all subsequent full- and bandwise
analyses reduce to studying $X_J^z$.
\end{restatable}

\paragraph{Complex-coordinate derivation.}
Write the two real coordinates of each rotary pair as
$Q_n=q_{n,1}+iq_{n,2}$ and $K_n=k_{n,1}+ik_{n,2}$, and let
$R(\theta)$ denote their planar rotation.  For relative distance
$m=p_q-p_k$ and a position-independent score scale $c_{\mathrm{att}}>0$,
\[
\begin{aligned}
c_{\mathrm{att}}\sum_n
\langle R(p_q\omega_n)q_n,R(p_k\omega_n)k_n\rangle
&=c_{\mathrm{att}}\sum_n q_n^\top R(-m\omega_n)k_n\\
&=\Rea\sum_n
\underbrace{c_{\mathrm{att}}Q_n\overline{K_n}}_{z_n}
e^{im\omega_n}
=X_\cF^z(m).
\end{aligned}
\]
The same calculation applies to score margins by subtracting their
coefficient vectors.  Appendix~\ref{app:proof-rope-score-reduction} gives
the full derivation and the converse construction.

Operationally, the reduction isolates all dependence on relative distance in
a finite trigonometric polynomial.  It preserves the spectral competition
between the diagnostics while leaving content-state changes and downstream
attention computations outside the certificate.

Under the same fixed-content assumption, the argument does not depend on
whether an implementation stores each rotary pair contiguously or in a
split-half layout; only the pairing of real coordinates matters.  A non-rotary
tail then contributes a position-independent constant.  It is omitted from the
formal fully rotary results, while Appendix~\ref{app:prope-interpretation}
discusses its distinct effects on cross-key score margins and same-key
positional comparisons.

\subsection{Finite-Window High-Frequency Calibration}
\label{app:high}
\label{app:high-exact-calibration}

The moment conclusion below is Proposition~\ref{prop:high-calibration}
in the notation of this appendix.  The additional conditional statement
records the Gaussian approximation error used by the extensions.

\begin{restatable}[Finite-window moments and conditional Gaussian calibration]{lemma}{HighBandCalibrationFormal}
\label{lem:high-calibration-formal}
Fix $0<\eps<1/2$, a certified
window $M\geq\Gamma_\eps(\rho)$, and fixed coefficients $z$ with
$U_H(z)>0$. Use $H=\cH_\eps(M)$ from
\Eqref{eq:certified-partition} and the uniformly sampled relative distance
in \Eqref{eq:high-window-random-variable}. For its random score $Y_H^z$, write
\begin{equation}
\mu_H:=\avg{X_H^z},
\qquad
\sigma_H^2:=\Var_M(X_H^z).
\label{eq:high-window-moments}
\end{equation}

Then the following moment bounds hold:
\begin{equation}
\abs{\mu_H}\leq\eps U_H(z),
\qquad
\abs{\sigma_H^2-\tfrac12U_H(z)^2}
\leq\eps U_H(z)^2.
\label{eq:high-calibration}
\end{equation}

Because $0<\eps<1/2$, \Eqref{eq:high-calibration} implies
$\sigma_H^2\geq(1/2-\eps)U_H(z)^2>0$ whenever $U_H(z)>0$.  The standardization
below is therefore well defined.  The remaining, non-deterministic source of
approximation error is the Gaussian-shape discrepancy of the standardized
window law,
\begin{equation}
\Delta_{\mathrm{H,G}}(z;M)
:=
\sup_{x\in\mathbb R}
\abs{
\mathbb P\!\left[
\frac{Y_H^z-\mu_H}{\sigma_H}\leq x
\right]
-\Phi_{\mathrm{std}}(x)
}
\leq\tau_{\mathrm{H,G}}.
\label{eq:high-gaussian-shape-assumption}
\end{equation}
Under this condition, the exact form of the Gaussian conclusion is
\begin{equation}
\sup_{x\in\mathbb R}
\abs{
\mathbb P[Y_H^z\leq x]
-\Phi_{\mathrm{std}}\!\left(\frac{\sqrt2\,x}{U_H(z)}\right)
}
\leq
\tau_{\mathrm{H,G}}+c_{\mathrm G}(\eps),
\label{eq:high-gaussian-calibration}
\end{equation}
where the universal moment-replacement term is
\begin{equation}
c_{\mathrm G}(\eps)
:=
\frac{\eps}{\sqrt\pi}
+\frac{\log((1-2\eps)^{-1})}{2\sqrt{2\pi e}}
=O(\eps)
\qquad(\eps\to0).
\label{eq:high-gaussian-moment-error}
\end{equation}
We use the convention $\mathcal N(\mu,\sigma^2)$: the target variance is
$U_H(z)^2/2$, equivalently the target standard deviation is
$U_H(z)/\sqrt2$.

\end{restatable}

\subsection{Positional Requirements}
\label{subsec:positional}

\subsubsection{Local Positional Response}
\label{subsubsec:local-positional-response}

For a head whose function requires neighboring relative positions to remain
distinguishable, moving a key by one token should induce a non-negligible
normalized score change.  In the terminology of
\S\ref{subsec:reversal-cutoff}, avoiding the window-level failure in
\Eqref{eq:positional-window-failure} at response level $\zeta$ requires at
least one adjacent score change of normalized magnitude $\zeta$.  We now
quantify how maintaining this response constrains the certified-high norm
share as the context scale grows.

\begin{restatable}[Direct scale-to-share response envelope]{lemma}{StableAdjacentResponseEnvelope}
\label{lem:stable-adjacent-response-envelope}
For every certified integer $M\geq2$ and every $a\neq0$,
\begin{equation}
\begin{aligned}
\max_{0\leq m\leq M-2}
\frac{\abs{X_\cF^a(m+1)-X_\cF^a(m)}}{U(a)}
&\leq
\underbrace{
\frac{\Gamma_\eps(\rho)\sqrt h}{M}\,r_N(a;M)
}_{\substack{\text{non-high band:}\\[-1pt]\text{small rotations}}}
+
\underbrace{
R_\cF\,r_H(a;M)
}_{\substack{\text{certified-high band:}\\[-1pt]\text{coefficient norm}}}\\
&\leq
\frac{\Gamma_\eps(\rho)\sqrt h}{M}
+R_\cF\,r_H(a;M).
\end{aligned}
\label{eq:stable-adjacent-response-envelope}
\end{equation}
where
$R_\cF:=\bigl(\sum_{n\in\cF}\abs{e^{i\omega_n}-1}^2\bigr)^{1/2}>0$
is the fixed full-grid adjacent-gain factor and is independent of $M$.
\end{restatable}

The non-high contribution is bounded by a $1/M$ envelope because every
non-high frequency rotates by less than $\Gamma_\eps(\rho)/M$ per token.
The second term has no explicit inverse-scale decay: its full-grid gain factor
is independent of $M$, but the certified-high norm share can still vary with
the window.

\begin{restatable}[Stable local-response floor on the certified-high norm
share]{corollary}{NormalizedPositionShareFloor}
\label{cor:normalized-positional-share-floor}
For every certified integer $M\geq2$, every $a\neq0$, and every $\zeta>0$,
if the score $X_\cF^a$ avoids the window-level failure in
\Eqref{eq:positional-window-failure} at response level $\zeta$, then
\begin{equation}
r_H(a;M)
\geq
\underline r_{\mathrm{pos}}(M;\zeta)
:=
\frac{
\left[\zeta-\Gamma_\eps(\rho)\sqrt h/M\right]_+
}{R_\cF}.
\label{eq:normalized-positional-share-floor}
\end{equation}
where $R_\cF$ is the fixed factor in
Lemma~\ref{lem:stable-adjacent-response-envelope}.
If the right-hand side exceeds one, the requested response is impossible.
\end{restatable}

Together, Lemma~\ref{lem:stable-adjacent-response-envelope} and
Corollary~\ref{cor:normalized-positional-share-floor} directly link the
context scale $M$ to the certified-high norm share: maintaining the same
adjacent response level $\zeta$ forces $r_H(a;M)$ above
$\underline r_{\mathrm{pos}}(M;\zeta)$.  This floor is the restriction to
certified integer windows of a continuous, non-decreasing function of real
$M>0$, and it is strictly increasing once positive.  Thus it provides a
stable non-decreasing lower bound even though $H(M)$ changes only at discrete
breakpoints.  Enlarging $H(M)$ alone would not give this conclusion, since a
newly certified coefficient may be zero and the actual share need not jump.
The continuity and monotonicity claim concerns the necessary floor, not the
actual stepwise quantity $r_H(a;M)$ or its increment between two windows.  No
prior assumption $U_H(a)>0$ is needed: a positive floor itself forces that
conclusion.  Sharper coefficient-specific refinements are given in
Appendix~\ref{app:exact-gain-response-envelope}.

The response floor applies to the same ordered margin used in the
context-length result below.

\subsection{Precise Context-Length Theorem and Its Explicit Bound}
\label{app:reversal-cutoff}
\label{app:reversal-cutoff-setup}

Fix a fully rotary head, $0<\eps<1/2$, and the same nonzero ordered
margin coefficients $d$ with $D(0)>0$ across a finite ordered set of
certified windows $\mathfrak M_{\mathrm{cal}}$. Fix tolerances
$0\leq\delta<1/2$ and $0<\zeta<R_\cF$, together with a uniform Gaussian
threshold-error bound $0\leq\overline\tau_{\mathrm G}<1/2$.
For the ordered margin $d$ in \S\ref{subsec:reversal-cutoff}, set
\begin{equation}
\mu_D(M):=\avg{X_\cF^d},
\qquad
\sigma_D^2(M):=\Var_M(X_\cF^d),
\label{eq:reversal-exact-moments}
\end{equation}
and assume $\sigma_D(M)>0$ at every window in
$\mathfrak M_{\mathrm{cal}}$. The exact-moment Gaussian threshold error is
\begin{equation}
\Delta_{\mathrm{rev}}(d;M)
:=
\abs{
p_{\mathrm{rev}}(d;M)
-\Phi_{\mathrm{std}}\!\left(-\frac{\mu_D(M)}{\sigma_D(M)}\right)
}.
\label{eq:reversal-threshold-error}
\end{equation}
The calibrated family in Theorem~\ref{thm:maximum-reliable-context} satisfies
$\Delta_{\mathrm{rev}}(d;M)\leq\overline\tau_{\mathrm G}$ uniformly.

Let
\begin{equation}
k_N(M):=\abs{\cN_\eps(M)},
\qquad
c_\eps:=\sqrt{\frac12-\eps},
\label{eq:reversal-envelope-count}
\end{equation}
and, for $0\leq r\leq1$, define
\begin{equation}
\begin{aligned}
u_M(r)
&:=\eps r+\sqrt{k_N(M)}\sqrt{1-r^2},\\
v_M(r)
&:=\left[
c_\eps r-\sqrt{k_N(M)}\sqrt{1-r^2}
\right]_+.
\end{aligned}
\label{eq:reversal-envelope-functions}
\end{equation}
Recall the full-grid adjacent-gain bound
$R_\cF=(\sum_{n\in\cF}|e^{i\omega_n}-1|^2)^{1/2}>0$.
The main theorem assumes $0<\zeta<R_\cF$.
By Corollary~\ref{cor:normalized-positional-share-floor}, retaining
response level $\zeta$ for the same margin $d$ requires
\begin{equation}
r_H(d;M)\geq\ell(M;\zeta)
:=\frac{[\zeta-\Gamma_\eps(\rho)\sqrt h/M]_+}{R_\cF}.
\label{eq:main-local-response-floor}
\end{equation}
Here $[x]_+=\max\{x,0\}$. The stated range of $\zeta$ ensures
$0\leq\ell(M;\zeta)<1$, within the domain of the reversal envelope.

The Gaussian reversal envelope and the floor used in the main text are
\begin{align}
\psi_M(r;\overline\tau_{\mathrm G})
&:=
\begin{cases}
\left[
\Phi_{\mathrm{std}}\!\left(-u_M(r)/v_M(r)\right)
-\overline\tau_{\mathrm G}
\right]_+,&v_M(r)>0,\\
0,&v_M(r)=0,
\end{cases}
\label{eq:reversal-envelope-probability}\\
\underline p_{\mathrm{rev}}(M;\zeta,\overline\tau_{\mathrm G})
&:=
\psi_M\!\left(\ell(M;\zeta);\overline\tau_{\mathrm G}\right).
\label{eq:gaussian-reversal-floor}
\end{align}
For fixed $\zeta$ and $\overline\tau_{\mathrm G}$, we abbreviate this
bound as $\underline p_{\mathrm{rev}}(M)$.

For every calibrated window at which the margin retains response level
$\zeta$, the lower-bound chain underlying
Theorem~\ref{thm:maximum-reliable-context} is
\begin{equation}
p_{\mathrm{rev}}(d;M)
\geq\psi_M\!\left(r_H(d;M);\overline\tau_{\mathrm G}\right)
\geq\psi_M\!\left(\ell(M;\zeta);\overline\tau_{\mathrm G}\right)
=\underline p_{\mathrm{rev}}
(M;\zeta,\overline\tau_{\mathrm G}).
\label{eq:main-reversal-floor}
\end{equation}
Appendix~\ref{app:proof-maximum-reliable-context} proves this chain
and its monotonicity in $M$.

When the crossing set is nonempty, the first excluded scale is
\begin{equation}
M_\dagger:=\min\left\{M\in\mathfrak M_{\mathrm{cal}}:
\underline p_{\mathrm{rev}}(M;\zeta,\overline\tau_{\mathrm G})>\delta
\right\}.
\label{eq:main-reversal-crossing-scale}
\end{equation}
For every $M\in\mathfrak M_{\mathrm{cal}}$ with $M\geq M_\dagger$,
\begin{equation}
p_{\mathrm{rev}}(D;M)>\delta
\qquad\text{or}\qquad
\max_{0\leq m\leq M-2}g_d(m)<\zeta.
\label{eq:appendix-joint-incompatibility}
\end{equation}
If the crossing set is empty, the bound supplies no exclusion threshold
within the chosen calibrated family.

If the sub-threshold set is nonempty, define
\begin{equation}
M_{\max}^{\mathrm{cert}}
:=\max\left\{
M\in\mathfrak M_{\mathrm{cal}}:
\underline p_{\mathrm{rev}}
(M;\zeta,\overline\tau_{\mathrm G})\leq\delta
\right\}.
\label{eq:main-certified-maximum-context}
\end{equation}
This is the largest calibrated scale not ruled out by the lower bound.
When the crossing set in \Eqref{eq:main-reversal-crossing-scale} is also
nonempty, it is the grid point immediately preceding $M_\dagger$.
The main text writes $M_{\max}=M_{\max}^{\mathrm{cert}}$ and uses the
regime in which both the crossing and sub-threshold sets are nonempty.
For calibrated windows, $M>M_{\max}$ is then equivalent to
$M\geq M_\dagger$. All such windows fail at least one requirement,
which gives the simultaneous-infeasibility statement in
\Eqref{eq:main-incompatibility-implication}.
A scale below the crossing remains a candidate: the lower bound alone
does not establish joint reliability there.
This count-based envelope is deliberately conservative.  It uses no
coefficient information beyond the certified-high share; the practical
predictor in Appendix~\ref{app:current-analytic-moments} computes the exact
full-window mean and variance, including cross-band covariance.

\newpage

\newtheorem{formalcontexttheorem}{Theorem}
\renewcommand{\theformalcontexttheorem}{\getrefnumber{thm:maximum-reliable-context}}
\begin{restatable}[Precise context-length limit on joint reliability]{formalcontexttheorem}{MaximumReliableContextFormal}
\label{thm:maximum-reliable-context-formal}
Fix a fully rotary head with the geometric frequency grid in
\Eqref{eq:grid}, a tolerance $0<\eps<1/2$, and a nonzero ordered margin
$D(m)=X_\cF^d(m)$ with $D(0)>0$.  Let
$\mathfrak M_{\mathrm{cal}}$ be a finite nonempty set of certified integer
windows $M\geq2$, with the same coefficients $d$ at every window.
Fix $0\leq\delta<1/2$, $0<\zeta<R_\cF$, and
$0\leq\overline\tau_{\mathrm G}<1/2$.  Assume, for every
$M\in\mathfrak M_{\mathrm{cal}}$, that $\sigma_D(M)>0$ and
\[
\Delta_{\mathrm{rev}}(d;M)
=\left|p_{\mathrm{rev}}(d;M)
-\Phi_{\mathrm{std}}\!\left(-\frac{\mu_D(M)}{\sigma_D(M)}\right)\right|
\leq\overline\tau_{\mathrm G}.
\]
Let $\underline p_{\mathrm{rev}}(M;\zeta,\overline\tau_{\mathrm G})$
be the explicit bound in \Eqref{eq:gaussian-reversal-floor}.
This bound is non-decreasing across $\mathfrak M_{\mathrm{cal}}$, and
each calibrated window retaining the adjacent response
$\max_{0\leq m\leq M-2}g_d(m)\geq\zeta$ satisfies
\[
p_{\mathrm{rev}}(d;M)
\geq\underline p_{\mathrm{rev}}(M;\zeta,\overline\tau_{\mathrm G}).
\]
Write $M_-:=\min\mathfrak M_{\mathrm{cal}}$ and
$M_+:=\max\mathfrak M_{\mathrm{cal}}$.  If the reversal tolerance satisfies
\begin{equation}
\underline p_{\mathrm{rev}}(M_-)\leq\delta
<\underline p_{\mathrm{rev}}(M_+),
\label{eq:reversal-admissible-threshold-range}
\end{equation}
define $M_{\max}:=M_{\max}^{\mathrm{cert}}$ by
\Eqref{eq:main-certified-maximum-context}.  Then, for every
$M\in\mathfrak M_{\mathrm{cal}}$ with $M>M_{\max}$,
\[
p_{\mathrm{rev}}(d;M)>\delta
\qquad\text{or}\qquad
\max_{0\leq m\leq M-2}g_d(m)<\zeta.
\]
Thus $M_{\max}$ is an upper bound on joint score-level reliability within
the calibrated family.  A window at or below $M_{\max}$ remains a candidate;
the bound alone does not establish joint reliability there.

An explicit sufficient condition for the strict upper inequality in
\Eqref{eq:reversal-admissible-threshold-range} is
$M_+\geq M_{\mathrm{suff}}$ and $0\leq\delta<\delta_\star$, where
\begin{equation}
\delta_\star
:=\left[
\Phi_{\mathrm{std}}\!\left(-\frac{\eps}{\sqrt{1/2-\eps}}\right)
-\overline\tau_{\mathrm G}
\right]_+,
\label{eq:reversal-threshold-ceiling}
\end{equation}
\begin{equation}
\begin{aligned}
\omega_{\min}&:=\rho^{h-1},\\
M_{\mathrm{suff}}&:=\max\left\{
2,\left\lceil\frac{\Gamma_\eps(\rho)}{\omega_{\min}}\right\rceil,
\left\lfloor\frac{\Gamma_\eps(\rho)\sqrt h}{\zeta}\right\rfloor+1
\right\}.
\end{aligned}
\label{eq:reversal-sufficient-window}
\end{equation}
Indeed, $\underline p_{\mathrm{rev}}(M)\leq\delta_\star<1/2$ at every
calibrated window, with equality for each such window
$M\geq M_{\mathrm{suff}}$.  The lower inequality
$\underline p_{\mathrm{rev}}(M_-)\leq\delta$ remains required.
\end{restatable}

\paragraph{Interpreting the reversal tolerance.}
The tolerance $\delta$ specifies the allowed fraction of relative distances
at which the reference ordering reverses under uniform sampling.  For
example, $\delta=0.38$ permits reversals at up to $38\%$ of the positions
in the window. Preserving reference orderings reliably calls for a much
smaller tolerated fraction, so an upper restriction on $\delta$ is consistent
with the intended semantic-stability requirement. The range
in \Eqref{eq:reversal-admissible-threshold-range} identifies tolerances
for which this bound certifies an exclusion threshold.  Its use still
requires the stated calibration conditions and the specified window family.

\section{Extensions to Other Failure Modes}
\label{app:extended-theory}
\label{app:deferred-reliability-theory}

The score representation in Appendix~\ref{app:formal-theory} also supports
near--far comparisons of one key, exceedance of a fixed score barrier,
fully-high pairwise reversal, and a partial-RoPE interpretation.
Each result below retains its own event and calibration conditions.
Proofs are collected in Appendix~\ref{app:all-theory-proofs}.

\subsection{Near--Far Ordering}
\label{app:long-range-orderability}

\subsubsection{Extended Positional Diagnostic: Near--Far Ordering}
\label{app:near-far-definition}

Adjacent response measures local resolution, but it does not show whether a
head preserves a directional notion of distance at larger scales.  For a head
or task that requires monotone near preference, we orient the score so that a
farther occurrence should receive a lower score.  For an
integer $M\geq1$, independently sample
$m_{\mathrm{near}}$ from $\{M+1,\ldots,2M\}$ and
$m_{\mathrm{far}}$ from the equally long, more distant interval
$\{2M+1,\ldots,3M\}$.  For a fixed key $K$, define the ordering-error
probability
\begin{equation}
p_{\mathrm{ord}}(K;M)
:=
\mathbb P[
S_K(m_{\mathrm{far}})>S_K(m_{\mathrm{near}})].
\label{eq:long-range-ordering-probability}
\end{equation}
A reliable near--far order requires this probability to be small.  When it is
close to one half, the intended inequality is violated in half of the random
comparisons under this audit protocol.  We call this \emph{near--far order
ambiguity}; it is a score-level diagnostic of chance-like directional ordering
in the two-interval comparison, not a claim that all positional information
has disappeared.

This is a one-sided diagnostic: a small $p_{\mathrm{ord}}$ is necessary but
not sufficient for reliable near preference, because a position-independent
score also makes the strict event impossible.  Fixed-offset, periodic,
diagonal, and other non-monotone patterns may likewise remain detectable.
\subsubsection{Near--Far Order Ambiguity}
\label{subsubsec:long-range-orderability}

We apply the two-interval experiment from
Appendix~\ref{app:near-far-definition} to one fixed key.  The theorem below concerns
the asymptotic fully-high calibration regime: when all rotating frequencies
enter the certified-high band and the interval calibration errors vanish, the
near--far ordering error approaches chance.  Appendix~\ref{app:two-interval-calibration-quantities}
gives the precise calibration conditions.

\begin{restatable}[Near--far ordering approaches chance under full calibration]{theorem}{LongRangeOrderAmbiguity}
\label{thm:long-range-order-ambiguity}
Fix a fully rotary head together with fixed query and key content states that
induce a nonzero coefficient vector $a$.  Independently draw
$m_{\mathrm{near}}$ uniformly from $\{M+1,\ldots,2M\}$ and
$m_{\mathrm{far}}$ uniformly from $\{2M+1,\ldots,3M\}$.  Suppose the two
interval score laws lie in the asymptotic fully-high calibration regime of
Appendix~\ref{app:two-interval-calibration-quantities}.  Then
\begin{equation}
p_{\mathrm{ord}}(a;M)
:=
\mathbb P[
X_\cF^a(m_{\mathrm{far}})>X_\cF^a(m_{\mathrm{near}})]
\longrightarrow\frac12
\qquad(M\to\infty).
\label{eq:long-range-order-chance-limit}
\end{equation}
\end{restatable}

The slowly rotating part can initially support a signed mean advantage for
the nearer interval.  As those frequencies enter the certified-high band,
Proposition~\ref{prop:high-calibration} drives both interval means toward zero
and both variances toward the same half-energy value.  Under the stated shape
calibration, the two scores therefore approach the same distribution, so the
farther score exceeds the nearer one with probability one half.  The
vanishing shape error remains substantive: equalized moments alone do not
make a fixed finite trigonometric sum Gaussian.
\subsubsection{Two-Interval Calibration Error}
\label{app:two-interval-calibration-quantities}

For an integer shift $s$, define
\begin{equation}
a_n^{[s]}:=a_ne^{is\omega_n},
\qquad
X_\cF^a(s+t)=X_\cF^{a^{[s]}}(t),
\qquad
U(a^{[s]})=U(a).
\label{eq:shifted-interval-coefficients}
\end{equation}
The near and far intervals are length-$M$ windows generated by
$a^{[M+1]}$ and $a^{[2M+1]}$, respectively.  Define
\begin{equation}
\tau_{\mathrm{near}}(a;M)
:=\Delta_{\mathrm{H,G}}(a^{[M+1]};M),
\qquad
\tau_{\mathrm{far}}(a;M)
:=\Delta_{\mathrm{H,G}}(a^{[2M+1]};M),
\label{eq:two-interval-shape-errors}
\end{equation}
and
\begin{equation}
\tau_{\mathrm{ord}}(a;M,\eps)
:=
\tau_{\mathrm{near}}(a;M)
+\tau_{\mathrm{far}}(a;M)
+2c_{\mathrm G}(\eps).
\label{eq:two-interval-total-error}
\end{equation}
For fixed finite $h$, set
\begin{equation}
\eps_M
:=
\frac{2C_{\mathrm F}}{(1-\rho)M\omega_{h-1}}
\label{eq:long-range-vanishing-epsilon-app}
\end{equation}
for all sufficiently large $M$.  The asymptotic fully-high calibration regime
in Theorem~\ref{thm:long-range-order-ambiguity} imposes two requirements.
Both standardized translated laws must satisfy
\Eqref{eq:high-gaussian-shape-assumption} with $\eps=\eps_M$, and
$\tau_{\mathrm{ord}}(a;M,\eps_M)\to0$.  When available, valid held-out upper
bounds on the two shape errors in \Eqref{eq:two-interval-shape-errors} may be
used in \Eqref{eq:two-interval-total-error}.  This tolerance choice makes the
partition fully high because
$\Gamma_{\eps_M}(\rho)=M\omega_{h-1}$, so
$\cH_{\eps_M}(M)=\cF$.

\subsection{Score-Barrier Stability}
\label{subsec:semantic}
\label{app:single-key-exceedance}

\subsubsection{Single-Key Exceedance Against a Score Barrier}
\label{subsubsec:single-key-exceedance}

We use rare exceedance of a fixed barrier as a single-key score-level
certificate.  It acquires a semantic ranking interpretation only under the
winner-barrier construction below and its additional tail conditions.

Fix the query, layer, head, content states, and one key $K$ with nonzero RoPE
coefficient vector $b$.  For a certified window, draw
$\mathbf m\sim\operatorname{Unif}\{0,\ldots,M-1\}$ and, for a fixed score
barrier $S\in\mathbb R$, define
\begin{equation}
Y_b:=X_\cF^b(\mathbf m),
\qquad
p_{\mathrm{exc}}(b;S,M):=\mathbb P[Y_b>S].
\label{eq:single-key-exceedance-probability}
\end{equation}
Write $Y_N:=X_N^b(\mathbf m)$ and $Y_H:=X_H^b(\mathbf m)$.  Guided by
Proposition~\ref{prop:high-calibration}, retain the non-high moments and the
cross-band covariance, but replace the certified-high mean and variance by
$0$ and $U_H(b)^2/2$.  This gives the Gaussian exceedance predictor
\begin{equation}
\begin{aligned}
\widehat\sigma_b^2
&:=\Var_M(Y_N)+\frac12U_H(b)^2+2\Cov_M(Y_N,Y_H),\\
\widehat p_{\mathrm{exc}}(b;S,M)
&:=\Phi_{\mathrm{std}}\!\left(
\frac{\avg{X_N^b}-S}{\widehat\sigma_b}
\right),
\end{aligned}
\label{eq:single-key-exceedance-predictor}
\end{equation}
where $\Cov_M$ is covariance under the same uniform window and
$\widehat\sigma_b$ is the positive square root.  Let
$\tau_{\mathrm{exc}}(b;S,M)$ denote the total calibration error defined in
Appendix~\ref{app:single-key-calibration-quantities}.  It combines the
Gaussian-shape discrepancy of the full score with the high-band moment
replacement error; the result below is informative when this measurable
quantity is small.

\begin{restatable}[Calibrated single-key score exceedance]{lemma}{SingleKeyExceedance}
\label{lem:single-key-exceedance}
For every certified $M$, nonzero $b$, and $S\in\mathbb R$ in the
non-degenerate calibration regime of
Appendix~\ref{app:single-key-calibration-quantities},
\begin{equation}
\abs{
p_{\mathrm{exc}}(b;S,M)
-\widehat p_{\mathrm{exc}}(b;S,M)
}
\leq\tau_{\mathrm{exc}}(b;S,M).
\label{eq:single-key-exceedance-calibration}
\end{equation}
\end{restatable}

Let $u_{\mathrm{exc}}(b;S,M,\delta)$ denote the explicit,
coefficient-specific ceiling given in
Appendix~\ref{app:single-key-ceiling-candidate}.  Its formula is deferred
because only the resulting bound, rather than the quadratic inversion used to
obtain it, is needed here.

\begin{restatable}[Small score-barrier exceedance imposes a certified-high
norm-share ceiling]{corollary}{SingleKeyHighShareCeiling}
\label{cor:single-key-high-share-ceiling}
Under the hypotheses of Lemma~\ref{lem:single-key-exceedance}, assume
$U_N(b)U_H(b)>0$ and
$0<\delta+\tau_{\mathrm{exc}}(b;S,M)<1/2$.  Then
\begin{equation}
p_{\mathrm{exc}}(b;S,M)\leq\delta
\quad\Longrightarrow\quad
r_H(b;M)\leq u_{\mathrm{exc}}(b;S,M,\delta)<1.
\label{eq:single-key-high-share-ceiling-conclusion}
\end{equation}
\end{restatable}

This is exactly the same $r_H(b;M)=U_H(b)/U(b)$ constrained from below by the
local-response floor when it is applied to $b$.  Hence, within the calibration
and interior hypotheses above, a local-response floor above $u_{\mathrm{exc}}$
forces $p_{\mathrm{exc}}(b;S,M)>\delta$.  The endpoint cases are different:
$U_H(b)=0$ gives $r_H=0$ directly, whereas at $U_N(b)=0$ rare exceedance of a
sufficiently high barrier need not imply any ceiling below one.

\paragraph{Winner-barrier interpretation.}
Let $\mathcal K$ be a finite key bank at the reference placement $m=0$, choose
any reference winner, and define
\begin{equation}
K_+\in\arg\max_{K_j\in\mathcal K}X_\cF^{b_j}(0),
\qquad
S_{\max}:=X_\cF^{b_+}(0),
\qquad
\alpha_+(S_{\max};M)
:=\mathbb P[X_\cF^{b_+}(\mathbf m)>S_{\max}].
\label{eq:winner-score-barrier}
\end{equation}
For every competitor $K_j\neq K_+$, event inclusion gives
\begin{equation}
p_{\mathrm{rev}}(b_+,b_j;M)
:=\mathbb P[X_\cF^{b_j}(\mathbf m)>X_\cF^{b_+}(\mathbf m)]
\geq
\left[
p_{\mathrm{exc}}(b_j;S_{\max},M)-\alpha_+(S_{\max};M)
\right]_+.
\label{eq:winner-barrier-reversal-lower-bound}
\end{equation}
Consequently, fix $q>0$ and suppose
$\alpha_+(S_{\max};M)\leq\delta-q$.  If every competitor lies in the
calibrated interior regime above and its local-response floor exceeds
$u_{\mathrm{exc}}(b_j;S_{\max},M,\delta)$, then each competitor has the
marginal guarantee
$p_{\mathrm{rev}}(b_+,b_j;M)>q$.  Thus choosing $S$ as the largest reference
score exposes a competitor-wise marginal reversal risk across the bank once
all candidate-specific intervals are incompatible.  It does not assert that
all competitors reverse at the same relative distance, nor does the maximum-barrier
choice alone imply the result: the calibration, interval crossing, and
winner-tail condition must all hold.

The score-barrier ceiling above applies in the calibrated interior regime.  The fully-high endpoint is covered by the ordered-pair
result in Appendix~\ref{subsubsec:fully-high-reversal}.

\subsubsection{Calibration Quantities}
\label{app:single-key-calibration-quantities}

For the fixed score in \S\ref{subsubsec:single-key-exceedance}, write
$Y_N:=X_N^b(\mathbf m)$, $Y_H:=X_H^b(\mathbf m)$, and set
\begin{align}
\mu_b&:=\avg{X_\cF^b},
&\sigma_b^2&:=\Var_M(X_\cF^b),\notag\\
\mu_N&:=\avg{X_N^b},
&\sigma_N^2&:=\Var_M(X_N^b),\notag\\
\kappa_{NH}&:=
\avg{(X_N^b-\mu_N)(X_H^b-\avg{X_H^b})},\notag\\
\widehat\sigma_b^2&:=
\sigma_N^2+\frac12U_H(b)^2+2\kappa_{NH}.
\label{eq:single-key-moment-descriptors}
\end{align}
Thus $\widehat\sigma_b^2$ is exactly the moment-calibrated variance proxy in
\Eqref{eq:single-key-exceedance-predictor}.  The non-degenerate calibration
regime used in Lemma~\ref{lem:single-key-exceedance} is
\begin{equation}
\sigma_b^2>0,
\qquad
\widehat\sigma_b^2>\eps U_H(b)^2.
\label{eq:single-key-nondegenerate-regime}
\end{equation}
In this regime, define the full-score Gaussian-shape discrepancy
\begin{equation}
\Delta_{\mathrm{exc}}(b;M)
:=\sup_{x\in\mathbb R}
\abs{
\mathbb P\!\left[
\frac{Y_b-\mu_b}{\sigma_b}\leq x
\right]-\Phi_{\mathrm{std}}(x)
}
\label{eq:single-key-shape-error}
\end{equation}
and the variance floor and total calibration error
\begin{align}
s_-(b;M)&:=
\sqrt{\widehat\sigma_b^2-\eps U_H(b)^2},\notag\\
\tau_{\mathrm{exc}}(b;S,M)
&:=\Delta_{\mathrm{exc}}(b;M)
+\frac1{\sqrt{2\pi}}
\left[
\frac{\eps U_H(b)}{s_-(b;M)}
+\frac{\abs{S-\mu_N}\eps U_H(b)^2}
{s_-(b;M)\widehat\sigma_b
 \bigl(s_-(b;M)+\widehat\sigma_b\bigr)}
\right].
\label{eq:single-key-total-error}
\end{align}
The first term measures the deviation of the \emph{full} score law from
Gaussianity; the remaining terms account for replacing the high-band moments
using Proposition~\ref{prop:high-calibration}.
\subsubsection{Explicit Ceiling Candidate}
\label{app:single-key-ceiling-candidate}

For the interior case $U_N(b)U_H(b)>0$ and
$0<\delta+\tau_{\mathrm{exc}}(b;S,M)<1/2$, define
\begin{align}
\lambda_S&:=\frac{S-\mu_N}{U_N(b)},
&\beta_N&:=\frac{\sigma_N^2}{U_N(b)^2},\notag\\
\gamma_{NH}&:=\frac{\kappa_{NH}}{U_N(b)U_H(b)},
&
\xi_{\mathrm{exc}}
&:=\Phi_{\mathrm{std}}^{-1}\!\left(
1-\delta-\tau_{\mathrm{exc}}(b;S,M)
\right).
\label{eq:single-key-normalized-descriptors}
\end{align}
The corresponding explicit ceiling candidate is
\begin{align}
\overline t_{\mathrm{exc}}
&:=-2\gamma_{NH}
+\sqrt{
4\gamma_{NH}^2
+2\left(
\frac{\lambda_S^2}{\xi_{\mathrm{exc}}^2}-\beta_N
\right)},\notag\\
u_{\mathrm{exc}}(b;S,M,\delta)
&:=\frac{\overline t_{\mathrm{exc}}}
{\sqrt{1+\overline t_{\mathrm{exc}}^2}}.
\label{eq:single-key-high-share-ceiling}
\end{align}
Its radicand need not be nonnegative for arbitrary inputs.  Under the
antecedent of Corollary~\ref{cor:single-key-high-share-ceiling}, the induced
high-to-non-high norm ratio is feasible for the resulting quadratic; this
guarantees that the
candidate is real and positive.

\subsection{Fully-High Pairwise Reversal}
\label{subsubsec:fully-high-reversal}
\label{app:long-context-reversal}

For an ordered key pair, we call the event that the first score falls below
the second a directional reversal; it is semantic when the first key is the
preferred candidate.

\begin{restatable}[Fully-high pairwise reversal approaches chance]{theorem}{FullyHighReversal}
\label{thm:fully-high-reversal}
For a fully rotary head, fix a query and an ordered pair of keys with coefficient vectors
$a,b\in\mathbb C^h$ such that $d:=a-b\neq0$.  For any $0<\eps<1/2$ and any
certified window satisfying
$\cH_\eps(M)=\cF$, draw
$\mathbf m\sim\operatorname{Unif}\{0,\ldots,M-1\}$.  If the Gaussian-shape
condition in \Eqref{eq:high-gaussian-shape-assumption} holds for $z=d=a-b$ with
tolerance $\tau_{\mathrm{H,G}}(d;M)$, then
\begin{equation}
\begin{aligned}
p_{\mathrm{rev}}(a,b;M)
&:=\mathbb P\!\left[X_\cF^a(\mathbf m)<X_\cF^b(\mathbf m)\right]
=\mathbb P\!\left[X_\cF^{a-b}(\mathbf m)<0\right],\\
\abs{p_{\mathrm{rev}}(a,b;M)-\frac12}
&\leq\tau_{\mathrm{H,G}}(a-b;M)+c_{\mathrm G}(\eps).
\end{aligned}
\label{eq:fully-high-pairwise-reversal}
\end{equation}
Here $c_{\mathrm G}(\eps)$ is the universal moment-replacement term defined in
\Eqref{eq:high-gaussian-moment-error}.
For fixed finite $h$, define
$\eps_M:=2C_{\mathrm F}/((1-\rho)M\omega_{h-1})$ for all sufficiently large
$M$.  If the same shape calibration holds for the resulting fully-high
partitions and $\tau_{\mathrm{H,G}}(a-b;M)\to0$, then
\begin{equation}
\lim_{M\to\infty}p_{\mathrm{rev}}(a,b;M)=\frac12.
\label{eq:fully-high-reversal-limit}
\end{equation}
\end{restatable}

Fully-high membership controls moments but does not by itself make a fixed
finite trigonometric sum Gaussian or guarantee chance-level reversal at a
finite window.  The theorem provides a calibrated error bound and an
asymptotic conclusion only under the stated shape condition.  The main theorem
instead combines the local-response floor with a conservative finite-window
reversal floor for the same ordered score margin.

\paragraph{Scope of the asymptotic premise.}
For fixed finite $h$ and fixed nonzero $d$, the score satisfies
$|X_\cF^d(m)|\leq\sum_n|d_n|$ at every relative distance. Its limiting variance
under the fully-high moment bounds is $U(d)^2/2>0$. Thus its standardized
support remains bounded, while a Gaussian has positive tails beyond that
bound. The all-threshold Gaussian-shape discrepancy therefore cannot vanish
solely by increasing $M$ in this fixed-dimensional setting. The vanishing-error
implication above is retained as a conditional statement from the earlier
analysis; the present paper uses the finite-window error bound and the
finite-window context-scale theorem. The accepted small finite-window
Gaussian approximation is consistent with this distinction.

\subsection{Interpretation for p-RoPE}
\label{app:prope-interpretation}
The main reversal-cutoff theorem and the reversal results above concern
fully rotary RoPE.  In
a partial-RoPE variant~\citep{ICLR2025_e6d58fc6}, the score of a fixed key with fixed content states can
be written as
\[
S_K^{\mathrm p}(m)=c_K+X_{\mathcal R}^{b_{\mathcal R}}(m),
\]
where $\mathcal R$ indexes the retained rotating dimensions and the NoPE
component $c_K$ is independent of relative distance.  This constant does not
cancel against a fixed score barrier:
\[
p_{\mathrm{exc}}^{\mathrm p}(K;S,M)
=
\mathbb P[
X_{\mathcal R}^{b_{\mathcal R}}(\mathbf m)>S-c_K].
\]
Depending on its sign, a sufficiently strong NoPE anchor can therefore relax
or tighten the score-barrier requirement after the ceiling is recalibrated with
the effective barrier $S-c_K$.  Likewise, a favorable NoPE margin adds a
nonzero DC term to the cross-key score difference, so the fully rotary
one-half conclusion of Theorem~\ref{thm:fully-high-reversal} need not apply.
Thus p-RoPE may avoid a particular score-barrier or fully rotary reversal
cutoff; the magnitude of its NoPE share alone,
without the sign of the induced score margin, is not sufficient.

The same constant cancels exactly from the positional interval comparison:
\[
S_K^{\mathrm p}(m_{\mathrm{near}})
-S_K^{\mathrm p}(m_{\mathrm{far}})
=
X_{\mathcal R}^{b_{\mathcal R}}(m_{\mathrm{near}})
-X_{\mathcal R}^{b_{\mathcal R}}(m_{\mathrm{far}}).
\]
NoPE dimensions may be viewed as zero-frequency dimensions and hence never
become certified high, but for the same reason they supply no changing signal
with which to order the two intervals.  If the retained rotating component is
nonzero and satisfies all hypotheses of
Theorem~\ref{thm:long-range-order-ambiguity}, including full certification and
two-interval calibration, the same bound applies to that component.  If no
rotating component remains, the two sampled scores are identical and
$p_{\mathrm{ord}}=0$; this is positional indifference, not a reliable
preference for the near interval.  Hence an unrotated anchor can support
score-barrier stability, but it cannot by itself provide long-distance
positional discrimination.

\section{Proofs of the Theoretical Results}
\label{app:all-theory-proofs}

\subsection{Proof of Proposition~\ref{prop:rope-score-representation}}
\label{app:proof-rope-score-reduction}

\RoPEScoreReduction*

\begin{proof}
The proof has two directions.  We first rewrite each rotary plane in complex
form and use real linearity; we then construct real rotary blocks for an
arbitrary coefficient vector.

Let $q,k_+,k_-\in\mathbb R^{2h}$ denote one pre-RoPE query and two pre-RoPE keys.
For each frequency $n$, group the corresponding real rotary coordinates into
the blocks
$q_n:=(q_{n,1},q_{n,2})^\top$ and
$(k_\pm)_n:=((k_\pm)_{n,1},(k_\pm)_{n,2})^\top$, and represent these blocks by
\begin{equation}
Q_n:=q_{n,1}+iq_{n,2},
\qquad
(K_\pm)_n:=(k_\pm)_{n,1}+i(k_\pm)_{n,2}.
\label{eq:complex-rope-coordinates}
\end{equation}
Let $p_q$ and $p_k$ be the query and key positions, set $m:=p_q-p_k$, and
write
\begin{equation}
R(\theta):=
\begin{pmatrix}
\cos\theta&-\sin\theta\\
\sin\theta&\cos\theta
\end{pmatrix}.
\label{eq:rope-block-rotation}
\end{equation}
For a position-independent attention-score scale $c_{\mathrm{att}}>0$, define
\begin{equation}
S_\pm(m):=
c_{\mathrm{att}}\sum_{n\in\cF}
\left\langle R(p_q\omega_n)q_n,R(p_k\omega_n)(k_\pm)_n\right\rangle.
\label{eq:real-rope-score-definition}
\end{equation}

The rotation matrices satisfy
\begin{equation}
R(\alpha)^\top R(\beta)=R(\beta-\alpha).
\label{eq:rotation-composition}
\end{equation}
Consequently, the contribution of the $n$th rotary plane is
\begin{align}
\left\langle
R(p_q\omega_n)q_n,R(p_k\omega_n)(k_\pm)_n
\right\rangle
&=q_n^\top R\bigl((p_k-p_q)\omega_n\bigr)(k_\pm)_n\notag\\
&=q_n^\top R(-m\omega_n)(k_\pm)_n.
\label{eq:real-rope-plane-score}
\end{align}
The same plane contribution admits the following complex representation:
\begin{equation}
Q_n\overline{(K_\pm)_n}
=q_{n,1}(k_\pm)_{n,1}+q_{n,2}(k_\pm)_{n,2}
+i\bigl(q_{n,2}(k_\pm)_{n,1}-q_{n,1}(k_\pm)_{n,2}\bigr).
\label{eq:complex-coefficient-expansion}
\end{equation}
Taking the real part after multiplication by $e^{im\omega_n}$ gives
\begin{align}
\Rea\!\left(Q_n\overline{(K_\pm)_n}e^{im\omega_n}\right)
&=\bigl(q_{n,1}(k_\pm)_{n,1}+q_{n,2}(k_\pm)_{n,2}\bigr)
\cos(m\omega_n)\notag\\
&\quad+\bigl(q_{n,1}(k_\pm)_{n,2}-q_{n,2}(k_\pm)_{n,1}\bigr)
\sin(m\omega_n)\notag\\
&=q_n^\top R(-m\omega_n)(k_\pm)_n.
\label{eq:complex-real-plane-identity}
\end{align}
Combining \Eqref{eq:real-rope-score-definition} and
\Eqref{eq:complex-real-plane-identity} yields
\begin{align}
S_\pm(m)
&=\Rea\sum_{n\in\cF}
\bigl(c_{\mathrm{att}}Q_n\overline{(K_\pm)_n}\bigr)e^{im\omega_n}
=X_\cF^{a_\pm}(m),\label{eq:rope-score-representation}\\
(a_+)_n&:=c_{\mathrm{att}}Q_n\overline{(K_+)_n},
\qquad
(a_-)_n:=c_{\mathrm{att}}Q_n\overline{(K_-)_n}.
\label{eq:rope-score-coefficients}
\end{align}
Thus every fixed score has the required form.  Real linearity in the
coefficient vector shows that every finite real linear combination has the
same form.  In particular, setting
\begin{equation}
a_n:=(a_+)_n,\qquad
b_n:=(a_-)_n,\qquad
d_n:=a_n-b_n
\label{eq:coefficient-triple}
\end{equation}
gives
\begin{equation}
D(m):=S_+(m)-S_-(m)=X_\cF^d(m).
\label{eq:rope-margin-representation}
\end{equation}

Conversely, fix any $z\in\mathbb C^h$.  Choosing $Q_n=1$ and
$(K_+)_n=\overline{z_n/c_{\mathrm{att}}}$ for every $n$ gives
$c_{\mathrm{att}}Q_n\overline{(K_+)_n}=z_n$.  Each chosen complex coordinate
corresponds to a real two-dimensional rotary block, so this constructs real
$q,k_+\in\mathbb R^{2h}$ that realize $X_\cF^z$.
\end{proof}

\subsection{Finite-Window Fourier-Frame Bound}
\label{app:frame-bound}

For $\alpha\in\mathbb R$, define the normalized finite-sum kernel
\begin{equation}
K_M(\alpha):=\frac1M\sum_{m=0}^{M-1}e^{im\alpha}.
\label{eq:KM-def}
\end{equation}
For a finite frequency set $J\subseteq\cF$, define
\begin{equation}
\Lambda_J^0:=\{0\}\cup\{+\omega_n:n\in J\}\cup\{-\omega_n:n\in J\},
\qquad
\delta_J:=\min_{\lambda\neq\lambda'\in\Lambda_J^0}
\dist(\lambda-\lambda',2\pi\mathbb Z),
\label{eq:signed-wrapped-spacing}
\end{equation}
and the finite-window Gram matrix and signed-frame defect
\begin{equation}
\mathcal G_J(M):=
\bigl[K_M(\lambda'-\lambda)\bigr]_{\lambda,\lambda'\in\Lambda_J^0},
\qquad
\mathfrak d_J(M):=\norm{\mathcal G_J(M)-I}_{\op}.
\label{eq:defect}
\end{equation}

\begin{restatable}[Finite-window Fourier-frame bound]{lemma}{FiniteWindowFrame}
\label{lem:frame-bound}
If $\delta_J>0$, then
\begin{equation}
\mathfrak d_J(M)\leq\frac{C_{\mathrm F}}{M\delta_J}.
\label{eq:frame-bound}
\end{equation}
In particular, $M\delta_J\geq2C_{\mathrm F}/\eps$ implies
$\mathfrak d_J(M)\leq\eps/2$.
\end{restatable}
\begin{proof}
We decompose the off-diagonal kernel into two cosecant forms, apply the
separated-point Hilbert inequality to each form, and then normalize by the
window length.

The finite geometric-sum kernel satisfies
\begin{equation}
\sum_{m=0}^{M-1}e^{im(\lambda-\lambda')}
=
\frac{e^{i(M-1/2)(\lambda-\lambda')}
      -e^{-i(\lambda-\lambda')/2}}
     {2i\sin((\lambda-\lambda')/2)}.
\label{eq:cosecant-decomposition}
\end{equation}
In the off-diagonal quadratic form, the two numerator terms decompose into
cosecant forms whose coefficient vectors differ only by diagonal unitary
modulation.  Apply the circular Hilbert inequality of
\citet[Theorem~1, Eq.~(1.2)]{Montgomery_1974} with
$x_r=\lambda_r/(2\pi)$ and modulo-one spacing $\delta_J/(2\pi)$.  It bounds
each term, including its factor $1/2$, by $\pi/\delta_J$ in operator norm.
Dividing by $M$ and applying the triangle inequality gives
\begin{equation}
\mathfrak d_J(M)
\leq\frac{2\pi}{M\delta_J}.
\end{equation}
Thus the lemma holds with the universal choice $C_{\mathrm F}:=2\pi$.  The
stated sufficient condition for the defect bound follows by substitution.
\end{proof}
\subsection{Proof of Proposition~\ref{prop:high-calibration} and Lemma~\ref{lem:high-calibration-formal}}
\label{app:high-main}

\HighBandCalibration*

\HighBandCalibrationFormal*

\begin{proof}
By the exact formulation in Appendix~\ref{app:high-exact-calibration},
it is enough to establish the two moment
inequalities in \Eqref{eq:high-calibration} and then transfer a
$\tau_{\mathrm{H,G}}$ Gaussian-shape bound to
\Eqref{eq:high-gaussian-calibration}.  The proof proceeds in four stages.  We
first certify separation of the signed
frequencies and convert that separation into blockwise Gram-matrix control.  We
then extract the mean and variance bounds and finally transfer the conditional
Gaussian approximation from the exact moments to their calibrated targets.

Use the finite-sum kernel, signed-frequency spacing, and Gram-matrix
defect defined in Appendix~\ref{app:frame-bound}.

We first verify the required signed-frequency separation.  Because
$M\geq\Gamma_\eps(\rho)$, the certified-high set is nonempty.  Write
$H=\{0,\ldots,v\}$.  To lower-bound the signed wrapped spacing, consider the
four possible pair types.  Distances from the signed frequencies to zero are
at least $\rho^v$; same-sign distances are at least
$(1-\rho)\rho^v$; opposite-sign distances are at least $2\rho^v$; and the
wrap-around distance is larger than $2\pi-2$.  These cases cover every pair of
distinct signed frequencies.  Hence
\begin{equation}
\delta_H\geq(1-\rho)\rho^v.
\label{eq:geometric-spacing}
\end{equation}
Since $v\in H$,
\begin{equation}
M\delta_H
\geq M(1-\rho)\rho^v
\geq(1-\rho)\Gamma_\eps(\rho)
=\frac{2C_{\mathrm F}}{\eps}.
\label{eq:high-spacing-condition}
\end{equation}
Lemma~\ref{lem:frame-bound} therefore gives
$\mathfrak d_H(M)\leq\eps/2$.

To read off the individual blocks, define the vector $g_H$ and the matrices
$G_H^-$ and $G_H^+$ by
\begin{equation}
(g_H)_n:=K_M(\omega_n),
\label{eq:g-vector}
\end{equation}
\begin{equation}
(G_H^-)_{n,n'}:=K_M(\omega_{n'}-\omega_n),
\qquad
(G_H^+)_{n,n'}:=K_M(\omega_n+\omega_{n'}).
\label{eq:G-matrices}
\end{equation}
Order $\Lambda_H^0$ as $(0,+H,-H)$.  Since
$K_M(-\alpha)=\overline{K_M(\alpha)}$, the complete centered Gram matrix is
\begin{equation}
\mathcal G_H(M)-I
=
\begin{bmatrix}
0 & g_H^\top & g_H^* \\
\overline{g_H} & G_H^- - I & \overline{G_H^+} \\
g_H & G_H^+ & \overline{G_H^-} - I
\end{bmatrix}.
\label{eq:block-matrix}
\end{equation}
Every displayed block is a compression of the full centered matrix.  Therefore
\begin{equation}
\norm{g_H}_2\leq\mathfrak d_H(M),
\qquad
\norm{G_H^- - I}_{\op}\leq\mathfrak d_H(M),
\qquad
\norm{G_H^+}_{\op}\leq\mathfrak d_H(M).
\label{eq:block-bounds}
\end{equation}

We now convert these block bounds into moment estimates.  Let
$Z_H^z(m):=\sum_{n\in H}z_ne^{im\omega_n}$, so
$X_H^z=\Rea Z_H^z$.  The first row of \Eqref{eq:block-matrix} gives
\begin{equation}
\avg{X_H^z}=\Rea(g_H^\top z_H),
\qquad
\abs{\avg{X_H^z}}
\leq\mathfrak d_H(M)U_H(z).
\label{eq:mean-defect}
\end{equation}
Using $(\Rea Z)^2=\frac12\abs{Z}^2+\frac12\Rea(Z^2)$,
\begin{equation}
\avg{(X_H^z)^2}
=\frac12z_H^*G_H^-z_H
+\frac12\Rea\bigl(z_H^\top G_H^+z_H\bigr).
\label{eq:second-moment-blocks}
\end{equation}
The block bounds imply
\begin{equation}
\abs{\avg{(X_H^z)^2}-\frac12U_H(z)^2}
\leq\mathfrak d_H(M)U_H(z)^2.
\end{equation}
After subtracting the squared mean,
\begin{equation}
\abs{\Var_M(X_H^z)-\frac12U_H(z)^2}
\leq
\bigl(\mathfrak d_H(M)+\mathfrak d_H(M)^2\bigr)U_H(z)^2.
\label{eq:variance-defect}
\end{equation}
Finally,
\begin{equation}
\mathfrak d_H(M)+\mathfrak d_H(M)^2
\leq\frac\eps2+\frac{\eps^2}{4}\leq\eps.
\end{equation}
Substitution into \Eqref{eq:mean-defect} and \Eqref{eq:variance-defect}
proves \Eqref{eq:high-calibration}.

It remains to prove the conditional Gaussian conclusion.  Assume
$U_H(z)>0$ and set $s_0:=U_H(z)/\sqrt{2}$.  The variance bound in
\Eqref{eq:high-calibration} gives
\begin{equation}
\sigma_H^2\geq\left(\frac12-\eps\right)U_H(z)^2>0,
\qquad
\sqrt{1-2\eps}\leq\frac{\sigma_H}{s_0}
\leq\sqrt{1+2\eps}.
\label{eq:high-gaussian-scale-ratio}
\end{equation}
Thus the standardized variable in
\Eqref{eq:high-gaussian-shape-assumption} is well defined, and that condition
implies
\begin{equation}
\sup_{x\in\mathbb R}
\abs{
\mathbb P[Y_H^z\leq x]
-\Phi_{\mathrm{std}}\!\left(\frac{x-\mu_H}{\sigma_H}\right)
}
\leq\tau_{\mathrm{H,G}}.
\label{eq:high-gaussian-exact-moment-law}
\end{equation}

We next replace the exact moments by their calibrated targets.  Differentiating
$\Phi_{\mathrm{std}}((x-\mu)/s)$ with respect to $\log s$ shows that its
absolute derivative is at most $1/\sqrt{2\pi e}$, uniformly in $x$ and
$\mu$.  Hence \Eqref{eq:high-gaussian-scale-ratio} yields
\begin{equation}
\sup_{x\in\mathbb R}
\abs{
\Phi_{\mathrm{std}}\!\left(\frac{x-\mu_H}{\sigma_H}\right)
-\Phi_{\mathrm{std}}\!\left(\frac{x-\mu_H}{s_0}\right)
}
\leq
\frac{\log((1-2\eps)^{-1})}{2\sqrt{2\pi e}}.
\label{eq:high-gaussian-scale-replacement}
\end{equation}
Since the standard Gaussian CDF is $1/\sqrt{2\pi}$-Lipschitz, the mean bound
in \Eqref{eq:high-calibration} also gives
\begin{equation}
\sup_{x\in\mathbb R}
\abs{
\Phi_{\mathrm{std}}\!\left(\frac{x-\mu_H}{s_0}\right)
-\Phi_{\mathrm{std}}\!\left(\frac{x}{s_0}\right)
}
\leq
\frac{\abs{\mu_H}}{s_0\sqrt{2\pi}}
\leq\frac{\eps}{\sqrt\pi}.
\label{eq:high-gaussian-mean-replacement}
\end{equation}
Combining \Eqref{eq:high-gaussian-exact-moment-law},
\Eqref{eq:high-gaussian-scale-replacement}, and
\Eqref{eq:high-gaussian-mean-replacement} by the triangle inequality proves
\Eqref{eq:high-gaussian-calibration}.
This completes the conditional Gaussian step.
\end{proof}
\subsection{Strengthened Bandwise Calibration}
\label{app:bandwise-strengthening}

The main text uses only the stated calibration.  For completeness, the same
frame argument also controls derivative variance and the score--derivative
covariance on any band satisfying the stated separation condition.

\begin{proposition}[Strengthened calibration for a separated band]
\label{prop:bandwise-calibration-app}
Let $J\subseteq\cF$ satisfy
$M\delta_J\geq2C_{\mathrm F}/\eps$.  Define
\begin{equation}
E_0(J;z):=\sum_{n\in J}\abs{z_n}^2,
\qquad
E_1(J;z):=\sum_{n\in J}\omega_n^2\abs{z_n}^2,
\end{equation}
and
\begin{equation}
(X_J^z)'(m):=\Rea\sum_{n\in J}i\omega_nz_ne^{im\omega_n}.
\end{equation}
Then
\begin{align}
\abs{\avg{X_J^z}}&\leq\eps\sqrt{E_0(J;z)},\\
\abs{\Var_M(X_J^z)-\tfrac12E_0(J;z)}&\leq\eps E_0(J;z),\\
\abs{\Var_M((X_J^z)')-\tfrac12E_1(J;z)}&\leq\eps E_1(J;z),\\
\abs{\operatorname{Cov}_M(X_J^z,(X_J^z)')}
&\leq\eps\sqrt{E_0(J;z)E_1(J;z)}.
\label{eq:bandwise-strengthened-calibration}
\end{align}
\end{proposition}

\begin{proof}
Lemma~\ref{lem:frame-bound} gives $\mathfrak d_J(M)\leq\eps/2$.  Repeating
the block calculation in Appendix~\ref{app:high-main} with $H$ replaced by
$J$ proves the first two inequalities.  For the derivative, set
$w_n=i\omega_nz_n$.  Then $X_J^w=(X_J^z)'$ and
$\norm{w_J}_2^2=E_1(J;z)$, so the same variance calculation applies.

For the covariance, the bilinear version of
\Eqref{eq:second-moment-blocks} is
\begin{equation}
\avg{X_J^zX_J^w}
=\frac12\Rea\bigl(z_J^*G_J^-w_J+z_J^\top G_J^+w_J\bigr).
\end{equation}
The identity contribution vanishes because
$\Rea(z_J^*w_J)=0$.  \Eqref{eq:block-bounds} and
\Eqref{eq:mean-defect}, with $H$ replaced by $J$, therefore give
\begin{equation}
\abs{\operatorname{Cov}_M(X_J^z,X_J^w)}
\leq
\bigl(\mathfrak d_J(M)+\mathfrak d_J(M)^2\bigr)
\norm{z_J}_2\norm{w_J}_2.
\end{equation}
Finally,
$\mathfrak d_J+\mathfrak d_J^2\leq\eps/2+\eps^2/4\leq\eps$.
\end{proof}

\subsection{Direct Finite-Sum Intuition}
\label{app:direct}

The next calculation is independent of the Fourier-frame inequality above.  It
is not needed for the main-text calibration result, but it makes the
high-frequency centering mechanism explicit.

\begin{proposition}[Direct centering on the geometric grid]
\label{prop:direct-centering}
For a contiguous band $J=[u,v]$ of size $s=v-u+1$,
\begin{equation}
\abs{\avg{X_J^z}}
\leq
\frac{\pi}{M\rho^v}
\sqrt{\frac{1-\rho^{2s}}{1-\rho^2}}
U_J(z).
\label{eq:direct-centering}
\end{equation}
\end{proposition}

\begin{proof}
For $0<\alpha\leq1$,
\begin{equation}
K_M(\alpha)
=e^{i(M-1)\alpha/2}
\frac{\sin(M\alpha/2)}{M\sin(\alpha/2)},
\qquad
\abs{K_M(\alpha)}\leq\frac{\pi}{M\alpha}.
\end{equation}
Therefore Cauchy--Schwarz gives
\begin{align}
\abs{\avg{X_J^z}}
&\leq\frac{\pi}{M}
\left(\sum_{n=u}^v\rho^{-2n}\right)^{1/2}U_J(z)\\
&=\frac{\pi}{M\rho^v}
\sqrt{\frac{1-\rho^{2s}}{1-\rho^2}}U_J(z).
\end{align}
\end{proof}

\subsection{Positional Response Bounds and Refinements}
\label{app:positional}

The proofs and coefficient-specific refinements in this section use the
following appendix-only shorthand:
\begin{align}
A_J(z)&:=\sum_{n\in J}\abs{z_n},
\label{eq:appendix-l1-coefficient-norm}\\
\Delta_1X_J^z(t)&:=X_J^z(t+1)-X_J^z(t),
\label{eq:local-difference}\\
g_n&:=\abs{e^{i\omega_n}-1}.
\label{eq:adjacent-frequency-gain}
\end{align}

\subsubsection{Auxiliary Non-High-Only Response Bound}

\begin{lemma}[Non-high-only positional response has linear share cost]
\label{lem:nonhigh-response}
For every certified $M$ and every nonzero $z$,
\begin{equation}
\sup_t\abs{\Delta_1X_N^z(t)}
\leq
\frac{\Gamma_\eps(\rho)}{M}A_N(z;M).
\label{eq:nonhigh-response-bound}
\end{equation}
Consequently, for every $\zeta>0$, if
$\sup_t\abs{\Delta_1X_N^z(t)}/U(z)\geq\zeta$, then
\begin{equation}
r_N(z;M)
\geq
\frac{\zeta M}{\Gamma_\eps(\rho)\sqrt h}.
\label{eq:linear-nonhigh-share-cost}
\end{equation}
\end{lemma}

\begin{proof}
For every $n\in N$, the partition gives
$\omega_n<\Gamma_\eps(\rho)/M$.  Hence
\begin{align}
\abs{\Delta_1X_N^z(t)}
&\leq
\sum_{n\in N}g_n\abs{z_n}\\
&\leq
\sum_{n\in N}\omega_n\abs{z_n}\\
&\leq
\frac{\Gamma_\eps(\rho)}{M}A_N(z;M),
\end{align}
which proves \Eqref{eq:nonhigh-response-bound}.  Cauchy--Schwarz and
$\abs{N}\leq h$ give
\begin{equation}
A_N(z;M)
\leq
\sqrt h\,U_N(z)
=
\sqrt h\,U(z)r_N(z;M).
\end{equation}
Therefore a normalized non-high response of size at least $\zeta$ implies
\begin{equation}
\zeta
\leq
\frac{\Gamma_\eps(\rho)\sqrt h}{M}r_N(z;M),
\end{equation}
which proves \Eqref{eq:linear-nonhigh-share-cost}.
\end{proof}

The gain multiplier in this auxiliary envelope improves on the universal
pointwise bound throughout the certified regime.  Indeed,
$M\geq\Gamma_\eps(\rho)$ implies
\begin{equation}
0<\frac{\Gamma_\eps(\rho)}{M}\leq1<2,
\label{eq:nonhigh-cutoff-nonvacuous}
\end{equation}
where $2$ is the universal bound on $\abs{e^{i\omega}-1}$.  For the Qwen
grid used in our experiments, $B=10^6$ and $2h=128$, so
$\rho=(10^6)^{-1/64}\approx0.805842$.  With the fixed experimental choice
$\eps=10^{-2}$, this gives
\begin{equation}
\Gamma_{10^{-2}}(\rho)\approx6472.25,
\qquad
\frac{\Gamma_{10^{-2}}(\rho)}{32768}\approx0.1975.
\label{eq:qwen-nonhigh-cutoff-scale}
\end{equation}
Thus the experimental window is certified, and its non-high gain coefficient
is well below the trivial response coefficient.

\subsubsection{Proof of Lemma~\ref{lem:stable-adjacent-response-envelope}}
\label{app:proof-stable-adjacent-response-envelope}

\StableAdjacentResponseEnvelope*

\begin{proof}
For every relative distance $t$, split the score difference over the non-high and
certified-high bands.  The triangle inequality gives
\begin{equation}
\abs{\Delta_1X_\cF^a(t)}
\leq
\abs{\Delta_1X_N^a(t)}
+\abs{\Delta_1X_H^a(t)}.
\label{eq:adjacent-band-split-proof}
\end{equation}
Lemma~\ref{lem:nonhigh-response} and Cauchy--Schwarz imply
\begin{equation}
\abs{\Delta_1X_N^a(t)}
\leq
\frac{\Gamma_\eps(\rho)}{M}A_N(a;M)
\leq
\frac{\Gamma_\eps(\rho)\sqrt h}{M}U_N(a).
\label{eq:stable-nonhigh-term-proof}
\end{equation}
For the certified-high band, another application of Cauchy--Schwarz gives
\begin{equation}
\abs{\Delta_1X_H^a(t)}
\leq
\sum_{n\in H}g_n\abs{a_n}
\leq
\left(\sum_{n\in H}g_n^2\right)^{1/2}U_H(a)
\leq R_\cF U_H(a).
\label{eq:stable-high-term-proof}
\end{equation}
Substitute these two bounds into \Eqref{eq:adjacent-band-split-proof}, divide
by $U(a)>0$, and take the maximum over the integer window.  This proves the
first line of \Eqref{eq:stable-adjacent-response-envelope}; the second follows from
$r_N(a;M)\leq1$.
\end{proof}

\subsubsection{Proof of Corollary~\ref{cor:normalized-positional-share-floor}}
\label{app:proof-normalized-positional-share-floor}

\NormalizedPositionShareFloor*

\begin{proof}
Combining the response requirement with
Lemma~\ref{lem:stable-adjacent-response-envelope} yields
\begin{equation}
R_\cF r_H(a;M)
\geq
\zeta-\frac{\Gamma_\eps(\rho)\sqrt h}{M}.
\end{equation}
Because $r_H(a;M)\geq0$ and $R_\cF>0$, take the positive part and divide by
$R_\cF$ to prove \Eqref{eq:normalized-positional-share-floor}.  A lower bound
above one contradicts $r_H\leq1$.
\end{proof}

\subsubsection{Exact-Gain Response Envelope}
\label{app:exact-gain-response-envelope}

For this appendix only, abbreviate the \S\ref{sec:failure-modes}
window diagnostic as
\begin{equation}
\widehat{\overline G}_M(a)
:=
\max_{0\leq m\leq M-2}
\frac{\abs{X_\cF^a(m+1)-X_\cF^a(m)}}{U(a)}.
\label{eq:normalized-position-maxima}
\end{equation}

We also define the continuous normalized envelope
\begin{equation}
\overline{\mathcal G}(a)
:=
\sup_{t\in\mathbb R}
\frac{\abs{\Delta_1X_\cF^a(t)}}{U(a)}.
\label{eq:appendix-continuous-position-envelope}
\end{equation}
To state two sharper relaxations, one coefficient-specific and one depending
only on band norms, set
\begin{align}
R_H(M)&:=
\left(\sum_{n\in H}g_n^2\right)^{1/2},
&
R_N(M)&:=
\left(\sum_{n\in N}g_n^2\right)^{1/2},\notag\\
R_\cF^2&=R_H(M)^2+R_N(M)^2.
\label{eq:appendix-response-factors}
\end{align}
The full-grid factor $R_\cF$ is defined in
Lemma~\ref{lem:stable-adjacent-response-envelope} and is independent of the
partition.  For
$a\neq0$ with $U_H(a)>0$, also define
\begin{align}
w_N(a;M)
&:=
\frac{\sum_{n\in N}g_n\abs{a_n}}{U(a)},
\label{eq:normalized-nonhigh-quantities}\\
\kappa_H(a;M)
&:=
\frac{\sum_{n\in H}g_n\abs{a_n}}{U_H(a)},
\label{eq:high-shape-response-factor}\\
T(a)
&:=
w_N(a;M)+\kappa_H(a;M)r_H(a;M)
=
\frac{\sum_{n\in\cF}g_n\abs{a_n}}{U(a)}.
\label{eq:pointwise-upper-certificate}
\end{align}

\begin{proposition}[Band-norm relaxations of the exact-gain envelope]
\label{prop:exact-gain-normalized-envelope}
For every certified integer $M\geq2$ and every positional vector $a\neq0$ with
$U_H(a)>0$, let $r=r_H(a;M)$ and define
\begin{equation}
F_M(r):=
R_N(M)\sqrt{1-r^2}+R_H(M)r.
\label{eq:r-only-position-envelope}
\end{equation}
Then
\begin{equation}
\widehat{\overline G}_M(a)
\leq
\overline{\mathcal G}(a)
\leq
T(a)
\mathrel{=}
w_N(a;M)+\kappa_H(a;M)r
\leq
w_N(a;M)+R_H(M)r
\leq
F_M(r).
\label{eq:exact-gain-normalized-envelope}
\end{equation}
\end{proposition}

\begin{proof}
The integer-window maximum ranges over a subset of the continuous
relative distances, so
$\widehat{\overline G}_M(a)\leq\overline{\mathcal G}(a)$.  For every $t$,
the triangle inequality gives
\begin{equation}
\frac{\abs{\Delta_1X_\cF^a(t)}}{U(a)}
\leq
\frac{\sum_{n\in\cF}g_n\abs{a_n}}{U(a)}
=T(a),
\end{equation}
and hence $\overline{\mathcal G}(a)\leq T(a)$.  The identity
$T(a)=w_N(a;M)+\kappa_H(a;M)r$ is given by
\Eqref{eq:pointwise-upper-certificate}.  Cauchy--Schwarz on the high band
gives
\begin{equation}
\kappa_H(a;M)
\leq
\left(\sum_{n\in H}g_n^2\right)^{1/2}
=R_H(M),
\end{equation}
and hence
\begin{equation}
T(a)
\leq
w_N(a;M)+R_H(M)\frac{U_H(a)}{U(a)}
=w_N(a;M)+R_H(M)r.
\end{equation}
Applying Cauchy--Schwarz to the non-high part gives
\begin{equation}
w_N(a;M)
\leq
R_N(M)\frac{U_N(a)}{U(a)}
=R_N(M)\sqrt{1-r^2}.
\end{equation}
Substitution proves \Eqref{eq:exact-gain-normalized-envelope}.
\end{proof}

The middle quantity $T(a)$ is a coefficient-specific, phase-agnostic upper
certificate obtained from the triangle inequality.  The final function $F_M$
is an upper envelope that retains only the two band norms.  Neither is a
prediction of $\mathbb E[\widehat{\overline G}_M\mid r_H]$.  In particular,
\begin{equation}
F_M(0)=R_N(M),
\qquad
r_*(M):=\frac{R_H(M)}{R_\cF},
\label{eq:position-envelope-endpoint-turn}
\end{equation}
so the envelope need not vanish at zero certified-high norm share and is
increasing only for $0\leq r<r_*(M)$.

\begin{corollary}[Sharper coefficient-shape floor]
\label{cor:coefficient-shape-floor-app}
For every certified integer $M\geq2$, every $a\neq0$ with $U_H(a)>0$, and every
$\zeta>0$, define
\begin{equation}
f_{\mathrm{pos}}(a;M,\zeta)
:=
\frac{[\zeta-w_N(a;M)]_+}{\kappa_H(a;M)}.
\label{eq:coefficient-shape-normalized-floor}
\end{equation}
If $\widehat{\overline G}_M(a)\geq\zeta$, then
\begin{equation}
r_H(a;M)\geq f_{\mathrm{pos}}(a;M,\zeta).
\label{eq:coefficient-shape-floor-app}
\end{equation}
If the right-hand side exceeds one, the requested response is impossible.
\end{corollary}

\begin{proof}
Let $r=r_H(a;M)$.  The response condition and
\Eqref{eq:exact-gain-normalized-envelope} give
\begin{equation}
\zeta
\leq
w_N(a;M)+\kappa_H(a;M)r.
\end{equation}
Since $U_H(a)>0$ and every $g_n$ is positive,
$\kappa_H(a;M)>0$.  Rearranging proves
\Eqref{eq:coefficient-shape-floor-app}.
\end{proof}

\begin{corollary}[Monotone coefficient-specific relaxation]
\label{cor:coefficient-specific-monotone-floor-app}
Under the assumptions of
Corollary~\ref{cor:coefficient-shape-floor-app}, define
\begin{equation}
f_{\mathrm{pos}}^{\mathrm c}(a;M,\zeta)
:=
\frac{[\zeta-w_N(a;M)]_+}{R_\cF}.
\label{eq:coefficient-specific-normalized-floor}
\end{equation}
If $\widehat{\overline G}_M(a)\geq\zeta$, then
\begin{equation}
r_H(a;M)
\geq f_{\mathrm{pos}}(a;M,\zeta)
\geq f_{\mathrm{pos}}^{\mathrm c}(a;M,\zeta)
\geq \underline r_{\mathrm{pos}}(M;\zeta).
\label{eq:positional-floor-hierarchy-app}
\end{equation}
For fixed $a$ and $\zeta$, this relaxation is non-decreasing in $M$.
\end{corollary}

\begin{proof}
Cauchy--Schwarz gives
$\kappa_H(a;M)\leq R_H(M)\leq R_\cF$, and hence
$f_{\mathrm{pos}}(a;M,\zeta)\geq
f_{\mathrm{pos}}^{\mathrm c}(a;M,\zeta)$.  The leading lower bound in the
stated hierarchy follows from
Corollary~\ref{cor:coefficient-shape-floor-app}.  Moreover,
$g_n\leq\omega_n<\Gamma_\eps(\rho)/M$ for every $n\in N$, so
Cauchy--Schwarz gives
\begin{equation}
w_N(a;M)
\leq
\frac{\Gamma_\eps(\rho)\sqrt h}{M},
\end{equation}
so $f_{\mathrm{pos}}^{\mathrm c}(a;M,\zeta)\geq
\underline r_{\mathrm{pos}}(M;\zeta)$.
As $M$ increases, $N(M)$ shrinks.  Therefore $w_N(a;M)$ is non-increasing,
while $R_\cF$ is independent of the partition.  Hence
$f_{\mathrm{pos}}^{\mathrm c}(a;M,\zeta)$ is non-decreasing.
\end{proof}

\subsubsection{General Separations}
\label{app:positional-general-tau}

For a separation $\ell>0$, define
\begin{equation}
\Delta_\ell X_J^z(t):=X_J^z(t+\ell)-X_J^z(t),
\end{equation}
\begin{equation}
\beta_N(M,\ell):=\max_{n\in N}\abs{e^{i\ell\omega_n}-1},
\qquad
R_H(M,\ell):=
\left(\sum_{n\in H}\abs{e^{i\ell\omega_n}-1}^2\right)^{1/2}.
\end{equation}

\begin{proposition}[Response bound at a general separation]
\label{prop:general-separation}
For every $t$,
\begin{equation}
\abs{\Delta_\ell X_\cF^z(t)}
\leq
\beta_N(M,\ell)A_N(z;M)+R_H(M,\ell)U_H(z)
\leq
\frac{\ell\Gamma_\eps(\rho)}{M}A_N(z;M)
+R_H(M,\ell)U_H(z).
\label{eq:general-separation-bound}
\end{equation}
\end{proposition}

\begin{proof}
The first inequality follows from the triangle inequality on $N$ and
Cauchy--Schwarz on $H$.  For the second inequality, use
$\abs{e^{i\ell\omega_n}-1}\leq\ell\omega_n
<\ell\Gamma_\eps(\rho)/M$ on $N$.
\end{proof}

\subsubsection{Coarse Constant Floor}

The following coarser form is sometimes convenient when only uniform norm
budgets are available.

\begin{corollary}[Constant positional-share floor]
\label{cor:constant-floor-app}
Assume $A_N(a;M)\leq A_*$ and $U(a)\leq U_*$.  If
\begin{equation}
M\geq\max\left\{\Gamma_\eps(\rho),
\frac{2\Gamma_\eps(\rho)A_*}{\eta}\right\}
\end{equation}
and $\abs{\Delta_1X_\cF^a(t)}\geq\eta$, then
\begin{equation}
r_H(a;M)\geq\frac{\eta}{4\sqrt h\,U_*}.
\end{equation}
\end{corollary}

\begin{proof}
Taking $\ell=1$ in \Eqref{eq:general-separation-bound}, the assumed lower bound
on the window length makes the non-high term at most $\eta/2$.  Thus
$R_H(M)U_H(a)\geq\eta/2$.  Divide by $U(a)\leq U_*$ and use
$R_H(M)\leq2\sqrt h$.
\end{proof}
\subsection{Proof of Theorem~\ref{thm:long-range-order-ambiguity}}
\label{app:proof-long-range-order-ambiguity}

\LongRangeOrderAmbiguity*

\begin{proof}
We compare both translated-window laws with a common Gaussian reference law.
We first replace the near-window CDF in the ordering probability, then replace
the far-window CDF, and finally pass to the asymptotic regime.

Let $Y_{\mathrm{near}}$ and $Y_{\mathrm{far}}$ be the independent interval
scores in \Eqref{eq:long-range-order-chance-limit}, with CDFs
$F_{\mathrm{near}}$ and $F_{\mathrm{far}}$.  Since both translated windows
have length $M$, are fully high, and satisfy the stated Gaussian-shape
calibration, Proposition~\ref{prop:high-calibration} compares them with the
same continuous Gaussian reference CDF
\begin{equation}
G_a(x)
:=
\Phi_{\mathrm{std}}\!\left(\frac{\sqrt2\,x}{U(a)}\right).
\label{eq:common-interval-gaussian-proof}
\end{equation}
Specifically,
\begin{equation}
\begin{aligned}
\sup_x\abs{F_{\mathrm{near}}(x)-G_a(x)}
&\leq
\eta_{\mathrm{near}}
:=\tau_{\mathrm{near}}(a;M)+c_{\mathrm G}(\eps),\\
\sup_x\abs{F_{\mathrm{far}}(x)-G_a(x)}
&\leq
\eta_{\mathrm{far}}
:=\tau_{\mathrm{far}}(a;M)+c_{\mathrm G}(\eps).
\end{aligned}
\label{eq:interval-cdf-calibration-proof}
\end{equation}

Using the left limit of the near-interval CDF, the strict ordering event is
\begin{equation}
p_{\mathrm{ord}}(a;M)
=
\int_{\mathbb R}F_{\mathrm{near}}(y^-)\,dF_{\mathrm{far}}(y).
\label{eq:ordering-left-cdf-proof}
\end{equation}
Continuity of $G_a$ and
\Eqref{eq:interval-cdf-calibration-proof} give
\begin{equation}
\abs{
p_{\mathrm{ord}}(a;M)
-\int_{\mathbb R}G_a(y)\,dF_{\mathrm{far}}(y)
}
\leq\eta_{\mathrm{near}}.
\label{eq:near-interval-replacement-proof}
\end{equation}
If $Z$ has CDF $G_a$ and is independent of $Y_{\mathrm{far}}$, then
\begin{align}
\int_{\mathbb R}G_a(y)\,dF_{\mathrm{far}}(y)
&=\mathbb P[Z<Y_{\mathrm{far}}]\notag\\
&=\int_{\mathbb R}\bigl(1-F_{\mathrm{far}}(z)\bigr)\,dG_a(z).
\label{eq:far-interval-replacement-identity-proof}
\end{align}
The second inequality in \Eqref{eq:interval-cdf-calibration-proof} and
$\int(1-G_a)\,dG_a=1/2$ therefore imply
\begin{equation}
\abs{
\int_{\mathbb R}G_a(y)\,dF_{\mathrm{far}}(y)-\frac12
}
\leq\eta_{\mathrm{far}}.
\label{eq:far-interval-replacement-proof}
\end{equation}
Combining the last two bounds gives
\begin{equation}
\abs{p_{\mathrm{ord}}(a;M)-\frac12}
\leq\tau_{\mathrm{ord}}(a;M,\eps).
\label{eq:long-range-order-finite-bound-proof}
\end{equation}
The left-CDF identity handles the strict event without a no-ties assumption.

Finally, $\eps_M\to0$.  Taking $\eps=\eps_M$ in
\Eqref{eq:long-range-order-finite-bound-proof} and using
$\tau_{\mathrm{ord}}(a;M,\eps_M)\to0$ proves
\Eqref{eq:long-range-order-chance-limit}.
\end{proof}
\subsection{Proof of Lemma~\ref{lem:single-key-exceedance}}
\label{app:proof-single-key-exceedance}

\SingleKeyExceedance*

\begin{proof}
We first compare the exact and calibrated moments, then compare the true tail
with the exact-moment Gaussian approximation using the shape discrepancy, and
finally propagate the moment replacements through the Gaussian CDF.

Let $\mu_H:=\avg{X_H^b}$ and
$\sigma_H^2:=\Var_M(X_H^b)$.  The exact sum $Y_b=Y_N+Y_H$ gives
\begin{equation}
\mu_b=\mu_N+\mu_H,
\qquad
\sigma_b^2=\sigma_N^2+\sigma_H^2+2\kappa_{NH}.
\label{eq:single-key-exact-moment-split-proof}
\end{equation}
Proposition~\ref{prop:high-calibration}, applied with $z=b$, therefore yields
\begin{equation}
\abs{\mu_b-\mu_N}\leq\eps U_H(b),
\qquad
\abs{\sigma_b^2-\widehat\sigma_b^2}
\leq\eps U_H(b)^2.
\label{eq:single-key-moment-replacement-proof}
\end{equation}
The assumed variance floor implies
$\sigma_b\geq s_-(b;M)>0$.

By the definition in \Eqref{eq:single-key-shape-error} and the identity
$1-\Phi_{\mathrm{std}}(x)=\Phi_{\mathrm{std}}(-x)$,
\begin{equation}
\abs{
p_{\mathrm{exc}}(b;S,M)
-\Phi_{\mathrm{std}}\!\left(\frac{\mu_b-S}{\sigma_b}\right)
}
\leq\Delta_{\mathrm{exc}}(b;M).
\label{eq:single-key-exact-gaussian-tail-proof}
\end{equation}
A no-ties assumption is unnecessary because
$p_{\mathrm{exc}}=1-\mathbb P[Y_b\leq S]$ uses exactly the closed-threshold CDF
appearing in \Eqref{eq:single-key-shape-error}.

The moment bounds also imply
\begin{align}
\abs{
\frac{\mu_b-S}{\sigma_b}
-\frac{\mu_N-S}{\widehat\sigma_b}
}
&\leq
\frac{\abs{\mu_b-\mu_N}}{\sigma_b}
+\abs{S-\mu_N}
\abs{\frac1{\sigma_b}-\frac1{\widehat\sigma_b}}\notag\\
&\leq
\frac{\eps U_H(b)}{s_-(b;M)}
+\frac{\abs{S-\mu_N}\eps U_H(b)^2}
{s_-(b;M)\widehat\sigma_b
 \bigl(s_-(b;M)+\widehat\sigma_b\bigr)}.
\label{eq:single-key-snr-replacement-proof}
\end{align}
Here we used
$\abs{\sigma_b-\widehat\sigma_b}
=\abs{\sigma_b^2-\widehat\sigma_b^2}/
(\sigma_b+\widehat\sigma_b)$.
Since $\Phi_{\mathrm{std}}$ is $1/\sqrt{2\pi}$-Lipschitz, combining
\Eqref{eq:single-key-exact-gaussian-tail-proof} and
\Eqref{eq:single-key-snr-replacement-proof} proves
\Eqref{eq:single-key-exceedance-calibration}.
\end{proof}
\subsection{Proof of Corollary~\ref{cor:single-key-high-share-ceiling}}
\label{app:proof-single-key-high-share-ceiling}

\SingleKeyHighShareCeiling*

\begin{proof}
The proof converts the probability constraint into a variance bound, rewrites
that bound in the share coordinate, and then solves the resulting quadratic
inequality.

If $p_{\mathrm{exc}}(b;S,M)\leq\delta$, then
Lemma~\ref{lem:single-key-exceedance} gives
\begin{equation}
\widehat p_{\mathrm{exc}}(b;S,M)
\leq\delta+\tau_{\mathrm{exc}}(b;S,M)<\frac12.
\label{eq:single-key-predictor-upper-bound-proof}
\end{equation}
By monotonicity and symmetry of the standard Gaussian CDF,
\begin{equation}
S-\mu_N\geq \xi_{\mathrm{exc}}\widehat\sigma_b>0,
\qquad
\frac{\widehat\sigma_b^2}{U_N(b)^2}
\leq\frac{\lambda_S^2}{\xi_{\mathrm{exc}}^2}.
\label{eq:single-key-variance-ceiling-proof}
\end{equation}
In particular, $\lambda_S>0$.  Define
\begin{equation}
t:=\frac{U_H(b)}{U_N(b)}
=\frac{r_H(b;M)}{\sqrt{1-r_H(b;M)^2}}>0.
\label{eq:single-key-share-coordinate-proof}
\end{equation}
Substituting the normalized descriptors into
\Eqref{eq:single-key-variance-ceiling-proof} gives the quadratic condition
\begin{equation}
\frac12t^2+2\gamma_{NH}t+\beta_N
\leq\frac{\lambda_S^2}{\xi_{\mathrm{exc}}^2}.
\label{eq:single-key-quadratic-ceiling-proof}
\end{equation}
The actual $t>0$ is feasible, so the discriminant in
\Eqref{eq:single-key-high-share-ceiling} is nonnegative and its upper root
$\overline t_{\mathrm{exc}}$ satisfies
$t\leq\overline t_{\mathrm{exc}}$; in particular,
$\overline t_{\mathrm{exc}}>0$.  The map
$t\mapsto t/\sqrt{1+t^2}$ is strictly increasing on $[0,\infty)$, which proves
\Eqref{eq:single-key-high-share-ceiling-conclusion} and also gives
$u_{\mathrm{exc}}<1$.
\end{proof}
\subsection{Proof of Theorem~\ref{thm:maximum-reliable-context-formal}}
\label{app:proof-maximum-reliable-context}

\MaximumReliableContextFormal*

\begin{proof}
Fix a calibrated window at which the same margin $d$ retains normalized
adjacent response $\zeta$.  Corollary~\ref{cor:normalized-positional-share-floor}
applied to $d$ gives
\begin{equation}
r:=r_H(d;M)\geq\ell(M;\zeta).
\label{eq:reversal-share-floor-proof}
\end{equation}
Write $N=\cN_\eps(M)$ and $H=\cH_\eps(M)$.  The high-band moment bounds in
\Eqref{eq:high-calibration} imply
\begin{equation}
\abs{\avg{X_H^d}}\leq\eps U_H(d),
\qquad
\sqrt{\Var_M(X_H^d)}\geq c_\eps U_H(d).
\label{eq:reversal-high-moments-proof}
\end{equation}
For the non-high band, Cauchy--Schwarz gives
\begin{equation}
\abs{\avg{X_N^d}}
\leq\sum_{n\in N}\abs{d_n}
\leq\sqrt{k_N(M)}\,U_N(d),
\qquad
\sqrt{\Var_M(X_N^d)}
\leq\sqrt{k_N(M)}\,U_N(d).
\label{eq:reversal-nonhigh-moments-proof}
\end{equation}
The second inequality follows because centering cannot increase the root mean
square and $\abs{X_N^d(m)}\leq\sum_{n\in N}\abs{d_n}$ pointwise.

Since $U_H(d)=U(d)r$ and $U_N(d)=U(d)\sqrt{1-r^2}$, the triangle and reverse
triangle inequalities in finite-window $L_2$ yield
\begin{equation}
\abs{\mu_D(M)}\leq U(d)u_M(r),
\qquad
\sigma_D(M)\geq U(d)v_M(r).
\label{eq:reversal-snr-envelope-proof}
\end{equation}
Let
$g_M(s):=c_\eps s-\sqrt{k_N(M)}\sqrt{1-s^2}$; this function is
non-decreasing on $[0,1]$ and $v_M(s)=[g_M(s)]_+$.  If $v_M(r)=0$, then
$r\geq\ell(M;\zeta)$ implies $v_M(\ell(M;\zeta))=0$, so both envelope values
in \Eqref{eq:main-reversal-floor} vanish and the desired bound follows from
$p_{\mathrm{rev}}(d;M)\geq0$.  Otherwise,
\Eqref{eq:reversal-threshold-error} and
\Eqref{eq:reversal-snr-envelope-proof} give
\begin{equation}
\begin{aligned}
p_{\mathrm{rev}}(d;M)
&\geq
\Phi_{\mathrm{std}}\!\left(-\frac{\mu_D(M)}{\sigma_D(M)}\right)
-\overline\tau_{\mathrm G}\\
&\geq
\Phi_{\mathrm{std}}\!\left(-\frac{u_M(r)}{v_M(r)}\right)
-\overline\tau_{\mathrm G}.
\end{aligned}
\label{eq:reversal-probability-envelope-proof}
\end{equation}
For the second line, explicitly,
$-\mu_D/\sigma_D\geq-\abs{\mu_D}/\sigma_D\geq-u_M(r)/v_M(r)$.
Since a probability is nonnegative, taking the positive part of the last
lower bound proves
$p_{\mathrm{rev}}(d;M)\geq
\psi_M(r;\overline\tau_{\mathrm G})$.

It remains to verify monotonicity.  On the region $v_M(r)>0$, set
$x(r):=\sqrt{1-r^2}/r$.  Then
\begin{equation}
\frac{u_M(r)}{v_M(r)}
=
\frac{\eps+\sqrt{k_N(M)}x(r)}
{c_\eps-\sqrt{k_N(M)}x(r)}.
\label{eq:reversal-envelope-ratio-proof}
\end{equation}
The right-hand side is non-decreasing in $x$, whereas $x(r)$ decreases with $r$.
Thus $\psi_M(r;\overline\tau_{\mathrm G})$ is non-decreasing in $r$ on the
positive-$v_M$ region.  On the preceding region $v_M=0$, it is identically
zero.  If $k_N(M)>0$, the ratio diverges at the boundary and its Gaussian
lower tail tends to zero.  If $k_N(M)=0$, the ratio equals $\eps/c_\eps$
for every $r>0$, and the prescribed value $\psi_M(0;\overline\tau_{\mathrm G})=0$
also preserves monotonicity.  Positive-part clipping therefore preserves
global monotonicity.
Combining this fact with \Eqref{eq:reversal-share-floor-proof} proves
\Eqref{eq:main-reversal-floor}.

Finally, $\ell(M;\zeta)$ is non-decreasing in $M$, while $k_N(M)$ is
non-increasing because the certified-high prefix expands with $M$.  The same
ratio in \Eqref{eq:reversal-envelope-ratio-proof} depends on these quantities
through
$\sqrt{k_N(M)}\sqrt{1-\ell(M;\zeta)^2}/\ell(M;\zeta)$, which is
non-increasing whenever $\ell(M;\zeta)>0$.  When $\ell(M;\zeta)=0$ or
$v_M(\ell(M;\zeta))=0$, the floor is zero.  Therefore
$\underline p_{\mathrm{rev}}(M;\zeta,\overline\tau_{\mathrm G})$ is
non-decreasing across $\mathfrak M_{\mathrm{cal}}$.  Its minimum and maximum
over this finite family are attained at $M_-$ and $M_+$, respectively.
Thus \Eqref{eq:reversal-admissible-threshold-range} is equivalent to the
sub-threshold and crossing sets both being nonempty.  At and beyond the
first scale $M_\dagger$ in \Eqref{eq:main-reversal-crossing-scale}, the
lower bound exceeds $\delta$, so any margin retaining response $\zeta$
has $p_{\mathrm{rev}}(d;M)>\delta$.  Finiteness and monotonicity make the
largest sub-threshold window the grid point immediately preceding
$M_\dagger$ and rule out every larger calibrated scale.  This proves
$M_{\max}=M_{\max}^{\mathrm{cert}}$ in
\Eqref{eq:main-certified-maximum-context}.

To prove the explicit sufficient condition, observe that whenever
$v_M(r)>0$, \Eqref{eq:reversal-envelope-ratio-proof} and $x(r)\geq0$ give
$u_M(r)/v_M(r)\geq\eps/c_\eps$.  Monotonicity of the Gaussian CDF and
positive-part clipping imply
$\psi_M(r;\overline\tau_{\mathrm G})\leq\delta_\star$.
The same bound holds when $v_M(r)=0$, since $\psi_M$ then vanishes.
Consequently $\underline p_{\mathrm{rev}}(M)\leq\delta_\star<1/2$.
For $M\geq M_{\mathrm{suff}}$, the definition in
\Eqref{eq:reversal-sufficient-window} gives
$M\omega_{\min}\geq\Gamma_\eps(\rho)$ and
$M>\Gamma_\eps(\rho)\sqrt h/\zeta$.  Hence $k_N(M)=0$ and
$\ell(M;\zeta)>0$, so
\[
\frac{u_M(\ell(M;\zeta))}{v_M(\ell(M;\zeta))}
=\frac{\eps}{c_\eps},
\qquad
\underline p_{\mathrm{rev}}(M)=\delta_\star.
\]
If $M_+\geq M_{\mathrm{suff}}$ and $0\leq\delta<\delta_\star$, this
equality proves the strict upper inequality in
\Eqref{eq:reversal-admissible-threshold-range}.  Its lower inequality
ensures that the sub-threshold set is also nonempty, completing the proof.
\end{proof}
\subsection{Proof of Theorem~\ref{thm:fully-high-reversal}}
\label{app:fully-high-reversal}

\FullyHighReversal*

\begin{proof}
We first derive the finite-window reversal bound and then obtain the
asymptotic limit by letting the calibration error vanish.

Fix $0<\eps<1/2$ and a window for which $\cH_\eps(M)=\cF$.  Because
$d=a-b\neq0$,
\begin{equation}
U_H(d)=U(d)>0,
\qquad
Y_H^d=X_\cF^d(\mathbf m).
\label{eq:fully-high-identification-proof}
\end{equation}
Under the Gaussian-shape calibration assumed in the theorem,
Proposition~\ref{prop:high-calibration}, applied with $z=d$, gives
\begin{equation}
\sup_{x\in\mathbb R}
\abs{
\mathbb P[X_\cF^d(\mathbf m)\leq x]
-\Phi_{\mathrm{std}}\!\left(\frac{\sqrt{2}\,x}{U(d)}\right)
}
\leq
\tau_{\mathrm{H,G}}(d;M)+c_{\mathrm G}(\eps).
\label{eq:fully-high-cdf-bound-proof}
\end{equation}
Let $x_k=-1/k$.  Then $x_k\uparrow0$ and
$\{X_\cF^d(\mathbf m)<0\}=\bigcup_{k\geq1}
\{X_\cF^d(\mathbf m)\leq x_k\}$.  Continuity from below of probability and
continuity of $\Phi_{\mathrm{std}}$ allow $k\to\infty$ in
\Eqref{eq:fully-high-cdf-bound-proof}, which proves
\Eqref{eq:fully-high-pairwise-reversal}.  This left-limit argument controls
the strict reversal event without imposing a no-ties assumption.

For the asymptotic claim, the definition in the theorem gives
\begin{equation}
\Gamma_{\eps_M}(\rho)=M\omega_{h-1},
\qquad
\cH_{\eps_M}(M)=\cF,
\qquad
\eps_M\longrightarrow0.
\label{eq:vanishing-epsilon-fully-high-proof}
\end{equation}
Applying the finite-window reversal bound with $\eps=\eps_M$ and using
$\tau_{\mathrm{H,G}}(d;M)\to0$ together with
$c_{\mathrm G}(\eps_M)\to0$ proves
\Eqref{eq:fully-high-reversal-limit}.
\end{proof}

\section{Empirical Validation of the Theory}
\label{app:theory-validation}
\label{sec:experiment}
\label{app:historical-calibration-experiments}

We examine the numerical consistency and tightness of the local-response
bounds on fixed model coefficient vectors, then document the frequency
partition and score-distribution evidence used in the main text.
The local-response experiment measures the maximum normalized adjacent gap
that appears in the theoretical certificate. The task-level Positional
Score in Appendix~\ref{app:current-score-definitions} averages these gaps
over relative distances. All experiments below hold cached coefficient
vectors fixed while varying relative distance over the stated windows.
\subsection{Local Positional-Response Envelope and Local-Response Floor}
\label{subsec:experiment-positional-response}

We audit the sharper coefficient-specific refinement of the stable
local-response envelope, recorded in
Appendix~\ref{app:exact-gain-response-envelope}.  This experiment uses the theorem-side choice
$C_{\mathrm F}=2\pi$ and $\eps=10^{-2}$ at $M=32768$.  We evaluate Qwen3-8B
and Llama-3.1-8B using ten global queries, $200$ sampled key tokens, and all
$32$ first-layer attention heads.  Thus each model contributes $64{,}000$
query--head--key coefficient vectors.  Qwen3 has the geometric RoPE grid in
\Eqref{eq:grid}; the certified partition has $\abs{H}=8$ and $\abs{N}=56$,
so this model lies directly within the assumptions of the main theory.  For
Llama-3.1, we select the largest prefix satisfying the realized
signed-frequency separation condition in the Fourier-frame argument, which
gives $\abs{H}=9$ and $\abs{N}=55$.  Since Llama-3.1 uses scaled,
non-geometric RoPE frequencies, we report its row as a stress test outside the geometric-grid assumptions.  This subsection audits the
local result in Appendix~\ref{subsubsec:local-positional-response}; it does not
test the two-interval calibration in the near--far theorem of
Appendix~\ref{subsubsec:long-range-orderability}.

For each coefficient vector $a_i$, we use the observable gap $\zeta_i$ and
its pointwise certificate $T(a_i)$ from
Proposition~\ref{prop:exact-gain-normalized-envelope}.  The proposition and
Corollary~\ref{cor:coefficient-shape-floor-app} require, respectively,
\begin{equation}
\zeta_i\leq T(a_i),
\qquad
f_{\mathrm{pos}}(a_i;M,\zeta_i)\leq r_H(a_i;M).
\label{eq:experiment-positional-audit-inequalities}
\end{equation}

Figure~\ref{fig:experiment-positional-response} reports all $128{,}000$
vectors.  Every point satisfies both inequalities in
\Eqref{eq:experiment-positional-audit-inequalities}.  In the left column, the
median tightness ratios $\zeta_i/T(a_i)$ are $0.829$ for Qwen3 and $0.788$ for
Llama-3.1.  The middle column shows a strong monotone association between the
certified-high norm share and both responses: the Spearman correlations of $r_H$
with $(\zeta_i,T(a_i))$ are $(0.975,0.969)$ for Qwen3 and $(0.988,0.989)$ for
Llama-3.1.  This panel describes the sampled coefficient vectors,
because $T(a_i)$ retains the point-specific quantities $w_N(a_i;M)$ and
$\kappa_H(a_i;M)$.

In the right column, the median ratios
$f_{\mathrm{pos}}(a_i;M,\zeta_i)/r_H(a_i;M)$ are $0.782$ for Qwen3 and
$0.743$ for Llama-3.1.  Here $\zeta_i$ is the observed maximum gap of the
same coefficient vector used to compute $f_{\mathrm{pos}}$.  This is a same-sample consistency and tightness check. The proximity of the lower
bound to the equality line quantifies its empirical tightness.

\begin{figure}[H]
\centering
\includegraphics[width=\linewidth]{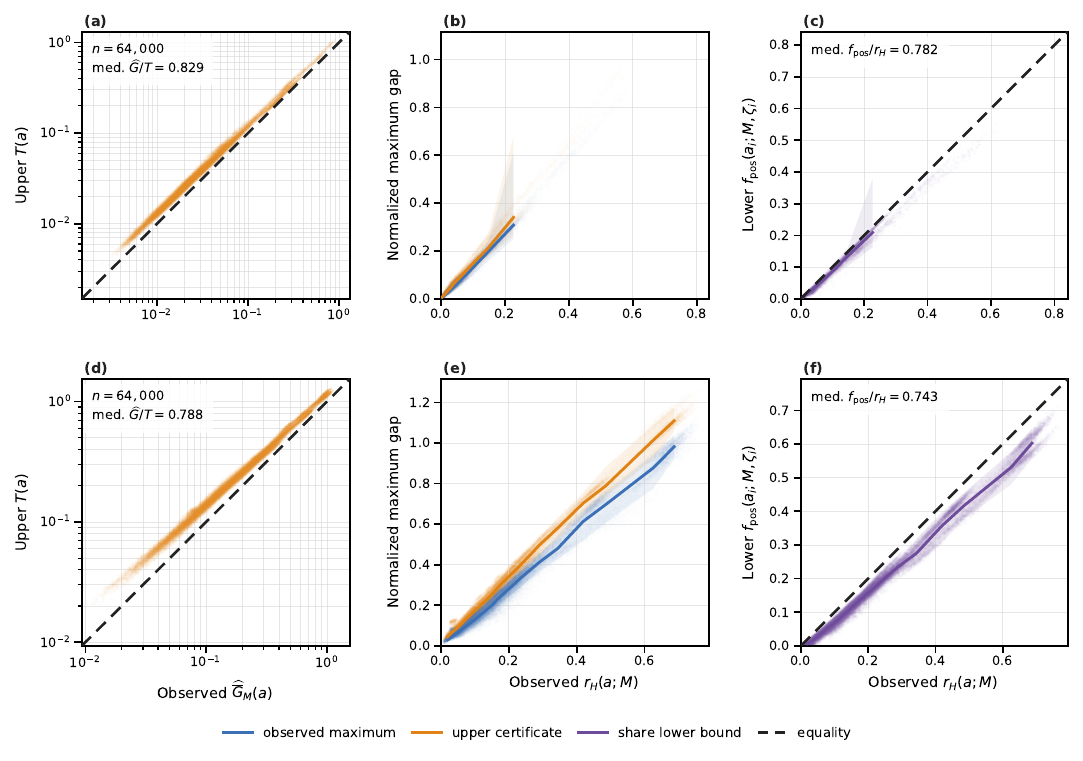}
\caption{Pointwise consistency and tightness audit of the sharper
coefficient-specific local positional-response envelope and local-response
floor from the appendix.  Panels (a)--(c) form the Qwen3-8B row and
panels (d)--(f) form the Llama-3.1-8B row; every point is
one first-layer query--head--key coefficient vector.  Left: the observed
integer-window maximum gap $\zeta_i$ against the coefficient-specific upper
certificate $T(a_i)$ on logarithmic axes.  Middle: the observed gap (blue) and
upper certificate (orange) against $r_H(a_i;M)$; solid curves and shaded
regions are medians and $10$--$90\%$ ranges over $32$ equal-count $r_H$ bins,
not an $r_H$-only predictor.  Right: the pointwise lower bound
$f_{\mathrm{pos}}(a_i;M,\zeta_i)$ against the observed certified-high norm share.
Dashed diagonals indicate equality in the left and right columns.  The Qwen3
row is an in-assumption geometric-grid audit, whereas the non-geometric Llama
row is an out-of-assumption stress test.  The right column uses the same point to set
$\zeta_i$ and therefore measures consistency and tightness rather than
out-of-sample predictive accuracy.}
\label{fig:experiment-positional-response}
\end{figure}

\subsection{Experimental Details for Figure 3}
\label{app:fig3-experimental-details}
\label{app:high-frequency-evidence}

This subsection gives the experimental protocols for
\S\ref{subsec:partition} and Figure~\ref{fig:main-high-frequency-cutoff}.

\subsubsection{Empirical Choice of the Frequency Split}
\label{app:high-frequency-cutoff-protocol}

For practical frequency selection, we use an empirical value
$C_{\mathrm F}^{\mathrm{emp}}$ for the constant in the theoretical split
criterion:
\begin{equation}
H_{\mathrm{op}}(M)=\left\{n:M\omega_n\geq
\frac{2C_{\mathrm F}^{\mathrm{emp}}}{(1-\rho)\eps_{\mathrm{ref}}}\right\},
\qquad C_{\mathrm F}^{\mathrm{emp}}=0.06,\quad\eps_{\mathrm{ref}}=0.1.
\label{eq:experiment-operational-partition}
\end{equation}
Here $\omega_n$ are the actual model frequencies, including Llama-3.1's
native frequency scaling, and $\rho=B^{-1/h}$ is computed from the base
grid. We choose $C_{\mathrm F}^{\mathrm{emp}}=0.06$ based on empirical
observations from the evaluated models and windows, and use this single
value across both models and all three context lengths. The selected
bands achieve randomness scores between $0.9936$ and $0.9997$ under the
metric defined below, showing close agreement between measured and
predicted noise amplitudes across all six settings.

The proof of Lemma~\ref{lem:frame-bound} uses $C_{\mathrm F}=2\pi$ as a
conservative sufficient constant for a uniform bound over arbitrary fixed
coefficient vectors. This choice establishes the theoretical guarantee in
Appendix~\ref{app:high}; $C_{\mathrm F}^{\mathrm{emp}}$ adapts the split
criterion to the observed model coefficients. The reference value
$\eps_{\mathrm{ref}}$ controls the empirical threshold, and we assess the
resulting split through measured noise-amplitude agreement. The empirical rule
also applies to Llama-3.1's scaled frequency grid, whose evaluation lies
outside the geometric-grid assumptions of the theoretical guarantee.
The intervention experiments retain their original parameter
$c_{\mathrm{int}}=0.05$, as documented in
Appendix~\ref{app:intervention-experiment}.

For Figures~\ref{fig:high-frequency-cutoff-qwen}
and~\ref{fig:high-frequency-cutoff-llama}, we evaluate
Qwen3-8B~\citep{qwen3} and Llama-3.1-8B~\citep{llama3} at
$M\in\{8192,32768,131072\}$. For each model, the same $40{,}960$
first-layer query--pair--head coefficient vectors are used at all three
window lengths. We hold these cached vectors fixed and evaluate their RoPE
score margins over $m\in\{0,\ldots,M-1\}$, using the paper's convention
$m=\text{query position}-\text{key position}$ and margin coefficients
$d_n=Q_n\overline{(K_{a,n}-K_{b,n})}/\sqrt{d_{\mathrm{head}}}$ for keys
$a$ and $b$. Thus the 128k curves extend the
relative-distance window of fixed representations; they require no 128k
model forward pass. For every prefix $H_k=\{0,\ldots,k-1\}$, we compute
the aggregate variance ratio
\begin{equation}
\mathcal R_{\mathrm{var}}(H_k;M)
:=
\frac{\sum_i \Var_M(X_{H_k}^{d_i})}
{\sum_i U_{H_k}(d_i)^2/2},
\label{eq:experiment-aggregate-variance-ratio}
\end{equation}
where $i$ indexes the sampled query--pair--head points and $\Var_M$ uses
the uniform distribution over the $M$ integer distances. We then compare
noise amplitudes through
\begin{equation}
a(H_k;M)=\sqrt{\mathcal R_{\mathrm{var}}(H_k;M)},
\qquad
R(H_k;M)=\max\{0,1-|a(H_k;M)-1|\}.
\label{eq:experiment-randomness-proxy}
\end{equation}
The amplitude ratio $a$ is formed after aggregating the measured and
predicted variances. The plotted randomness proxy $R\in[0,1]$ equals one
at exact amplitude agreement and decreases with the absolute relative
amplitude error. For example, ratios $a=0.9$ and $a=1.1$ both give
$R=0.9$. Positive cross-frequency covariance can give
$\mathcal R_{\mathrm{var}}>1$; the transformation preserves this discrepancy
as a decrease in $R$. We retain all resulting dips and recoveries.
This metric measures noise-scale agreement and supplies no test of
Gaussian shape or temporal independence.

In Figures~\ref{fig:high-frequency-cutoff-qwen}
and~\ref{fig:high-frequency-cutoff-llama}, the gray curves plot $R(H_k;M)$
as the number $k$ of retained components increases. Orange circles mark
the largest prefixes satisfying
\Eqref{eq:experiment-operational-partition} for 8k, 32k, and 128k windows,
from left to right. The main axes magnify the selected cutoffs and nearby
changes; insets show all 64 prefixes, with orange boxes indicating the
enlarged regions.

\subsubsection{Score-Distribution Illustration}
\label{app:high-frequency-distribution-protocol}
Figure~\ref{fig:high-frequency-gaussian-qwen} reuses the Qwen3-8B layer-0
projection bank. For each of its 10 query tokens and 32 query heads, we
uniformly sample four distinct key tokens from the existing 200-token key
bank (seed 20260923), giving 1,280 query--key--head combinations. These are
individual attention scores; Figures~\ref{fig:high-frequency-cutoff-qwen}
and~\ref{fig:high-frequency-cutoff-llama} use score margins.
For each fixed combination $i$, we evaluate $S_{H_{40}}^{(i)}(m)$ at every
$m\in\{0,\ldots,32767\}$ and divide by
$\sigma_i=U_{H_{40}}(z^{(i)})/\sqrt{2}$, where
$H_k=\{0,\ldots,k-1\}$. The gray bars average the resulting densities with equal
weight. No empirical centering or variance fitting is used.
The orange curve is $\mathcal N(0,1)$ under this normalization. This
illustration uses the 32k cutoff from
\Eqref{eq:experiment-operational-partition} and the same cached projections,
with no new model forward pass.

\subsection{Additional Evidence on High-Frequency Scores}
\label{app:additional-high-frequency-evidence}

\subsubsection{The Concentrated Qwen Frequency Behind the Noise-Amplitude Drop}
\label{app:qwen-frequency-cliff}

The sharp drop in the inset of Figure~\ref{fig:high-frequency-cutoff-qwen}
occurs when the retained prefix grows from 51 to 52 frequencies. The newly included
frequency has zero-based index $n=51$, so it is the 52nd complex coordinate
pair, corresponding to real channels 51 and 115 in the split-half layout.
On the same 40,960 margin-coefficient vectors $d_i$ used in that panel,
its aggregate coefficient-energy share is
\begin{equation}
\frac{\sum_i |d_{i,51}|^2}{\sum_i\sum_n |d_{i,n}|^2}
=92.3507\%.
\end{equation}
This is the fraction of the summed squared coefficient norms, with each
sample entering the sum once. It is distinct from a fraction of the
measured score variance or a norm share before squaring.

At $M=32768$, this frequency rotates by only $M\omega_{51}=0.54225$
radians over the window. Its own finite-window variance is 4.4599\% of
the half-energy prediction. Including it multiplies the total predicted
variance by 16.1731, while the aggregate variance ratio
$\mathcal R_{\mathrm{var}}$ falls from 0.883744 to 0.086922 and the plotted
randomness proxy $R$ falls from 0.940077 to 0.294826.
The slowly varying component therefore adds
substantial coefficient energy without the variance expected from a fast
oscillation. This accounts for the abrupt drop and illustrates the need
to distinguish coefficient magnitude from frequency-dependent variation.
The component lies outside all three operational high-frequency prefixes.

\subsubsection{Representation Assumptions and Gaussian Score Shape}
\label{app:high-frequency-assumption-comparison}

The moment guarantees in Proposition~\ref{prop:high-calibration} allow
arbitrary fixed coefficient magnitudes and phases. Relevant prior analyses
use several different assumptions. \citet{NEURIPS2024_9f12dd32} derive
their expected preference bound from independent identically distributed
query and key coordinates with a common variance and a similar-key noise
model. Their theorem does not require Gaussian coordinate distributions.
\citet{du2026ropedistinguishespositionstokens} also fix the query and key
vectors and sample relative distance. Their high-frequency analysis gives
the near-zero mean and half-energy variance approximation; its normal
approximation uses regular amplitudes without a dominant component and
approximate independence across frequencies. Our finite-window moment
bounds control cross-frequency terms explicitly and allow concentrated
coefficient energy.

Deterministic analyses also appear in \citet{Su_2024},
\citet{ICLR2025_e6d58fc6}, and \citet{ICLR2026_68f1a693}.
In particular, the Gaussian query/key model in Proposition 3.2 of
\citet{ICLR2025_e6d58fc6} is a counterexample to universal distance decay;
their other constructions and semantic-attention results use fixed
vectors. These results establish that fixed-vector analysis is already
part of the literature. The present contribution is the finite-window
moment control and the resulting frequency criterion. Gaussian score
shape remains a separate approximation: the displayed pooled density
supports it empirically, and the figure does not establish that every
individual query--key combination has a Gaussian window distribution.

\section{The 49 Reference Task Settings}
\label{app:diagnostic-task-set}

The diagnostic reference set contains 49 task/configuration settings from
nine benchmark sources. A setting combines a task with a difficulty or
input-length configuration. Table~\ref{tab:diagnostic-task-inventory}
lists every setting by grouping configurations of the same task. The set
covers key and document retrieval, multi-hop question answering, text and
list ordering, state tracking, and long-context reasoning.
For example, the LIFBench adjacent-element retrieval task contributes
three length settings (3k, 6k, and 13k), while RULER2 multi-key retrieval
contributes four difficulty settings (basic, easy, medium, and hard).
These 49 settings form the ranking reference set in
Figure~\ref{fig:diagnostic-workflow}; each point in
Figure~\ref{fig:semantic-preference-cross-model} and each bar in
Figure~\ref{fig:intervention-accuracy-gain} corresponds to one setting.

Qwen3-8B and Llama-3.1-8B-Instruct use the same example IDs within each
setting: 2,211 examples per model and 4,422 example--model records in total.
We record query and key activations during standard evaluation of these
examples. The resulting diagnostics supplement each task's behavioral
accuracy with information about relative failure susceptibility.

\paragraph{Sources and configurations.}
The 12 MK, MV, and QA settings use the official NeMo Skills RULER2
suite.\footnote{\href{https://github.com/NVIDIA-NeMo/Skills/blob/cb54e911ad5b2cee87444fc89656fabca021c8cc/nemo_skills/dataset/ruler2/prepare.py}{NeMo Skills RULER2 task definitions.}}
We use its nominal 8k configuration. Each family has basic, easy, medium,
and hard variants, with
100 examples per configuration. The two LongBench tasks use unmodified
examples whose native chat inputs fit within 16,384 tokens for both
models. L-Eval TopicRet contributes 150 question records from 50
conversations, with three topic-retrieval questions per conversation.
FLenQA uses the books-padding, random-evidence configuration at context
sizes 500 and 3,000. NoLiMa uses book 1 with the needle at 52\% depth.
The three LIFBench settings contain 12 examples each; the remaining
33 settings contain 25 examples each.

Length labels in the table retain the source benchmark's configuration
names. For the diagnostic calculation, $M$ is the complete input length
under the evaluated model's tokenizer, including its prompt and chat
template. The same length label can therefore yield different $M$ across
models. The counts below describe the diagnostic reference set; matched
coverage for the intervention experiments is reported separately in
Appendix~\ref{app:current-intervention}.

\begingroup
\footnotesize
\setlength{\tabcolsep}{3pt}
\renewcommand{\arraystretch}{1.12}
\begin{longtable}{@{}>{\raggedright\arraybackslash}p{0.99in}>{\raggedright\arraybackslash}p{2.10in}>{\raggedright\arraybackslash}p{1.17in}rr@{}}
\caption{The complete 49-setting diagnostic reference set. Each configuration
in a row is a separate setting. $K$ is the number of settings and $N$ is
the number of examples per setting and per model. Lengths are nominal
benchmark configurations; the diagnostic uses each model's actual input
length.}\label{tab:diagnostic-task-inventory}\\
\toprule
Task & What the model must do & Configurations & $K$ & $N$ \\
\midrule
\endfirsthead
\multicolumn{5}{l}{\tablename~\thetable\ (continued)}\\
\toprule
Task & What the model must do & Configurations & $K$ & $N$ \\
\midrule
\endhead
\midrule
\multicolumn{5}{r}{Continued on the next page}\\
\endfoot
\bottomrule
\endlastfoot
\multicolumn{5}{@{}l}{\textbf{RULER2 (NeMo Skills)}}\\*
Multi-key (MK) & Retrieve by key; harder variants answer retrieved MMLU questions. & Basic, easy, medium, hard & 4 & 100 \\
Multi-value (MV) & Retrieve multiple values or select an ordered question for a key. & Basic, easy, medium, hard & 4 & 100 \\
Question answering & Retrieve supporting documents and answer HotpotQA questions. & Basic, easy, medium, hard & 4 & 100 \\
\addlinespace[3pt]
\multicolumn{5}{@{}l}{\textbf{LongBench-v1}~\citep{bai-etal-2024-longbench}}\\*
Qasper & Answer questions about an academic paper. & Native chat $\leq16$k & 1 & 25 \\
MuSiQue & Combine evidence across documents for multi-hop answers. & Native chat $\leq16$k & 1 & 25 \\
\addlinespace[3pt]
\multicolumn{5}{@{}l}{\textbf{L-Eval}~\citep{an-etal-2024-l}}\\*
TopicRet & Retrieve the first, second, or third conversation topic. & Native conversation & 1 & 150 \\
\addlinespace[3pt]
\multicolumn{5}{@{}l}{\textbf{Ada-LEval}~\citep{wang-etal-2024-ada}}\\*
TSort & Restore the original order of shuffled text fragments. & 1k & 1 & 25 \\
\addlinespace[3pt]
\multicolumn{5}{@{}l}{\textbf{FLenQA}~\citep{levy-etal-2024-task}}\\*
PIR & Infer a room property through a person--room relation. & 500; 3,000 & 2 & 25 \\
MonoRel & Infer a transitive conclusion from ordered relations. & 500; 3,000 & 2 & 25 \\
Simplified RuleTaker & Deduce a true/false conclusion from rules and facts. & 500; 3,000 & 2 & 25 \\
\addlinespace[3pt]
\multicolumn{5}{@{}l}{\textbf{LongReason}~\citep{ling2025longreasonsyntheticlongcontextreasoning}}\\*
Mixed reasoning & Solve reading, logical, and mathematical reasoning questions. & 8k; 16k & 2 & 25 \\
\addlinespace[3pt]
\multicolumn{5}{@{}l}{\textbf{BABILong}~\citep{NEURIPS2024_babilong}}\\*
qa1: person location & Track a person's current location. & 8k; 16k & 2 & 25 \\
qa2: object location & Track an object's current location. & 8k; 16k & 2 & 25 \\
qa3: previous location & Recover an object's location before a specified event. & 8k; 16k & 2 & 25 \\
qa4: directional relations & Infer a directional relation between two entities. & 8k; 16k & 2 & 25 \\
qa5: giving events & Identify the giver, recipient, or object in a giving event. & 8k; 16k & 2 & 25 \\
qa7: counting & Count the objects a person currently holds. & 8k; 16k & 2 & 25 \\
qa8: lists and sets & List the objects a person currently holds. & 8k; 16k & 2 & 25 \\
qa9: negation & Answer location questions with explicit negation. & 8k; 16k & 2 & 25 \\
qa10: indefinite knowledge & Distinguish known, false, and uncertain locations. & 8k; 16k & 2 & 25 \\
\addlinespace[3pt]
\multicolumn{5}{@{}l}{\textbf{LIFBench}~\citep{wu-etal-2025-lifbench}}\\*
List offset query & Return the item immediately before or after a target. & Official base 3k; 6k; 13k & 3 & 12 \\
\addlinespace[3pt]
\multicolumn{5}{@{}l}{\textbf{NoLiMa}~\citep{pmlr-v267-modarressi25a}}\\*
One-hop retrieval & Retrieve a fact through one implicit semantic association. & 8k; 16k & 2 & 25 \\
Two-hop retrieval & Retrieve a fact through two implicit semantic associations. & 8k; 16k & 2 & 25 \\
\midrule
\multicolumn{3}{@{}l}{\textbf{Total: 49 settings; 2,211 examples per model}} & \textbf{49} & \\
\end{longtable}
\endgroup

\section{Token and Head Selection}
\label{app:selection-protocol}

\subsection{Token Sampling from the Evaluated Inputs}
\label{app:token-selection}

The current diagnostic reuses the recorded query and key activations for
the 49 reference settings in Appendix~\ref{app:diagnostic-task-set}.
Each example is tokenized with the evaluated model's native chat template.
Character spans supplied by the task preparation identify the question
and its context. Eligible token occurrences overlap the corresponding
span, contain at least one letter or digit within that overlap, and have
nonempty character offsets; special tokens are excluded. We select up to
four query positions from the question and up to 16 key positions from
the context. Every selected key precedes every selected query in the
original input.

Selection uses a deterministic SHA-256 ordering without replacement.
For the original 12 RULER2 settings, the ordering is keyed by the protocol
identifier, example ID, sampling role, and token position. The remaining
37 settings also include the fixed seed 20260910. Separate sampling roles
select queries, keys, and the ordering used to pair keys. The saved token
positions remain fixed throughout the diagnostic calculations.

For the \emph{Common-pair} branch, each query is combined with every
selected key, giving up to 64 query--key pairs per example. For the
\emph{Common-triplet} branch, the selected keys are reordered with the
key-pairing role and grouped into eight disjoint pairs; each query is
combined with these pairs, giving up to 32 triplets with distinct keys.
All 2,211 examples per model have 16 selected keys. Of these, 2,161 have
four selected queries and yield 64 pairs and 32 triplets. The other 50
examples, from the two BABILong person-location settings, have three
eligible queries and yield 48 pairs and 24 triplets.
The Common branches use these sampled occurrences without task-target
labels. The semantic calculation subsequently orients each triplet by
its score ordering at zero relative distance, as defined in
Appendix~\ref{app:current-score-definitions}.

\subsection{Selecting a Fixed Set of Attention Heads}
\label{app:head-selection}

We select heads whose high-frequency norm share varies across the
reference settings. For each layer--head unit $h$, we first average the
Common-pair value $r_H$ over valid pairs within each example, then average
over examples within each setting. Let $\bar r_{t,h}$ be this mean for
setting $t$. With all 49 settings weighted equally, the selection statistic
is the sample variance
\begin{equation}
v_h=\frac{1}{48}\sum_{t=1}^{49}
\left(\bar r_{t,h}-\frac{1}{49}\sum_{s=1}^{49}\bar r_{s,h}\right)^2.
\label{eq:head-selection-variance}
\end{equation}
Each model is ranked separately by decreasing $v_h$. Ties are resolved
by increasing layer-major head index; neither model has a tie at the
reported selection boundaries. A fraction $p$ retains
$\lceil pH\rceil$ of the model's $H$ heads.

Qwen3-8B has 36 layers and 32 query heads per layer, for 1,152 heads in
total; its top 5\% and top 10\% sets contain 58 and 116 heads.
Llama-3.1-8B-Instruct has 32 layers and 32 query heads per layer, for
1,024 heads; its corresponding sets contain 52 and 103 heads.
Figure~\ref{fig:head-selection-heatmap} shows the selection across layers
and heads. The top-5\% set is fixed across tasks and supplies both the
Semantic and Positional Scores in the main diagnostic. The second
intervention stage expands this same ranking to the top 10\%, as detailed
in Appendix~\ref{app:current-intervention}. This ranking describes
variation within the chosen reference set, including its different
input-length settings.

\begin{figure}[t]
\centering
\includegraphics[width=\linewidth]{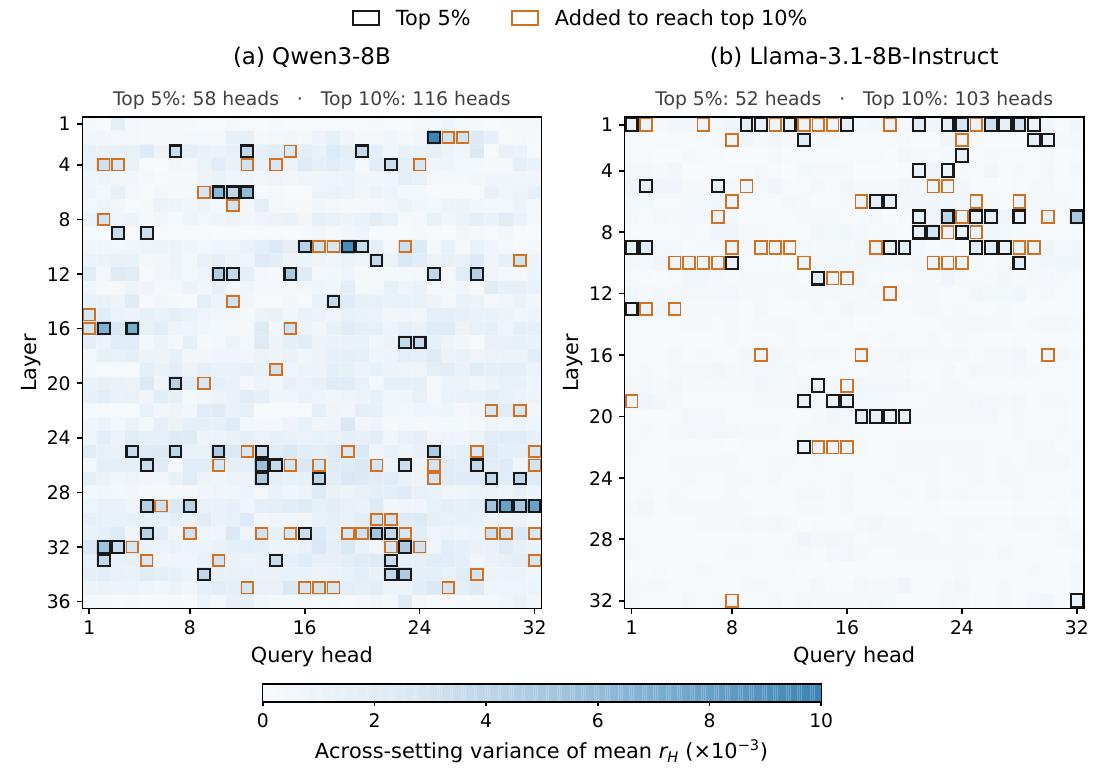}
\caption{Head selection from the 49 reference settings. Each cell shows
the sample variance of a head's setting-level mean Common-pair $r_H$,
using a shared color scale across models. Black outlines identify the
top-5\% heads used by the main diagnostic and first intervention stage.
Orange outlines identify the additional heads included in the top-10\%
set for the second stage. Layer and query-head indices start at one.
All settings within a model use the same selected heads.}
\label{fig:head-selection-heatmap}
\end{figure}

\section{Experimental Details}
\label{app:experiment-details}
\label{app:current-diagnostic-protocol}

We conduct two experiments on Qwen3-8B and Llama-3.1-8B-Instruct:
we compare diagnostic profiles across the 49 shared task settings, and
we evaluate high-frequency interventions in each setting's assigned
semantic or positional direction. Appendix~\ref{app:diagnostic-task-set}
specifies the task data, and Appendix~\ref{app:selection-protocol}
describes the fixed token and head selections.

\subsection{Experiment 1: Task and Cross-Model Diagnostic Profiles}
\label{app:diagnostic-experiment}

Each model uses 2,211 examples across the same 49 settings.
The primary diagnostic uses the fixed top-5\% head collection selected
from cross-task variation in the high-frequency norm share.
All scores are computed from cached query and key states at each example's
native input length and frequency configuration.
\subsubsection{State Extraction and Score Definitions}
\label{app:current-score-definitions}

For each task example, we use sampled query--key pairs and query--key--key
triplets, together with a fixed collection of layer--head units. In each
rotary plane, let $Q_n$ and $K_n$ denote the complex forms of the pre-RoPE
query and key components. Their score coefficient is
\begin{equation}
z_n=c_{\mathrm{att}}Q_n\overline{K_n},\qquad
S(m)=c_0+\Rea\sum_n z_ne^{im\omega_n},
\label{eq:current-tool-coefficients}
\end{equation}
where $c_{\mathrm{att}}$ is the model's attention-score scale and $c_0$ is
any non-rotary score contribution (zero for fully rotary heads). The extraction
preserves native query/key normalization, rotary-coordinate pairing, and the
mapping between query heads and key/value heads. For partially rotary heads,
the non-rotary contribution is retained as a constant in semantic margins;
the norm used for positional response refers to the rotary coefficients.

We keep these content vectors fixed and vary only the current layer's RoPE
phase over $m=0,\ldots,M-1$. The current task protocol sets $M$ to the
example's input length. The score at zero relative distance supplies the reference
comparison. These are virtual score evaluations and require no new forward
pass once the states have been recorded.

\paragraph{Semantic Score.}

For each triplet, orient the margin (including any non-rotary constant) so
that $D(0)>0$, with $d=z_+-z_-$ in this ordering, and compute
the full finite-window mean $\mu_D$ and variance $\sigma_D^2$ analytically.
Finite geometric sums give these moments, including the covariance between
frequency components; Appendix~\ref{app:current-analytic-moments} provides
the expressions. With $\Phi$ denoting the standard normal cumulative
distribution function, the Gaussian estimate is
\begin{equation}
\widehat p_{\mathrm{rev}}(d;M)
=\Phi\!\left(-\frac{\mu_D}{\sigma_D}\right),\qquad\sigma_D>0.
\label{eq:current-semantic-probability}
\end{equation}
We average these probabilities over valid triplets, then over the fixed
head collection, and finally over task examples with equal example weights.
The \emph{Semantic Score} is one minus this task-level reversal rate. Higher
values indicate more stable reference orderings.
Reference ties are excluded and their counts are retained. For zero variance,
we evaluate the deterministic strict-negative event directly.

The probability calculation uses the accepted Gaussian approximation and
analytic moments. It avoids enumerating reversal events at every position.
The high-frequency share provides the theoretical interpretation in
\S\ref{sec:theory}; the practical estimate uses the full finite-window
moments and coefficient vector. Its value may vary non-monotonically with
the window even though the conservative theoretical lower bound is
monotone. This distinction allows the tool to report the behavior of sampled
coefficients while the theory explains a necessary constraint on their joint
semantic and positional reliability.

\paragraph{Positional Score.}

For each query--key pair with $U(z)>0$, define
\begin{equation}
P(z;M):=\frac{1}{M-1}\sum_{m=0}^{M-2}
\frac{|S(m+1)-S(m)|}{U(z)}.
\label{eq:current-positional-score}
\end{equation}
We first normalize each pair by its own coefficient norm, then average over
pairs, the same fixed head collection, and task examples. The resulting
\emph{Positional Score} measures average normalized adjacent response.
Higher values indicate stronger local response. The score is computed
directly from adjacent differences; a zero coefficient norm is recorded as
undefined. This convention preserves the per-pair scale before aggregation.

The response condition in Theorem~\ref{thm:maximum-reliable-context} applies
to the pairwise margin $d=a-b$. Requiring the average adjacent response of
that same margin to reach $\zeta$ also implies the maximum-response condition
used in the proof. The Positional Score summarizes individual query--key
curves and supplies a complementary diagnostic of local response. The
theorem identifies a spectral constraint when both requirements are imposed
on the same margin. The task-level profile itself carries no context-length
certificate.

\subsubsection{Analytic Finite-Window Moments}
\label{app:current-analytic-moments}

For a real angular frequency $\theta$, define
\begin{equation}
K_M(\theta):=\frac1M\sum_{m=0}^{M-1}e^{im\theta}
=\begin{cases}
\displaystyle\frac{1-e^{iM\theta}}{M(1-e^{i\theta})},
&e^{i\theta}\ne1,\\
1,&e^{i\theta}=1.
\end{cases}
\label{eq:current-moment-kernel}
\end{equation}
For $D(m)=c_D+\Rea\sum_n d_ne^{im\omega_n}$, let
\begin{align}
\mu_D&=c_D+\Rea\sum_n d_nK_M(\omega_n),\\
C^-_{n\ell}&=K_M(\omega_n-\omega_\ell)
-K_M(\omega_n)\overline{K_M(\omega_\ell)},\\
C^+_{n\ell}&=K_M(\omega_n+\omega_\ell)
-K_M(\omega_n)K_M(\omega_\ell).
\end{align}
Expanding the square of the real part gives the exact variance
\begin{equation}
\sigma_D^2=\frac12\Rea\sum_{n,\ell}
\left(d_n\overline{d_\ell}C^-_{n\ell}
+d_nd_\ell C^+_{n\ell}\right).
\label{eq:current-analytic-variance}
\end{equation}
Here $c_D$ is the difference of any non-rotary score constants and is zero
for fully rotary heads. It contributes to the reference margin $D(0)$ and
to $\mu_D$, and leaves the variance unchanged.
These expressions require no probability distribution on $d$. The subsequent
Gaussian threshold probability is an approximation to the relative distance law,
with the previously established calibration retained.

\subsubsection{Relative Ranks and Failure Susceptibility}
\label{app:current-rank-boundary}

For a chosen model and head collection, rank a fixed reference set of $N>1$
task/length settings by decreasing Semantic Score and decreasing Positional
Score. Let $r_S$ and $r_P$ be the respective ranks, with rank one best and
equal values receiving equal competition ranks. We define the relative
semantic and positional failure susceptibilities as
\begin{equation}
w_S=\frac12+\frac{r_S-r_P}{2(N-1)},\qquad w_P=1-w_S.
\label{eq:current-relative-tendency}
\end{equation}
A larger $w_S$ indicates a relatively weaker semantic rank, and a larger
$w_P$ indicates a relatively weaker positional rank. Equal ranks give
$w_S=w_P=1/2$. These indices describe relative failure susceptibility within the chosen
comparison set; their interpretation depends on that set. Raw scores and
behavioral success rates remain visible alongside these indices.

\subsubsection{Aggregation and Cross-Model Comparison}
Semantic reversal probabilities are averaged over valid triplets within
each example and head, then over the selected heads, then equally over
examples. The Semantic Score is one minus this average probability.
For the Positional Score, each pair's adjacent differences are divided by
its own coefficient norm before averaging over pairs, heads, and examples.
Reference ties satisfy
$|D(0)|\leq64\epsilon_{\mathrm{mach}}\sum_n|d_n|$; their counts and the
counts of undefined scores are retained.

For Figure~\ref{fig:semantic-preference-cross-model}, each model ranks all
49 settings by decreasing $w_S$, assigning average ranks to ties.
We compute Spearman correlation as the Pearson correlation between these
two vectors of average ranks, with each setting weighted equally.
The resulting correlation is $0.657983$, reported as $0.658$ in the main text.
It includes all 49 settings, including the three adjacent-element retrieval
settings highlighted in orange. The plot preserves tied coordinates
without jitter. These rankings describe failure susceptibility within the
shared reference set.

\subsection{Experiment 2: Directed High-Frequency Intervention}
\label{app:intervention-experiment}

\subsubsection{Assigning the Optimization Direction}
The intervention direction ranks the 49 settings by decreasing
$w_S$, assigning the average rank to ties. Let $\rho_S$ denote this rank.
Settings with $\rho_S/49\leq1/2$ test reduced high-frequency contributions;
the remaining settings test amplification. This places 24 Qwen settings
and 23 Llama settings in the reduction group, keeping ties together.
For this reference set, reduction corresponds to $w_S$ strictly above
the model-specific median: $50.00\%$ for Qwen and $56.25\%$ for Llama.
The value $w_S=1/2$ separately indicates equal semantic and positional ranks.

The reduction group optimizes semantic stability by attenuating
high-frequency score contributions or freezing their rotations.
The amplification group optimizes positional sensitivity by increasing
the high-frequency contribution. The assigned direction remains fixed
through both stages of the search.

\subsubsection{Intervention Operators and Frequency Mask}
For a scalar $\alpha$, selected high-frequency query and key components
are each multiplied by $\sqrt{\alpha}$, which multiplies their score
coefficient by $\alpha$. Low-frequency components and unselected heads
retain their original operations. We preserve native query/key
normalization and the mapping between query heads and key/value heads.
The transformation changes both the coefficient norm and its high-frequency
share. A separate no-rotation condition freezes the selected high-frequency
RoPE rotations and belongs to the semantic direction.

The high-frequency mask is fixed from prompt length $M$ throughout generation:
\begin{equation}
H=\{n:M\omega_n\geq\gamma\},\qquad
\gamma=\frac{2c_{\mathrm{int}}}{(1-\rho)\epsilon},\qquad
\rho=\theta^{-2/d_{\mathrm{rot}}},\quad c_{\mathrm{int}}=0.05,\quad\epsilon=0.10.
\label{eq:intervention-operational-mask}
\end{equation}
Here $\theta$ is the RoPE base, $d_{\mathrm{rot}}$ is the rotary dimension,
and $\omega_n$ are the actual runtime frequencies, including native
frequency scaling. These intervention runs retain their original empirical
parameter $c_{\mathrm{int}}=0.05$ throughout both search stages.
The separate frequency-split audit in
Appendix~\ref{app:high-frequency-cutoff-protocol} uses $c_{\mathrm{op}}=0.06$;
its later calibration leaves the intervention masks and reported outcomes unchanged.

\subsubsection{Two Search Stages}
The first stage uses the diagnostic's top-5\% head collection:
58 heads for Qwen and 52 for Llama.
Semantic-direction settings evaluate no rotation and
$\alpha\in\{0,0.5\}$; positional-direction settings evaluate
$\alpha\in\{1.5,2,2.5\}$.
The unchanged $\alpha=1$ condition provides the baseline.
Candidate and baseline outputs share the original input IDs, greedy
decoding, stopping rules, and first-stage output limits.
The $\alpha=1$ implementation was checked against native-model outputs.

The second stage expands the same fixed head ranking to its top 10\%:
116 Qwen heads and 103 Llama heads. It selects the 25 Qwen and 23 Llama
settings without an observed first-stage improvement in their assigned
direction. Semantic-direction settings evaluate no rotation and
$\alpha\in\{0.25,0.5,0.75\}$; positional-direction settings evaluate
$\alpha\in\{1.25,1.5,1.75,2\}$.
The diagnostic $w_S$ and task rankings remain fixed at their top-5\% values.

\subsubsection{Matched Examples, Scoring, and Output Budgets}
We preserve the frozen first-stage comparison sets: 2,171 matched examples
for Qwen and 2,074 for Llama. Each model has 42 fully covered settings
and seven partially covered settings. The second stage completed 8,440
intervention records. Matching its candidates to saved baseline outputs
gives 945 Qwen and 1,071 Llama examples; 16 and 78 examples without baseline
outputs are excluded. A missing or untested result remains unreported.

Within each stage and setting, every displayed candidate is compared with
the original $\alpha=1$ baseline on the same examples. Accuracy is the
fraction of examples receiving an official task score of one.
Appendix~\ref{app:current-intervention} reports the full selected-direction
candidate results and matched counts. The stages retain their own
comparison sets: Qwen MV (easy) has 94 examples in the first stage and
95 in the second; the other second-stage settings use the same examples
as their first-stage comparisons.

Figure~\ref{fig:intervention-accuracy-gain} intersects the comparison sets
across both stages, recomputes their candidates and baseline on this common
set, and displays the largest observed gain in the assigned direction.
Thus Qwen MV (easy) uses 94 examples for both stages in the figure.
The baseline is excluded from the candidate maximum, so settings for which
every candidate reduces accuracy retain a negative bar.

The second stage reuses original input IDs and baseline outputs, with
shorter generation limits for some tasks. Its reported comparisons retain
the original baseline budget. An offline control truncates saved baseline
tokens to the shorter limits before rescoring; this changes three Qwen and
28 Llama baseline correctness labels. The original baseline is used in
all tables and the main figure.

The search reports the best observed candidate on the examples used to
compare intervention strengths. Second-stage settings are selected from
first-stage results. These gains describe the available interventions in
this exploratory search; choosing an intervention for deployment requires
evaluation on held-out examples.

\section{Results in the Assigned Optimization Direction}
\label{app:current-intervention}

We report every measured candidate in each setting's assigned semantic or
positional direction, following the protocol in
Appendix~\ref{app:intervention-experiment}. The tables include all 49 settings
per model, grouped by direction, with separate results for the top-5\% and
top-10\% head collections.

\paragraph{Observed gains.}
The first stage improves 24 of 49 Qwen settings and 26 of 49 Llama settings.
The second stage adds seven Qwen settings and six Llama settings, giving
\textbf{31/49 (63.3\%) and 32/49 (65.3\%)} improved settings, respectively.
The largest gains are 20 percentage points for Qwen and 25 for Llama.
These are the best observed outcomes from a search over intervention strengths
and two head subsets. The selected candidates have not been validated on
held-out examples.

\paragraph{Reading the tables.}
Each table retains its stage's matched examples and original baseline.
The count $n/N$ gives matched and planned examples. Candidate headers give
$\alpha$, and NR denotes no rotation. Each entry gives accuracy (\%) followed
by its change from the paired baseline (percentage points).
Dashes indicate settings without a second-stage evaluation, and $\dagger$
marks a second-stage generation limit that differs from the saved baseline.
All measured gains, including zero and negative values, are retained.
Figure~\ref{fig:intervention-accuracy-gain} uses the common samples across
stages: Qwen MV (easy) uses 94 examples in the figure and 95 in its
second-stage table. All other settings use the same paired cohort in the
figure and tables. The adjacent-element retrieval settings retain their
original 12 examples at each length.

The output-limit control described in
Appendix~\ref{app:intervention-experiment} yields 8/25 and 7/23 improving
second-stage settings when the saved baseline outputs are truncated to the
second-stage limits before rescoring. The tables retain the original-output
comparisons of 7/25 and 6/23.

\subsection{Qwen3-8B: Semantic Direction}

\begingroup
\scriptsize
\setlength{\tabcolsep}{2pt}
\renewcommand{\arraystretch}{1.12}
\begin{longtable}{@{}>{\raggedright\arraybackslash}p{1.55in}rrrrrr@{}}
\caption{Qwen3-8B: semantic direction on the top-5\% heads. Entries give accuracy in percent and the change from the paired baseline in parentheses (percentage points).}\label{tab:direction-qwen-semantic-5}\\
\toprule
Setting & $w_S$ (\%) & $n/N$ & Base & NR & $0$ & $0.5$ \\
\midrule
\endfirsthead
\multicolumn{7}{l}{\textit{Continued from the preceding page.}}\\
\toprule
Setting & $w_S$ (\%) & $n/N$ & Base & NR & $0$ & $0.5$ \\
\midrule
\endhead
\midrule
\multicolumn{7}{r}{\textit{Continued on the next page.}}\\
\endfoot
\bottomrule
\endlastfoot
QA (basic) & 54.2 & 100/100 & 78.00 & 70.00\,(-8.00) & 76.00\,(-2.00) & 75.00\,(-3.00) \\
QA (easy) & 54.2 & 100/100 & 76.00 & 66.00\,(-10.00) & 70.00\,(-6.00) & 77.00\,(+1.00) \\
Multi-hop QA (up to 16k) & 53.1 & 25/25 & 28.00 & 24.00\,(-4.00) & 24.00\,(-4.00) & 24.00\,(-4.00) \\
Rule deduction (context size 3,000) & 51.0 & 25/25 & 72.00 & 88.00\,(+16.00) & 88.00\,(+16.00) & 76.00\,(+4.00) \\
Person location (16k) & 61.5 & 25/25 & 76.00 & 80.00\,(+4.00) & 80.00\,(+4.00) & 68.00\,(-8.00) \\
Previous location (8k) & 56.2 & 25/25 & 28.00 & 32.00\,(+4.00) & 28.00\,(+0.00) & 32.00\,(+4.00) \\
Previous location (16k) & 60.4 & 25/25 & 32.00 & 32.00\,(+0.00) & 40.00\,(+8.00) & 36.00\,(+4.00) \\
Directional relations (8k) & 70.8 & 25/25 & 48.00 & 44.00\,(-4.00) & 40.00\,(-8.00) & 40.00\,(-8.00) \\
Directional relations (16k) & 72.9 & 25/25 & 48.00 & 56.00\,(+8.00) & 56.00\,(+8.00) & 56.00\,(+8.00) \\
Giving events (8k) & 62.5 & 25/25 & 72.00 & 80.00\,(+8.00) & 84.00\,(+12.00) & 80.00\,(+8.00) \\
Giving events (16k) & 80.2 & 25/25 & 60.00 & 64.00\,(+4.00) & 64.00\,(+4.00) & 60.00\,(+0.00) \\
Adjacent-element retrieval (3k) & 99.0 & 12/12 & 50.00 & 25.00\,(-25.00) & 25.00\,(-25.00) & 33.33\,(-16.67) \\
Adjacent-element retrieval (6k) & 95.8 & 12/12 & 25.00 & 8.33\,(-16.67) & 0.00\,(-25.00) & 0.00\,(-25.00) \\
Adjacent-element retrieval (13k) & 97.9 & 12/12 & 16.67 & 8.33\,(-8.33) & 8.33\,(-8.33) & 8.33\,(-8.33) \\
Object location (16k) & 74.0 & 25/25 & 40.00 & 40.00\,(+0.00) & 32.00\,(-8.00) & 24.00\,(-16.00) \\
Object counting (8k) & 68.8 & 25/25 & 20.00 & 24.00\,(+4.00) & 24.00\,(+4.00) & 24.00\,(+4.00) \\
Object counting (16k) & 77.1 & 25/25 & 4.00 & 20.00\,(+16.00) & 24.00\,(+20.00) & 16.00\,(+12.00) \\
Object sets (16k) & 56.2 & 25/25 & 36.00 & 36.00\,(+0.00) & 36.00\,(+0.00) & 36.00\,(+0.00) \\
Negation (8k) & 54.2 & 25/25 & 88.00 & 88.00\,(+0.00) & 92.00\,(+4.00) & 92.00\,(+4.00) \\
Negation (16k) & 71.9 & 25/25 & 68.00 & 72.00\,(+4.00) & 76.00\,(+8.00) & 76.00\,(+8.00) \\
Uncertain knowledge (16k) & 62.5 & 25/25 & 44.00 & 48.00\,(+4.00) & 48.00\,(+4.00) & 48.00\,(+4.00) \\
One-hop implicit retrieval (16k) & 61.5 & 25/25 & 12.00 & 4.00\,(-8.00) & 8.00\,(-4.00) & 12.00\,(+0.00) \\
Two-hop implicit retrieval (8k) & 65.6 & 25/25 & 0.00 & 0.00\,(+0.00) & 0.00\,(+0.00) & 0.00\,(+0.00) \\
Two-hop implicit retrieval (16k) & 67.7 & 25/25 & 0.00 & 0.00\,(+0.00) & 0.00\,(+0.00) & 0.00\,(+0.00) \\
\end{longtable}
\endgroup

\begingroup
\scriptsize
\setlength{\tabcolsep}{2pt}
\renewcommand{\arraystretch}{1.12}
\begin{longtable}{@{}>{\raggedright\arraybackslash}p{1.55in}rrrrrr@{}}
\caption{Qwen3-8B: semantic direction on the top-10\% heads. Entries give accuracy in percent and the change from the paired baseline in parentheses (percentage points). Dashes mark settings without a second-stage evaluation; $\dagger$ marks a generation limit that differs from the saved baseline.}\label{tab:direction-qwen-semantic-10}\\
\toprule
Setting & $n/N$ & Base & NR & $0.25$ & $0.5$ & $0.75$ \\
\midrule
\endfirsthead
\multicolumn{7}{l}{\textit{Continued from the preceding page.}}\\
\toprule
Setting & $n/N$ & Base & NR & $0.25$ & $0.5$ & $0.75$ \\
\midrule
\endhead
\midrule
\multicolumn{7}{r}{\textit{Continued on the next page.}}\\
\endfoot
\bottomrule
\endlastfoot
QA (basic)$^{\dagger}$ & 100/100 & 78.00 & 63.00\,(-15.00) & 82.00\,(+4.00) & 80.00\,(+2.00) & 78.00\,(+0.00) \\
QA (easy) & --- & --- & --- & --- & --- & --- \\
Multi-hop QA (up to 16k) & 25/25 & 28.00 & 24.00\,(-4.00) & 24.00\,(-4.00) & 24.00\,(-4.00) & 28.00\,(+0.00) \\
Rule deduction (context size 3,000) & --- & --- & --- & --- & --- & --- \\
Person location (16k) & --- & --- & --- & --- & --- & --- \\
Previous location (8k) & --- & --- & --- & --- & --- & --- \\
Previous location (16k) & --- & --- & --- & --- & --- & --- \\
Directional relations (8k) & 25/25 & 48.00 & 44.00\,(-4.00) & 36.00\,(-12.00) & 40.00\,(-8.00) & 52.00\,(+4.00) \\
Directional relations (16k) & --- & --- & --- & --- & --- & --- \\
Giving events (8k) & --- & --- & --- & --- & --- & --- \\
Giving events (16k) & --- & --- & --- & --- & --- & --- \\
Adjacent-element retrieval (3k) & 12/12 & 50.00 & 41.67\,(-8.33) & 25.00\,(-25.00) & 25.00\,(-25.00) & 33.33\,(-16.67) \\
Adjacent-element retrieval (6k) & 12/12 & 25.00 & 25.00\,(+0.00) & 16.67\,(-8.33) & 16.67\,(-8.33) & 16.67\,(-8.33) \\
Adjacent-element retrieval (13k) & 12/12 & 16.67 & 8.33\,(-8.33) & 8.33\,(-8.33) & 8.33\,(-8.33) & 16.67\,(+0.00) \\
Object location (16k) & 25/25 & 40.00 & 28.00\,(-12.00) & 28.00\,(-12.00) & 28.00\,(-12.00) & 32.00\,(-8.00) \\
Object counting (8k) & --- & --- & --- & --- & --- & --- \\
Object counting (16k) & --- & --- & --- & --- & --- & --- \\
Object sets (16k) & 25/25 & 36.00 & 32.00\,(-4.00) & 44.00\,(+8.00) & 44.00\,(+8.00) & 44.00\,(+8.00) \\
Negation (8k) & --- & --- & --- & --- & --- & --- \\
Negation (16k) & --- & --- & --- & --- & --- & --- \\
Uncertain knowledge (16k) & --- & --- & --- & --- & --- & --- \\
One-hop implicit retrieval (16k) & 25/25 & 12.00 & 0.00\,(-12.00) & 12.00\,(+0.00) & 12.00\,(+0.00) & 12.00\,(+0.00) \\
Two-hop implicit retrieval (8k) & 25/25 & 0.00 & 0.00\,(+0.00) & 0.00\,(+0.00) & 0.00\,(+0.00) & 0.00\,(+0.00) \\
Two-hop implicit retrieval (16k) & 25/25 & 0.00 & 0.00\,(+0.00) & 0.00\,(+0.00) & 0.00\,(+0.00) & 0.00\,(+0.00) \\
\end{longtable}
\endgroup

\subsection{Qwen3-8B: Positional Direction}

\begingroup
\scriptsize
\setlength{\tabcolsep}{2pt}
\renewcommand{\arraystretch}{1.12}
\begin{longtable}{@{}>{\raggedright\arraybackslash}p{1.55in}rrrrrr@{}}
\caption{Qwen3-8B: positional direction on the top-5\% heads. Entries give accuracy in percent and the change from the paired baseline in parentheses (percentage points).}\label{tab:direction-qwen-positional-5}\\
\toprule
Setting & $w_S$ (\%) & $n/N$ & Base & $1.5$ & $2$ & $2.5$ \\
\midrule
\endfirsthead
\multicolumn{7}{l}{\textit{Continued from the preceding page.}}\\
\toprule
Setting & $w_S$ (\%) & $n/N$ & Base & $1.5$ & $2$ & $2.5$ \\
\midrule
\endhead
\midrule
\multicolumn{7}{r}{\textit{Continued on the next page.}}\\
\endfoot
\bottomrule
\endlastfoot
MK (basic) & 49.0 & 100/100 & 99.00 & 100.00\,(+1.00) & 100.00\,(+1.00) & 100.00\,(+1.00) \\
MK (easy) & 16.7 & 100/100 & 96.00 & 97.00\,(+1.00) & 97.00\,(+1.00) & 97.00\,(+1.00) \\
MK (medium) & 22.9 & 95/100 & 72.63 & 72.63\,(+0.00) & 70.53\,(-2.11) & 70.53\,(-2.11) \\
MK (hard) & 32.3 & 95/100 & 64.21 & 67.37\,(+3.16) & 68.42\,(+4.21) & 70.53\,(+6.32) \\
MV (basic) & 46.9 & 100/100 & 22.00 & 15.00\,(-7.00) & 19.00\,(-3.00) & 20.00\,(-2.00) \\
MV (easy) & 15.6 & 94/100 & 43.62 & 43.62\,(+0.00) & 42.55\,(-1.06) & 37.23\,(-6.38) \\
MV (medium) & 35.4 & 94/100 & 35.11 & 38.30\,(+3.19) & 37.23\,(+2.13) & 36.17\,(+1.06) \\
MV (hard) & 22.9 & 94/100 & 52.13 & 53.19\,(+1.06) & 50.00\,(-2.13) & 48.94\,(-3.19) \\
QA (medium) & 42.7 & 94/100 & 69.15 & 72.34\,(+3.19) & 71.28\,(+2.13) & 71.28\,(+2.13) \\
QA (hard) & 40.6 & 94/100 & 78.72 & 78.72\,(+0.00) & 73.40\,(-5.32) & 76.60\,(-2.13) \\
Paper QA (up to 16k) & 12.5 & 25/25 & 28.00 & 20.00\,(-8.00) & 20.00\,(-8.00) & 20.00\,(-8.00) \\
Conversation topic retrieval & 21.9 & 150/150 & 58.00 & 64.67\,(+6.67) & 66.00\,(+8.00) & 64.67\,(+6.67) \\
Text sorting (1k) & 12.5 & 25/25 & 8.00 & 8.00\,(+0.00) & 8.00\,(+0.00) & 8.00\,(+0.00) \\
Room-property reasoning (context size 500) & 16.7 & 25/25 & 100.00 & 100.00\,(+0.00) & 100.00\,(+0.00) & 100.00\,(+0.00) \\
Room-property reasoning (context size 3,000) & 50.0 & 25/25 & 100.00 & 100.00\,(+0.00) & 100.00\,(+0.00) & 100.00\,(+0.00) \\
Relation chaining (context size 500) & 20.8 & 25/25 & 96.00 & 92.00\,(-4.00) & 96.00\,(+0.00) & 96.00\,(+0.00) \\
Relation chaining (context size 3,000) & 41.7 & 25/25 & 76.00 & 72.00\,(-4.00) & 68.00\,(-8.00) & 64.00\,(-12.00) \\
Rule deduction (context size 500) & 20.8 & 25/25 & 84.00 & 84.00\,(+0.00) & 72.00\,(-12.00) & 72.00\,(-12.00) \\
Mixed long reasoning (8k) & 45.8 & 25/25 & 84.00 & 80.00\,(-4.00) & 88.00\,(+4.00) & 88.00\,(+4.00) \\
Mixed long reasoning (16k) & 50.0 & 25/25 & 72.00 & 84.00\,(+12.00) & 88.00\,(+16.00) & 84.00\,(+12.00) \\
Person location (8k) & 35.4 & 25/25 & 92.00 & 88.00\,(-4.00) & 96.00\,(+4.00) & 96.00\,(+4.00) \\
Object location (8k) & 37.5 & 25/25 & 52.00 & 48.00\,(-4.00) & 48.00\,(-4.00) & 48.00\,(-4.00) \\
Object sets (8k) & 34.4 & 25/25 & 52.00 & 36.00\,(-16.00) & 36.00\,(-16.00) & 36.00\,(-16.00) \\
Uncertain knowledge (8k) & 45.8 & 25/25 & 64.00 & 60.00\,(-4.00) & 56.00\,(-8.00) & 56.00\,(-8.00) \\
One-hop implicit retrieval (8k) & 50.0 & 25/25 & 4.00 & 12.00\,(+8.00) & 12.00\,(+8.00) & 16.00\,(+12.00) \\
\end{longtable}
\endgroup

\begingroup
\scriptsize
\setlength{\tabcolsep}{2pt}
\renewcommand{\arraystretch}{1.12}
\begin{longtable}{@{}>{\raggedright\arraybackslash}p{1.55in}rrrrrr@{}}
\caption{Qwen3-8B: positional direction on the top-10\% heads. Entries give accuracy in percent and the change from the paired baseline in parentheses (percentage points). Dashes mark settings without a second-stage evaluation; $\dagger$ marks a generation limit that differs from the saved baseline.}\label{tab:direction-qwen-positional-10}\\
\toprule
Setting & $n/N$ & Base & $1.25$ & $1.5$ & $1.75$ & $2$ \\
\midrule
\endfirsthead
\multicolumn{7}{l}{\textit{Continued from the preceding page.}}\\
\toprule
Setting & $n/N$ & Base & $1.25$ & $1.5$ & $1.75$ & $2$ \\
\midrule
\endhead
\midrule
\multicolumn{7}{r}{\textit{Continued on the next page.}}\\
\endfoot
\bottomrule
\endlastfoot
MK (basic) & --- & --- & --- & --- & --- & --- \\
MK (easy) & --- & --- & --- & --- & --- & --- \\
MK (medium)$^{\dagger}$ & 95/100 & 72.63 & 70.53\,(-2.11) & 69.47\,(-3.16) & 67.37\,(-5.26) & 67.37\,(-5.26) \\
MK (hard) & --- & --- & --- & --- & --- & --- \\
MV (basic)$^{\dagger}$ & 100/100 & 22.00 & 22.00\,(+0.00) & 24.00\,(+2.00) & 27.00\,(+5.00) & 28.00\,(+6.00) \\
MV (easy) & 95/100 & 44.21 & 40.00\,(-4.21) & 37.89\,(-6.32) & 33.68\,(-10.53) & 28.42\,(-15.79) \\
MV (medium) & --- & --- & --- & --- & --- & --- \\
MV (hard) & --- & --- & --- & --- & --- & --- \\
QA (medium) & --- & --- & --- & --- & --- & --- \\
QA (hard)$^{\dagger}$ & 94/100 & 78.72 & 79.79\,(+1.06) & 77.66\,(-1.06) & 74.47\,(-4.26) & 72.34\,(-6.38) \\
Paper QA (up to 16k) & 25/25 & 28.00 & 24.00\,(-4.00) & 24.00\,(-4.00) & 24.00\,(-4.00) & 20.00\,(-8.00) \\
Conversation topic retrieval & --- & --- & --- & --- & --- & --- \\
Text sorting (1k) & 25/25 & 8.00 & 8.00\,(+0.00) & 8.00\,(+0.00) & 8.00\,(+0.00) & 8.00\,(+0.00) \\
Room-property reasoning (context size 500) & 25/25 & 100.00 & 100.00\,(+0.00) & 100.00\,(+0.00) & 100.00\,(+0.00) & 100.00\,(+0.00) \\
Room-property reasoning (context size 3,000) & 25/25 & 100.00 & 100.00\,(+0.00) & 100.00\,(+0.00) & 100.00\,(+0.00) & 100.00\,(+0.00) \\
Relation chaining (context size 500) & 25/25 & 96.00 & 92.00\,(-4.00) & 92.00\,(-4.00) & 92.00\,(-4.00) & 92.00\,(-4.00) \\
Relation chaining (context size 3,000) & 25/25 & 76.00 & 72.00\,(-4.00) & 72.00\,(-4.00) & 68.00\,(-8.00) & 60.00\,(-16.00) \\
Rule deduction (context size 500) & 25/25 & 84.00 & 84.00\,(+0.00) & 88.00\,(+4.00) & 88.00\,(+4.00) & 84.00\,(+0.00) \\
Mixed long reasoning (8k) & --- & --- & --- & --- & --- & --- \\
Mixed long reasoning (16k) & --- & --- & --- & --- & --- & --- \\
Person location (8k) & --- & --- & --- & --- & --- & --- \\
Object location (8k) & 25/25 & 52.00 & 56.00\,(+4.00) & 60.00\,(+8.00) & 56.00\,(+4.00) & 56.00\,(+4.00) \\
Object sets (8k) & 25/25 & 52.00 & 44.00\,(-8.00) & 40.00\,(-12.00) & 40.00\,(-12.00) & 40.00\,(-12.00) \\
Uncertain knowledge (8k) & 25/25 & 64.00 & 52.00\,(-12.00) & 44.00\,(-20.00) & 44.00\,(-20.00) & 44.00\,(-20.00) \\
One-hop implicit retrieval (8k) & --- & --- & --- & --- & --- & --- \\
\end{longtable}
\endgroup

\subsection{Llama-3.1-8B-Instruct: Semantic Direction}

\begingroup
\scriptsize
\setlength{\tabcolsep}{2pt}
\renewcommand{\arraystretch}{1.12}
\begin{longtable}{@{}>{\raggedright\arraybackslash}p{1.55in}rrrrrr@{}}
\caption{Llama-3.1-8B-Instruct: semantic direction on the top-5\% heads. Entries give accuracy in percent and the change from the paired baseline in parentheses (percentage points).}\label{tab:direction-llama-semantic-5}\\
\toprule
Setting & $w_S$ (\%) & $n/N$ & Base & NR & $0$ & $0.5$ \\
\midrule
\endfirsthead
\multicolumn{7}{l}{\textit{Continued from the preceding page.}}\\
\toprule
Setting & $w_S$ (\%) & $n/N$ & Base & NR & $0$ & $0.5$ \\
\midrule
\endhead
\midrule
\multicolumn{7}{r}{\textit{Continued on the next page.}}\\
\endfoot
\bottomrule
\endlastfoot
QA (easy) & 64.6 & 80/100 & 80.00 & 86.25\,(+6.25) & 87.50\,(+7.50) & 86.25\,(+6.25) \\
Multi-hop QA (up to 16k) & 67.7 & 25/25 & 24.00 & 20.00\,(-4.00) & 16.00\,(-8.00) & 16.00\,(-8.00) \\
Person location (16k) & 77.1 & 25/25 & 76.00 & 76.00\,(+0.00) & 84.00\,(+8.00) & 84.00\,(+8.00) \\
Previous location (8k) & 62.5 & 25/25 & 28.00 & 36.00\,(+8.00) & 28.00\,(+0.00) & 32.00\,(+4.00) \\
Previous location (16k) & 80.2 & 25/25 & 20.00 & 28.00\,(+8.00) & 24.00\,(+4.00) & 24.00\,(+4.00) \\
Directional relations (8k) & 69.8 & 25/25 & 36.00 & 40.00\,(+4.00) & 36.00\,(+0.00) & 40.00\,(+4.00) \\
Directional relations (16k) & 70.8 & 25/25 & 44.00 & 56.00\,(+12.00) & 48.00\,(+4.00) & 48.00\,(+4.00) \\
Giving events (8k) & 57.3 & 25/25 & 72.00 & 72.00\,(+0.00) & 68.00\,(-4.00) & 72.00\,(+0.00) \\
Giving events (16k) & 58.3 & 25/25 & 76.00 & 64.00\,(-12.00) & 64.00\,(-12.00) & 72.00\,(-4.00) \\
Object location (8k) & 57.3 & 25/25 & 52.00 & 52.00\,(+0.00) & 40.00\,(-12.00) & 44.00\,(-8.00) \\
Object location (16k) & 67.7 & 25/25 & 36.00 & 40.00\,(+4.00) & 32.00\,(-4.00) & 32.00\,(-4.00) \\
Object counting (8k) & 65.6 & 25/25 & 4.00 & 12.00\,(+8.00) & 8.00\,(+4.00) & 12.00\,(+8.00) \\
Object counting (16k) & 80.2 & 25/25 & 8.00 & 12.00\,(+4.00) & 16.00\,(+8.00) & 12.00\,(+4.00) \\
Object sets (8k) & 57.3 & 25/25 & 44.00 & 44.00\,(+0.00) & 52.00\,(+8.00) & 44.00\,(+0.00) \\
Object sets (16k) & 67.7 & 25/25 & 40.00 & 48.00\,(+8.00) & 48.00\,(+8.00) & 48.00\,(+8.00) \\
Negation (8k) & 58.3 & 25/25 & 88.00 & 88.00\,(+0.00) & 84.00\,(-4.00) & 84.00\,(-4.00) \\
Negation (16k) & 76.0 & 25/25 & 72.00 & 64.00\,(-8.00) & 72.00\,(+0.00) & 72.00\,(+0.00) \\
Uncertain knowledge (8k) & 71.9 & 25/25 & 76.00 & 80.00\,(+4.00) & 72.00\,(-4.00) & 72.00\,(-4.00) \\
Uncertain knowledge (16k) & 80.2 & 25/25 & 64.00 & 68.00\,(+4.00) & 56.00\,(-8.00) & 68.00\,(+4.00) \\
One-hop implicit retrieval (8k) & 70.8 & 25/25 & 36.00 & 40.00\,(+4.00) & 24.00\,(-12.00) & 40.00\,(+4.00) \\
One-hop implicit retrieval (16k) & 83.3 & 25/25 & 24.00 & 28.00\,(+4.00) & 36.00\,(+12.00) & 28.00\,(+4.00) \\
Two-hop implicit retrieval (8k) & 72.9 & 25/25 & 4.00 & 0.00\,(-4.00) & 8.00\,(+4.00) & 8.00\,(+4.00) \\
Two-hop implicit retrieval (16k) & 70.8 & 25/25 & 0.00 & 4.00\,(+4.00) & 0.00\,(+0.00) & 4.00\,(+4.00) \\
\end{longtable}
\endgroup

\begingroup
\scriptsize
\setlength{\tabcolsep}{2pt}
\renewcommand{\arraystretch}{1.12}
\begin{longtable}{@{}>{\raggedright\arraybackslash}p{1.55in}rrrrrr@{}}
\caption{Llama-3.1-8B-Instruct: semantic direction on the top-10\% heads. Entries give accuracy in percent and the change from the paired baseline in parentheses (percentage points). Dashes mark settings without a second-stage evaluation; $\dagger$ marks a generation limit that differs from the saved baseline.}\label{tab:direction-llama-semantic-10}\\
\toprule
Setting & $n/N$ & Base & NR & $0.25$ & $0.5$ & $0.75$ \\
\midrule
\endfirsthead
\multicolumn{7}{l}{\textit{Continued from the preceding page.}}\\
\toprule
Setting & $n/N$ & Base & NR & $0.25$ & $0.5$ & $0.75$ \\
\midrule
\endhead
\midrule
\multicolumn{7}{r}{\textit{Continued on the next page.}}\\
\endfoot
\bottomrule
\endlastfoot
QA (easy) & --- & --- & --- & --- & --- & --- \\
Multi-hop QA (up to 16k) & 25/25 & 24.00 & 24.00\,(+0.00) & 8.00\,(-16.00) & 8.00\,(-16.00) & 16.00\,(-8.00) \\
Person location (16k) & --- & --- & --- & --- & --- & --- \\
Previous location (8k) & --- & --- & --- & --- & --- & --- \\
Previous location (16k) & --- & --- & --- & --- & --- & --- \\
Directional relations (8k) & --- & --- & --- & --- & --- & --- \\
Directional relations (16k) & --- & --- & --- & --- & --- & --- \\
Giving events (8k) & 25/25 & 72.00 & 72.00\,(+0.00) & 68.00\,(-4.00) & 72.00\,(+0.00) & 72.00\,(+0.00) \\
Giving events (16k) & 25/25 & 76.00 & 68.00\,(-8.00) & 64.00\,(-12.00) & 68.00\,(-8.00) & 72.00\,(-4.00) \\
Object location (8k) & 25/25 & 52.00 & 20.00\,(-32.00) & 48.00\,(-4.00) & 40.00\,(-12.00) & 44.00\,(-8.00) \\
Object location (16k) & --- & --- & --- & --- & --- & --- \\
Object counting (8k) & --- & --- & --- & --- & --- & --- \\
Object counting (16k) & --- & --- & --- & --- & --- & --- \\
Object sets (8k) & --- & --- & --- & --- & --- & --- \\
Object sets (16k) & --- & --- & --- & --- & --- & --- \\
Negation (8k) & 25/25 & 88.00 & 80.00\,(-8.00) & 76.00\,(-12.00) & 84.00\,(-4.00) & 84.00\,(-4.00) \\
Negation (16k) & 25/25 & 72.00 & 52.00\,(-20.00) & 60.00\,(-12.00) & 64.00\,(-8.00) & 68.00\,(-4.00) \\
Uncertain knowledge (8k) & --- & --- & --- & --- & --- & --- \\
Uncertain knowledge (16k) & --- & --- & --- & --- & --- & --- \\
One-hop implicit retrieval (8k) & --- & --- & --- & --- & --- & --- \\
One-hop implicit retrieval (16k) & --- & --- & --- & --- & --- & --- \\
Two-hop implicit retrieval (8k) & --- & --- & --- & --- & --- & --- \\
Two-hop implicit retrieval (16k) & --- & --- & --- & --- & --- & --- \\
\end{longtable}
\endgroup

\subsection{Llama-3.1-8B-Instruct: Positional Direction}

\begingroup
\scriptsize
\setlength{\tabcolsep}{2pt}
\renewcommand{\arraystretch}{1.12}
\begin{longtable}{@{}>{\raggedright\arraybackslash}p{1.55in}rrrrrr@{}}
\caption{Llama-3.1-8B-Instruct: positional direction on the top-5\% heads. Entries give accuracy in percent and the change from the paired baseline in parentheses (percentage points).}\label{tab:direction-llama-positional-5}\\
\toprule
Setting & $w_S$ (\%) & $n/N$ & Base & $1.5$ & $2$ & $2.5$ \\
\midrule
\endfirsthead
\multicolumn{7}{l}{\textit{Continued from the preceding page.}}\\
\toprule
Setting & $w_S$ (\%) & $n/N$ & Base & $1.5$ & $2$ & $2.5$ \\
\midrule
\endhead
\midrule
\multicolumn{7}{r}{\textit{Continued on the next page.}}\\
\endfoot
\bottomrule
\endlastfoot
MK (basic) & 55.2 & 81/100 & 98.77 & 100.00\,(+1.23) & 97.53\,(-1.23) & 27.16\,(-71.60) \\
MK (easy) & 11.5 & 100/100 & 83.00 & 80.00\,(-3.00) & 5.00\,(-78.00) & 0.00\,(-83.00) \\
MK (medium) & 8.3 & 100/100 & 61.00 & 61.00\,(+0.00) & 24.00\,(-37.00) & 1.00\,(-60.00) \\
MK (hard) & 20.8 & 81/100 & 54.32 & 48.15\,(-6.17) & 1.23\,(-53.09) & 0.00\,(-54.32) \\
MV (basic) & 54.2 & 100/100 & 34.00 & 20.00\,(-14.00) & 7.00\,(-27.00) & 0.00\,(-34.00) \\
MV (easy) & 14.6 & 81/100 & 27.16 & 19.75\,(-7.41) & 0.00\,(-27.16) & 0.00\,(-27.16) \\
MV (medium) & 11.5 & 80/100 & 32.50 & 45.00\,(+12.50) & 3.75\,(-28.75) & 0.00\,(-32.50) \\
MV (hard) & 12.5 & 100/100 & 24.00 & 23.00\,(-1.00) & 3.00\,(-21.00) & 0.00\,(-24.00) \\
QA (basic) & 54.2 & 80/100 & 97.50 & 97.50\,(+0.00) & 45.00\,(-52.50) & 0.00\,(-97.50) \\
QA (medium) & 56.2 & 80/100 & 67.50 & 66.25\,(-1.25) & 36.25\,(-31.25) & 2.50\,(-65.00) \\
QA (hard) & 56.2 & 100/100 & 71.00 & 72.00\,(+1.00) & 35.00\,(-36.00) & 12.00\,(-59.00) \\
Paper QA (up to 16k) & 31.2 & 25/25 & 20.00 & 20.00\,(+0.00) & 20.00\,(+0.00) & 0.00\,(-20.00) \\
Conversation topic retrieval & 17.7 & 150/150 & 74.00 & 80.00\,(+6.00) & 35.33\,(-38.67) & 0.00\,(-74.00) \\
Text sorting (1k) & 10.4 & 25/25 & 0.00 & 0.00\,(+0.00) & 0.00\,(+0.00) & 0.00\,(+0.00) \\
Room-property reasoning (context size 500) & 19.8 & 25/25 & 64.00 & 60.00\,(-4.00) & 44.00\,(-20.00) & 56.00\,(-8.00) \\
Room-property reasoning (context size 3,000) & 56.2 & 25/25 & 44.00 & 44.00\,(+0.00) & 60.00\,(+16.00) & 48.00\,(+4.00) \\
Relation chaining (context size 500) & 21.9 & 25/25 & 52.00 & 52.00\,(+0.00) & 56.00\,(+4.00) & 56.00\,(+4.00) \\
Relation chaining (context size 3,000) & 45.8 & 25/25 & 56.00 & 52.00\,(-4.00) & 48.00\,(-8.00) & 48.00\,(-8.00) \\
Rule deduction (context size 500) & 15.6 & 25/25 & 60.00 & 52.00\,(-8.00) & 48.00\,(-12.00) & 48.00\,(-12.00) \\
Rule deduction (context size 3,000) & 42.7 & 25/25 & 48.00 & 48.00\,(+0.00) & 56.00\,(+8.00) & 40.00\,(-8.00) \\
Mixed long reasoning (8k) & 47.9 & 25/25 & 76.00 & 56.00\,(-20.00) & 56.00\,(-20.00) & 8.00\,(-68.00) \\
Mixed long reasoning (16k) & 39.6 & 25/25 & 64.00 & 68.00\,(+4.00) & 48.00\,(-16.00) & 4.00\,(-60.00) \\
Person location (8k) & 54.2 & 25/25 & 96.00 & 92.00\,(-4.00) & 44.00\,(-52.00) & 0.00\,(-96.00) \\
Adjacent-element retrieval (3k) & 26.0 & 12/12 & 0.00 & 0.00\,(+0.00) & 0.00\,(+0.00) & 0.00\,(+0.00) \\
Adjacent-element retrieval (6k) & 37.5 & 12/12 & 0.00 & 0.00\,(+0.00) & 8.33\,(+8.33) & 0.00\,(+0.00) \\
Adjacent-element retrieval (13k) & 39.6 & 12/12 & 16.67 & 8.33\,(-8.33) & 8.33\,(-8.33) & 0.00\,(-16.67) \\
\end{longtable}
\endgroup

\begingroup
\scriptsize
\setlength{\tabcolsep}{2pt}
\renewcommand{\arraystretch}{1.12}
\begin{longtable}{@{}>{\raggedright\arraybackslash}p{1.55in}rrrrrr@{}}
\caption{Llama-3.1-8B-Instruct: positional direction on the top-10\% heads. Entries give accuracy in percent and the change from the paired baseline in parentheses (percentage points). Dashes mark settings without a second-stage evaluation; $\dagger$ marks a generation limit that differs from the saved baseline.}\label{tab:direction-llama-positional-10}\\
\toprule
Setting & $n/N$ & Base & $1.25$ & $1.5$ & $1.75$ & $2$ \\
\midrule
\endfirsthead
\multicolumn{7}{l}{\textit{Continued from the preceding page.}}\\
\toprule
Setting & $n/N$ & Base & $1.25$ & $1.5$ & $1.75$ & $2$ \\
\midrule
\endhead
\midrule
\multicolumn{7}{r}{\textit{Continued on the next page.}}\\
\endfoot
\bottomrule
\endlastfoot
MK (basic) & --- & --- & --- & --- & --- & --- \\
MK (easy)$^{\dagger}$ & 100/100 & 83.00 & 79.00\,(-4.00) & 71.00\,(-12.00) & 68.00\,(-15.00) & 61.00\,(-22.00) \\
MK (medium)$^{\dagger}$ & 100/100 & 61.00 & 60.00\,(-1.00) & 57.00\,(-4.00) & 47.00\,(-14.00) & 31.00\,(-30.00) \\
MK (hard)$^{\dagger}$ & 81/100 & 54.32 & 44.44\,(-9.88) & 40.74\,(-13.58) & 23.46\,(-30.86) & 11.11\,(-43.21) \\
MV (basic)$^{\dagger}$ & 100/100 & 34.00 & 19.00\,(-15.00) & 5.00\,(-29.00) & 1.00\,(-33.00) & 6.00\,(-28.00) \\
MV (easy) & 81/100 & 27.16 & 29.63\,(+2.47) & 30.86\,(+3.70) & 12.35\,(-14.81) & 0.00\,(-27.16) \\
MV (medium) & --- & --- & --- & --- & --- & --- \\
MV (hard)$^{\dagger}$ & 100/100 & 24.00 & 19.00\,(-5.00) & 32.00\,(+8.00) & 30.00\,(+6.00) & 18.00\,(-6.00) \\
QA (basic)$^{\dagger}$ & 80/100 & 97.50 & 90.00\,(-7.50) & 86.25\,(-11.25) & 66.25\,(-31.25) & 13.75\,(-83.75) \\
QA (medium)$^{\dagger}$ & 80/100 & 67.50 & 41.25\,(-26.25) & 28.75\,(-38.75) & 38.75\,(-28.75) & 23.75\,(-43.75) \\
QA (hard) & --- & --- & --- & --- & --- & --- \\
Paper QA (up to 16k) & 25/25 & 20.00 & 20.00\,(+0.00) & 20.00\,(+0.00) & 20.00\,(+0.00) & 24.00\,(+4.00) \\
Conversation topic retrieval & --- & --- & --- & --- & --- & --- \\
Text sorting (1k) & 25/25 & 0.00 & 0.00\,(+0.00) & 8.00\,(+8.00) & 8.00\,(+8.00) & 16.00\,(+16.00) \\
Room-property reasoning (context size 500) & 25/25 & 64.00 & 52.00\,(-12.00) & 48.00\,(-16.00) & 52.00\,(-12.00) & 48.00\,(-16.00) \\
Room-property reasoning (context size 3,000) & --- & --- & --- & --- & --- & --- \\
Relation chaining (context size 500) & --- & --- & --- & --- & --- & --- \\
Relation chaining (context size 3,000) & 25/25 & 56.00 & 52.00\,(-4.00) & 52.00\,(-4.00) & 52.00\,(-4.00) & 60.00\,(+4.00) \\
Rule deduction (context size 500) & 25/25 & 60.00 & 52.00\,(-8.00) & 52.00\,(-8.00) & 56.00\,(-4.00) & 52.00\,(-8.00) \\
Rule deduction (context size 3,000) & --- & --- & --- & --- & --- & --- \\
Mixed long reasoning (8k) & 25/25 & 76.00 & 56.00\,(-20.00) & 48.00\,(-28.00) & 40.00\,(-36.00) & 36.00\,(-40.00) \\
Mixed long reasoning (16k) & --- & --- & --- & --- & --- & --- \\
Person location (8k) & 25/25 & 96.00 & 92.00\,(-4.00) & 76.00\,(-20.00) & 84.00\,(-12.00) & 32.00\,(-64.00) \\
Adjacent-element retrieval (3k) & 12/12 & 0.00 & 0.00\,(+0.00) & 8.33\,(+8.33) & 25.00\,(+25.00) & 8.33\,(+8.33) \\
Adjacent-element retrieval (6k) & --- & --- & --- & --- & --- & --- \\
Adjacent-element retrieval (13k) & 12/12 & 16.67 & 16.67\,(+0.00) & 16.67\,(+0.00) & 0.00\,(-16.67) & 0.00\,(-16.67) \\
\end{longtable}
\endgroup

\section{Additional Related Work}
\label{app:additional-related-work}

\subsection{Representation Assumptions in RoPE Theory}
\label{app:related-representation-assumptions}
Existing RoPE theories study both random representations and fixed content
vectors. \citet{ICLR2025_e6d58fc6} use independently sampled Gaussian queries
and keys to show that expected scores need not decay with distance; their
paper also provides deterministic constructions and a single-frequency
semantic-instability result.
\citet{NEURIPS2024_9f12dd32} derive base-dependent context bounds using
i.i.d. query and key coordinates with a common variance and a similar-key
model $k^*=q+\epsilon$ with zero-mean noise.
\citet{du2026ropedistinguishespositionstokens} hold query and key
vectors fixed and sample relative positions. Their normal approximation
assumes that no frequency amplitude dominates, and their context-dependent
failure estimates use additional amplitude regularity. For their theoretical
token-inversion estimate, they compare a key with a hypothetical phase-aligned
reference sharing its frequency amplitudes.
Our analysis uses the coefficients of realized query and key activations,
without prescribing a distribution over these activations. The finite-window
moment and adjacent-response bounds allow nonuniform frequency amplitudes.
Our diagnostic estimates reversal for sampled key pairs using their full
score-margin coefficients. The reversal
guarantees retain an explicit Gaussian approximation condition on the score
distribution over positions. This formulation supports diagnostics computed
directly from recorded model activations.

\subsection{RoPE Failure Mechanisms}
\label{app:related-failure-mechanisms}
Existing analyses identify failures in both the positional and semantic
functions of rotary attention. \citet{ICLR2025_e6d58fc6} show how rotating
semantic channels can lose their selectivity over long contexts, with a
formal result for a single frequency, and motivate p-RoPE.
\citet{du2026ropedistinguishespositionstokens} organize score-level failures
into four types. \emph{Position inversion} occurs when the same query--key
pair scores higher at a distant position than at a nearby one, reversing
the locality preference. \emph{Position aliasing} occurs when that pair
receives numerically equal scores at two different relative distances.
For two distinct keys evaluated at the same relative distance,
\emph{token inversion} reverses the ordering established at zero relative
distance, while \emph{token aliasing} makes their scores numerically equal.
The aliasing events explicitly account for finite-precision computation.
These definitions separate loss of positional discrimination from loss of
token discrimination, and distinguish a reversed preference from a tie.

In vision--language models,
\citet{qi2025semanticsrediscoveringspatialawareness} probe positional
sensitivity by adding a one-step RoPE rotation to visual keys while
holding their extracted content vectors and the query fixed. They report
changes in attention weights and in group-averaged score sensitivity.
This study highlights local positional resolution as a practical
concern. In \S\ref{sec:failure-modes}, our semantic reversal follows the
token-inversion event of \citet{du2026ropedistinguishespositionstokens}.
We define positional insensitivity through the absolute score gap between
adjacent positions, normalized by each query--key pair's own coefficient
norm. This criterion measures local
response strength independently of its direction; a preference for nearer
positions and equal scores at separated positions are distinct properties.

\subsection{Context Extension}
\label{app:related-context-extension}

Other work explains the mechanisms behind context-extension failures.
HoPE connects extrapolation failures to learned frequency components whose
attention-score patterns change beyond the training
window~\citep{chen-etal-2025-hope}. Resonance RoPE identifies a feature gap
at unseen integer positions, including in frequency components whose value
ranges are covered during training~\citep{wang-etal-2024-resonance}.
Frayed RoPE relates long-context degradation to dispersion and overlap of
query--key clusters, which weaken the ability of attention-sink tokens to
absorb attention~\citep{ICLR2026_fc684a13}.
These accounts connect frequency behavior and representation geometry to
specific failures, complementing the event-based classification in
Appendix~\ref{app:related-failure-mechanisms}.

Position Interpolation and YaRN adjust RoPE frequencies to extend
context~\citep{chen2023extendingcontextwindowlarge,ICLR2024_874a4d89}.
For position interpolation, \citet{wu2026datashapesropefrequency} show that
its practical effect depends on the task: frequency scaling can help
dependencies that stretch with input length while impairing retrieval at
fixed relative distances and precise local alignment. YaRN also identifies
the difficulty of distinguishing nearby positions after uniform frequency
scaling~\citep{ICLR2024_874a4d89}.

\subsection{Internal Representations and Failure Diagnosis}
\label{app:related-internal-diagnosis}
\citet{NEURIPS2024_1403ab1a} decompose hidden states using
position-wise means across sequences, connect positional-vector
extrapolation to failures, and develop interventions for NoPE models.
Other work diagnoses positional attention bias~\citep{hsieh-etal-2024-found},
retrieval-head behavior~\citep{ICLR2025_9b77f073},
and context-extension mechanisms through head ablations and activation
patching~\citep{pmlr-v267-zhao25ad}.
StreamingLLM links rolling-cache degradation to attention
sinks~\citep{ICLR2024_5e5fd18f}, and contrastive
attribution traces output errors through internal
states~\citep{tan2026contrastiveattributionwildinterpretability}.
Our local score measurements complement these analyses; relating a measured
vulnerability to an output error requires evidence about its downstream effect.

\subsection{Behavioral Evaluation and Attention Competition}
\label{app:related-behavioral-evaluation}
LongBench, RULER, and HELMET measure performance across long-context
tasks~\citep{bai-etal-2024-longbench,hsieh2024ruler,ICLR2025_f5332c82}.
Lost in the Middle and NoLiMa expose sensitivity to information placement
and retrieval beyond literal
matching~\citep{liu-etal-2024-lost,pmlr-v267-modarressi25a}.
Longer inputs motivate analyses of distractor
dilution~\citep{ICLR2026_b619cd6d} and critical attention scaling under
simplified token geometries~\citep{ICLR2026_8b08bbf8}.
Parallel context encoding has also been linked to elevated attention
entropy~\citep{zhang-etal-2025-attention}.
Our pre-softmax diagnostics add two internal measurements alongside task
performance; softmax competition and subsequent layers determine how these
properties influence the final prediction.

\subsection{Long-Context Failure Modes}
\label{app:related-long-context-failure-modes}
Existing benchmarks and studies \citep{hsieh2024ruler,bai-etal-2024-longbench,ICLR2025_f5332c82,NEURIPS2024_babilong, pmlr-v267-modarressi25a, goldman2024really, vodrahalli2024michelangelo, levy-etal-2024-task} identify long-context failures related to retrieval and reasoning, and that these failure modes can be coupled \citep{du2025contextlengthhurtsllm}. Our diagnostics feature two internal measurements and position the model on a preference spectrum with preset tasks anchored as combinations of both modes.

\stopcontents[appendix]
\endgroup

\end{document}